\documentclass[10pt]{article}

\usepackage[a4paper,top=30mm, bottom=50mm, left=25mm, right=25mm]{geometry}

\usepackage{lmodern}
\usepackage[T1]{fontenc}
\usepackage{palatino}

\usepackage{amsmath,amsthm}
\usepackage{amssymb, amsbsy}
\usepackage{gensymb}

\usepackage[dvipsnames]{xcolor}
\usepackage[pagebackref]{hyperref}
\hypersetup{
colorlinks=true,
linkcolor=Brown,
citecolor=OliveGreen,
urlcolor=RoyalBlue,
}

\usepackage{graphicx}
\graphicspath{ {images/} }
\usepackage{import}
\usepackage{sectsty}
\usepackage{enumitem, outlines}
\usepackage{lineno}

\usepackage{cite}
\usepackage[edges]{forest}

\usepackage{amsfonts}
\usepackage[linesnumbered,vlined,ruled]{algorithm2e}

\usepackage{makecell}

\usepackage{setspace}
\usepackage[list=true]{subcaption}
\usepackage[title]{appendix}
\usepackage{multirow}
\usepackage{array}

\newtheorem{theorem}{Theorem}
\newtheorem{lemma}{Lemma}

\date{}

\let\oldnl\nl
\newcommand{\nonl}{\renewcommand{\nl}{\let\nl\oldnl}}

\begin{document}

\title{Mutual Visibility by Fat Robots with Slim Omnidirectional Camera}




\author{
 {Kaustav Bose\footnote{Advanced Computing and Microelectronics Unit, Indian Statistical Institute, Kolkata, India; email: kaustavbose27@gmail.com}} \and
 
 {Abhinav Chakraborty \footnote{Advanced Computing and Microelectronics Unit, Indian Statistical Institute, Kolkata, India; email: abhinav.chakraborty06@gmail.com}}  \and
 
{Krishnendu Mukhopadhyaya \footnote{Advanced Computing and Microelectronics Unit, Indian Statistical Institute, Kolkata, India; email: krishnendu@isical.ac.in
}} 
 
 }

\maketitle              

\begin{abstract}
In the existing literature of the \textsc{Mutual Visibility} problem for autonomous robot swarms, the adopted visibility models have some idealistic assumptions that are not consistent with practical sensing device implementations. This paper investigates the problem in the more realistic visibility model called \emph{opaque fat robots with slim omnidirectional camera}. The robots are modeled as unit disks, each having an omnidirectional camera represented as a disk of smaller size. We assume that the robots have compasses that allow agreement in the direction and orientation of both axes of their local coordinate systems. The robots are equipped with visible lights which serve as a medium of communication and also as a form of memory. We present a distributed algorithm for the \textsc{Mutual Visibility} problem which is provably correct in the semi-synchronous setting. Our algorithm also provides a solution for \textsc{Leader Election} which we use as a subroutine in our main algorithm. Although \textsc{Leader Election} is trivial with two axis agreement in the full visibility model, it is challenging in our case and is of independent interest.
\end{abstract}

\section{Introduction}\label{sec:introduction}
\subsection{Motivation}

Robot swarms are distributed multi-robot systems consisting of simple and inexpensive robots that can collaboratively execute complex tasks. Such systems offer several advantages over traditional single robot systems, such as scalability, robustness and versatility. Employing robot swarms is particularly tempting in risky and hazardous scenarios, such as disaster management, military operations, etc. Coordination of autonomous robot swarms has been the object of study in a number of fields, including robotics, artificial intelligence, control theory and distributed computing. Within distributed computing, in particular, there has been a series of studies on the algorithmic aspects of distributed coordination of robot swarms (see \cite{flocchini2019distributed} for a comprehensive survey).  

In the traditional framework of theoretical studies, robot swarms are modeled as systems of \emph{autonomous} (there is no centralized control), \emph{anonymous} (no unique identifiers that can be used in computation), \emph{identical} (indistinguishable by physical appearance) and \emph{homogeneous} (all of them execute the same algorithm) computational entities that can freely move on the plane. The robots do not have access to any global coordinate system. They each have their own local coordinate system. Each robot is equipped with sensors that allow it to perceive its surroundings. The robots operate in Look-Compute-Move (LCM) cycles, i.e., when a robot becomes active, it takes a snapshot of its surroundings (Look), computes a destination point based on the snapshot (Compute) and moves towards the computed destination (Move). There are three types of schedulers considered in the literature that describe the timing of operations of the robots. In the fully synchronous model ($\mathcal{FSYNC}$), time is logically divided into global rounds and all the robots are activated in each round. The robots take their snapshots simultaneously and then execute their moves concurrently. The semi-synchronous model ($\mathcal{SSYNC}$) is the same as $\mathcal{FSYNC}$ except for the fact that not all robots are necessarily activated in each round. The most general model is the asynchronous model ($\mathcal{ASYNC}$), where there are no assumptions except that all robots are activated infinitely often. The standard movement models are \emph{rigid} and \emph{non-rigid}. In the \emph{rigid} model, a robot can move towards its computed destination accurately in one step. On the other hand, the \emph{non-rigid} model assumes that a robot may stop before it reaches its destination. However, there is a fixed but unknown $\delta > 0$ so that each robot traverses at least the distance $\delta$ unless its destination is closer than $\delta$.
 
 In the full visibility model, it is assumed that each robot can see all the other robots in the swarm. However, limitations in sensing capabilities may not allow the robots to obtain a complete view of the configuration. For example, if the robots are equipped with cameras as sensing devices, then the visibility of a robot can be obstructed by the presence of other robots. A simple model that captures this scenario is the opaque point robot model. In this model, the robots are assumed to be dimensionless points in the plane and two robots are visible to each other if and only if no other robot lies on the line segment joining them. A more realistic model is the opaque fat robot model, where the robots are seen as disks. In this model, two robots are visible to each other if and only if there exists an unobstructed line between two points on the surface of the robots. Although this model addresses the impracticality of the physical aspect of the previous point robot model, the visibility assumptions are still impractical. In order to implement this model in practice, the entire surface of a robot has to be covered with cameras, which is expensive and impractical. A more economical way of obtaining an omnidirectional view is to have a circular arrangement of fewer number of cameras around the center of the robot. This can be modeled as having a single camera represented as a disk whose diameter is smaller than the diameter of the robot. We say that a robot is visible to another robot if and only if there is an unobstructed line from a point on the surface of the former to the camera of the latter. This model was introduced by Honorat et al. in \cite{tcs/HonoratPT14}. The region obstructed by a single robot in this model has the shape of a truncated infinite cone instead of a line or a strip in the case of opaque point robots and traditional opaque fat robots models respectively (see Fig. \ref{fig: 3type}). Some implications of this model which show a stark difference from the previous models are mentioned in the following.
 
 \begin{itemize}
     \item A robot cannot always ascertain from its view whether or not it is on the convex hull of the configuration (see Fig. \ref{fig: false_hull}). 
     
     \item A robot cannot always ascertain from its view whether or not it is on a collinear configuration, i.e., all the robots are on the same line (see Fig. \ref{fig: coll1}  and Fig.  \ref{fig: coll2}).
     
     \item The visibility relation is not symmetric, i.e., $r_1$ being visible to $r_2$ does not imply that  $r_2$ is also visible to $r_1$ (see Fig. \ref{fig: asymm}). 
     
 \end{itemize}

  \begin{figure}[thb!]
     \centering
     \begin{subfigure}[b]{0.25\linewidth}
         \centering
         \includegraphics[width=\linewidth]{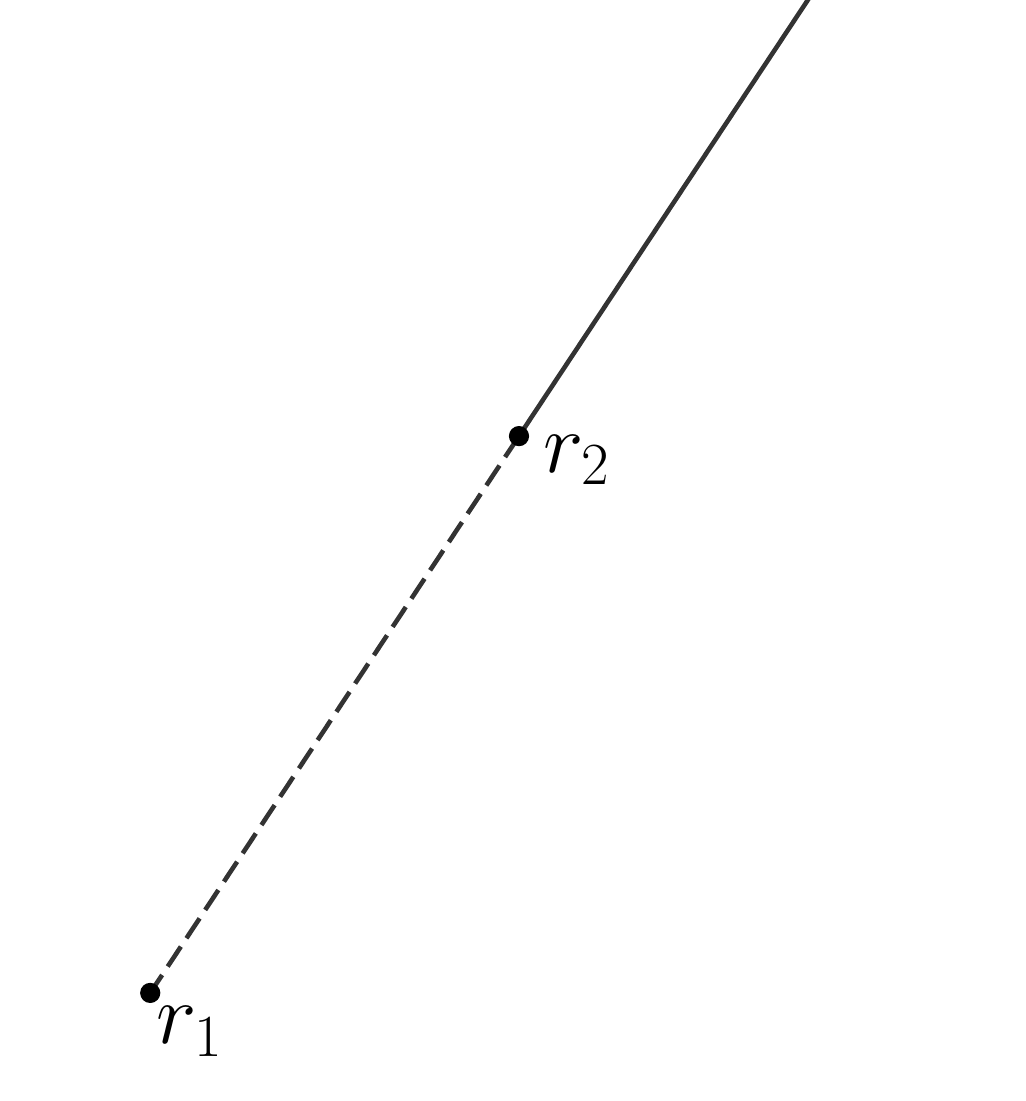}
         \caption{}
          \label{}
     \end{subfigure}
     \hfill
     \begin{subfigure}[b]{0.25\linewidth}
         \centering
         \includegraphics[width=\linewidth]{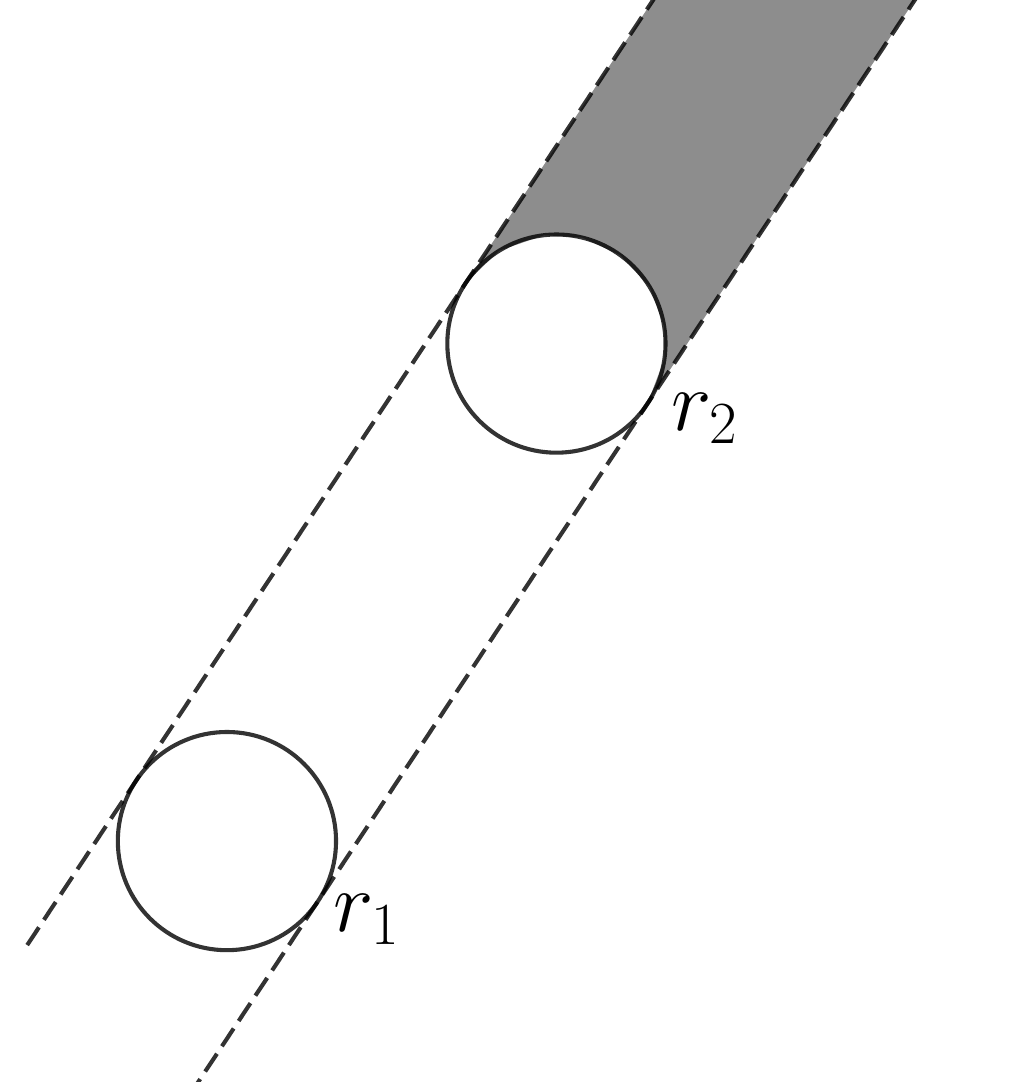}
         \caption{}
          \label{} 
    \end{subfigure}
     \hfill
     \begin{subfigure}[b]{0.25\linewidth}
         \centering
         \includegraphics[width=\linewidth]{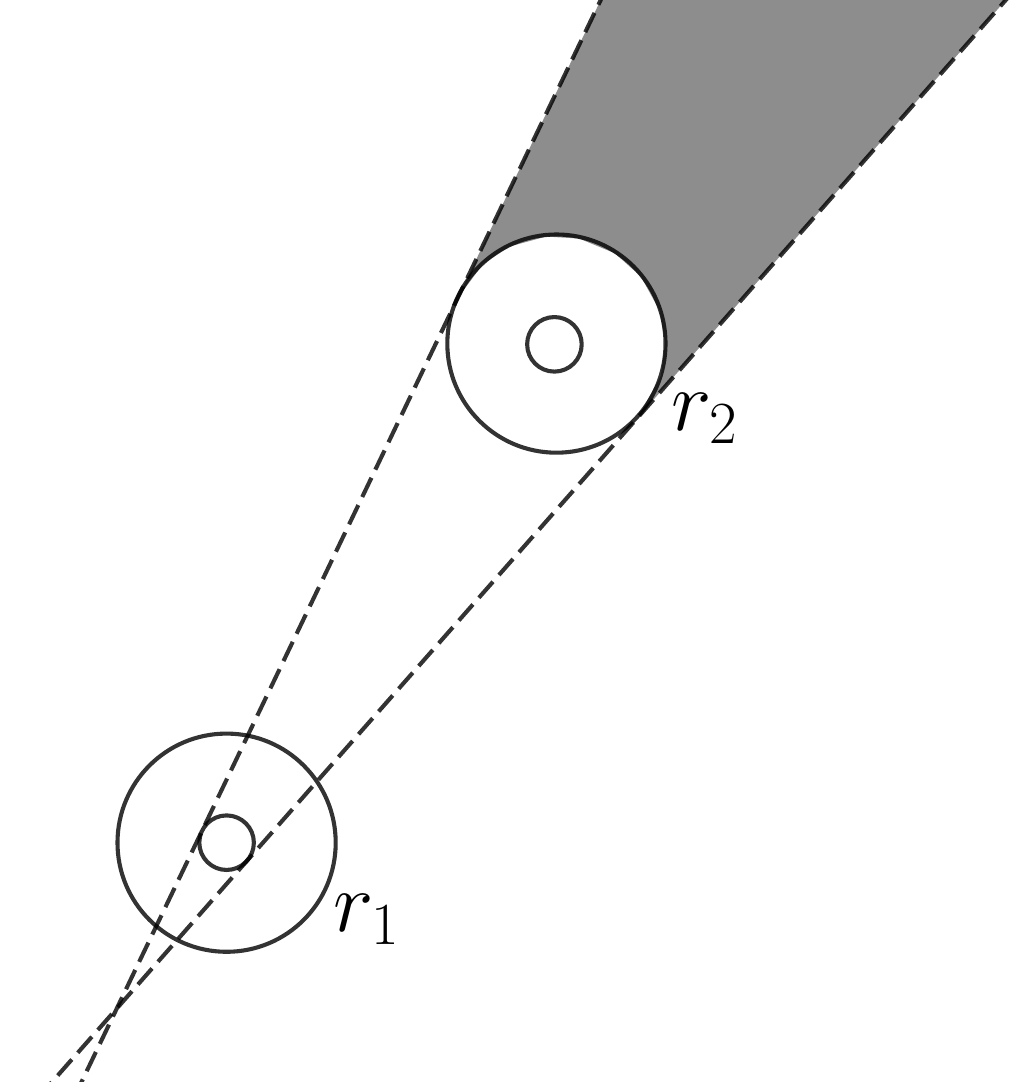}
         \caption{}
          \label{}         
     \end{subfigure}
\caption[Short Caption] {The region invisible to $r_1$ due to the obstruction by $r_2$. a) Opaque point robots. b) The traditional model of opaque fat robots. c) Fat robots with slim omnidirectional camera.}
\label{fig: 3type}
\end{figure}


\begin{figure}[thb!]
     \centering
     \begin{subfigure}[b]{0.6\linewidth}
         \centering
         \includegraphics[width=\linewidth]{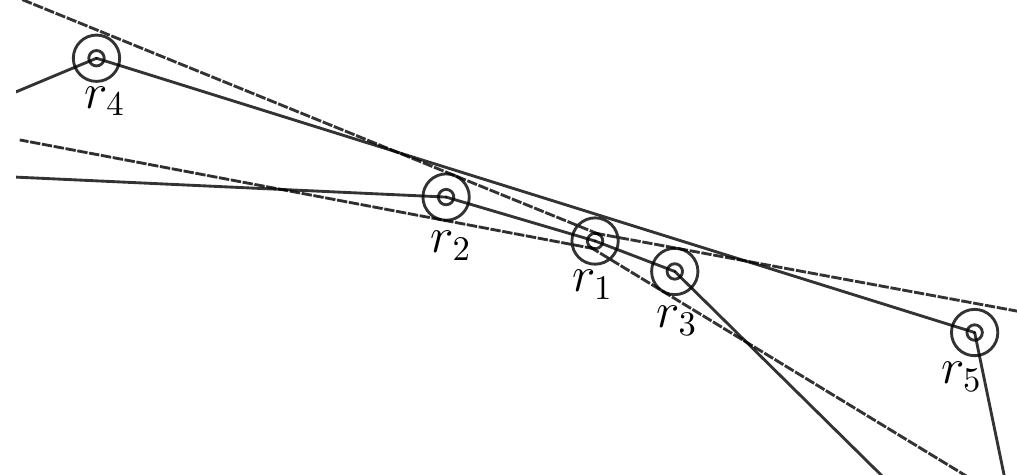}
         \caption{}
          \label{fig: false_hull}
     \end{subfigure}
     \\
     \begin{subfigure}[b]{0.35\linewidth}
         \centering
         \includegraphics[width=\linewidth]{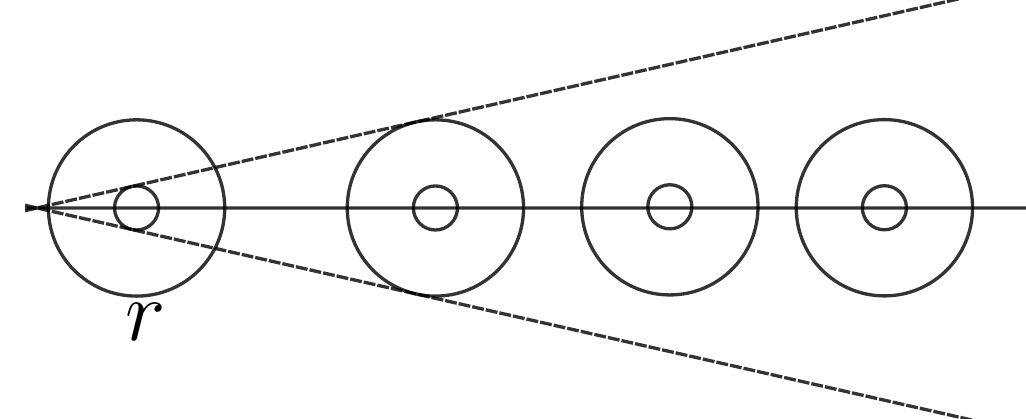}
         \caption{}
          \label{fig: coll1} 
    \end{subfigure}
    \hfill
     \begin{subfigure}[b]{0.35\linewidth}
         \centering
         \includegraphics[width=\linewidth]{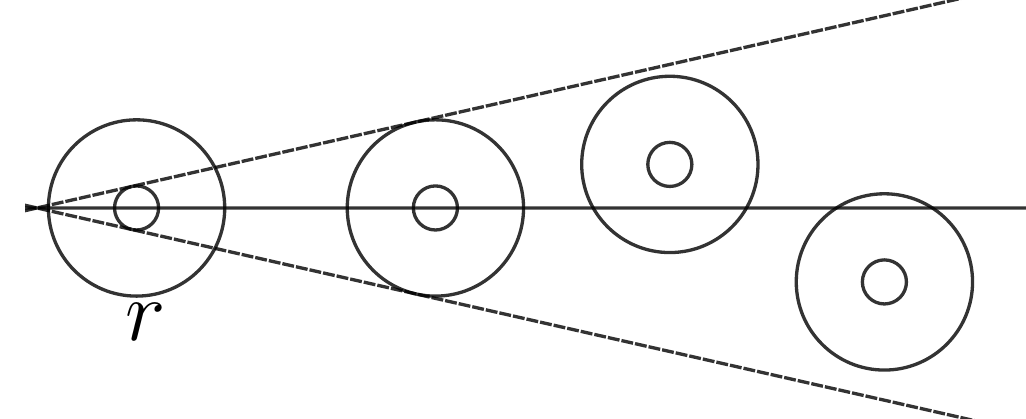}
         \caption{}
          \label{fig: coll2} 
    \end{subfigure}
     \\
     \begin{subfigure}[b]{0.35\linewidth}
         \centering
         \includegraphics[width=\linewidth]{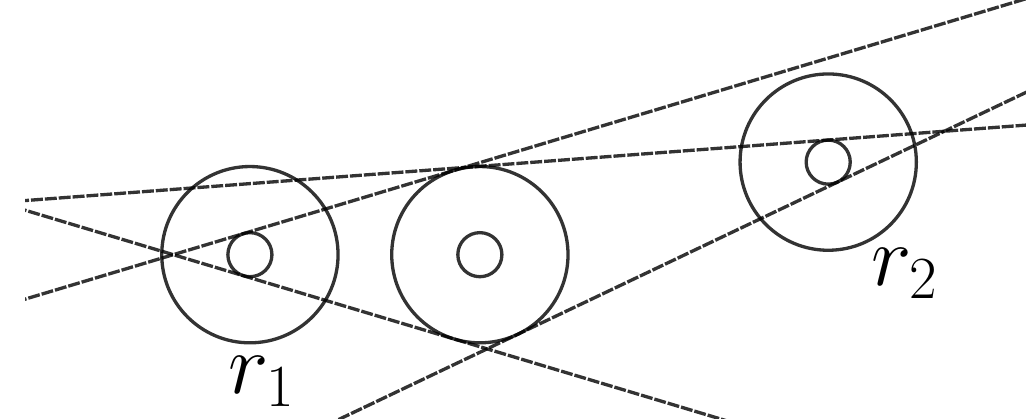}
         \caption{}
          \label{fig: asymm}         
     \end{subfigure}
\caption[Short Caption]{a) The robot $r_1$ `thinks' that it is on the convex hull as it lies on the convex hull of the robots visible to it, but it is actually an interior robot. b)-c) The robot $r$ cannot distinguish between a collinear and a non-collinear configuration. d) The robot $r_1$ is visible to $r_2$, but  $r_2$ is not visible to $r_1$.  }
\end{figure}

 The fundamental problem in the obstructed visibility scenario is the \textsc{Mutual Visibility} problem, where the robots are required to reach a configuration in which all robots can see each other. Although the problem has been extensively studied in the opaque point robot model and also in the traditional fat robot model, it is yet to be studied for opaque fat robots with slim omnidirectional camera. In fact, to the best of our knowledge, the only work which considered robots with slim omnidirectional camera is that of Honorat et al. \cite{tcs/HonoratPT14}, where  the gathering problem for four robots was studied.


 \subsection{Our Contribution}
 
 In this paper, we initiate the study of the \textsc{Mutual Visibility} problem for opaque fat robots with slim omnidirectional camera. In particular, each robot is modeled as a disk of unit radius and its camera is also a disk having the same center but a smaller radius. We consider the \textit{luminous robots} model ($\mathcal{LUMI}$) \cite{flocchini2019distributed} in which each robot is equipped with a visible light that can assume a constant number of predefined colors. The lights serve as a medium of communication and also as a form of memory. We assume that the robots have two axis agreement, i.e., all robots agree on the direction and orientation of both $X$ and $Y$ axes of their local coordinate systems. The total number of robots in the swarm is not known beforehand.  However, an upper bound on the width of the initial configuration is known in advance. In this setting, we have presented a distributed algorithm for the \textsc{Mutual Visibility} problem using 7 colors. We have proved the correctness of our algorithm in the semi-synchronous ($\mathcal{SSYNC}$) setting. Our algorithm also provides a solution for the \textsc{Leader Election} problem, which we use as a subroutine in our main algorithm. Although \textsc{Leader Election} is trivial with two axis agreement in the full visibility model, it is challenging in our case and is of independent interest. We have provided extensive simulation studies assessing our algorithm with respect to different performance metrics. To set up the simulation environment, we had to solve the problem of determining whether a robot is visible to another in a configuration of robots with slim omnidirectional camera. The simulation framework can be useful for conducting experiments for future works in this model.

 \subsection{Related Works}

The study of the \textsc{Mutual Visibility} problem was initiated by Di Luna et al. \cite{di2014mutual} for oblivious robots in $\mathcal{SSYNC}$ for the opaque point robot model. They solved the problem by forming a convex $n$-gon, where $n$ is the total number of robots in the swarm. However, they assumed that $n$ is known beforehand. Later, Sharma et al. \cite{sharma2015bounds} analyzed and improved the round complexity of the algorithm in $\mathcal{FSYNC}$. With the assumption that the total number of robots is not known in advance, the problem was first studied by Di Luna et al. \cite{di2017mutual} in the luminous robots ($\mathcal{LUMI}$) model. In the rigid movement model, they solved the problem (a) with $2$ colors in $\mathcal{SSYNC}$, and (b) with $3$ colors in $\mathcal{ASYNC}$. In the non-rigid movement model, they solved the problem (a) with $3$ colors in $\mathcal{SSYNC}$  and (b) with $3$ colors in $\mathcal{ASYNC}$ under one axis agreement. Sharma et al. \cite{sharma2015mutual} later improved the algorithms in terms of the number of colors required. Then, a series of papers \cite{vaidyanathan2015logarithmic,sharma2016complete,sharma2017log} appeared, aiming towards reducing the time complexity of the algorithm. Bhagat et al. \cite{bhagat2017optimum} first solved the problem in $\mathcal{ASYNC}$ without any agreement on the coordinate axes. The problem was studied for robots with finite persistent memory and no communication capability in \cite{BhagatM19,Bhagat20}. The problem was also considered for faulty robots \cite{aljohani2018complete,POUDEL2021116} and robots operating on a grid-based terrain \cite{ADHIKARY2022,HectorVST22,SharmaVT20}. All the papers mentioned till now have considered the opaque point robot model. In the traditional opaque fat robot model, the problem was studied in \cite{sharmafat18,SharmaBM18}. In \cite{sharmafat18}, Sharma et al. proposed an algorithm that solved the problem in $\mathcal{FSYNC}$ using 9 colors. Then, in \cite{SharmaBM18}, Sharma et al. proposed an algorithm that solves the problem in $\mathcal{SSYNC}$. However, they assumed that the number of robots $n$ is known beforehand and no two robots visible to each other in the initial configuration are more than $n$ units apart. Agathangelou et al. \cite{DBLP:conf/podc/AgathangelouGM13} studied the gathering problem in the traditional opaque fat robot model in $\mathcal{ASYNC}$. The first phase of their gathering algorithm corresponds to forming a mutually visible configuration. They assumed that $n$ is known beforehand and the robots have an agreement on the clockwise-counterclockwise orientation. 


\section{Model and Terminology}\label{mod}

\subsection{Model}

 \noindent\textbf{Robots.} The robots are assumed to be anonymous, autonomous, identical and homogeneous. Each robot is modeled as a disk of radius equal to 1 unit. The robots do not have access to any global coordinate system. Each robot perceives its surroundings based on its local coordinate system. We assume that their local coordinate systems agree on the direction and orientation of both axes. 
 

  \noindent\textbf{Visibility.} Each robot is equipped with a slim omnidirectional camera. The embedded camera is modeled as a disk whose center coincides with the center of the robot. The diameter of the camera is strictly smaller than the diameter of the robot. For any robot $r$, $\mathcal{B}_{r}$ and $\mathcal{C}_{r}$ will respectively denote the disks representing the body and the camera of $r$. For a set of robots $\mathfrak{R} = \{r_i, \ldots, r_n \}$, $r_j$ will be visible to $r_i$ iff $\exists$ points $p_i$ and $p_j$ on the boundaries of $\mathcal{C}_{r_i}$ and $\mathcal{B}_{r_j}$ respectively such that the line segment joining $p_i$ and $p_j$ (not including the end-points) lies inside $\mathbb{R}^2 \setminus (\bigcup\limits_{k \neq i} \mathcal{B}_{r_k} \cup \mathcal{C}_{r_i})$. 
  

    \noindent\textbf{Lights.} Each robot is equipped with a visible light that can assume a constant number of pre-defined colors. The lights serve as a communication medium and also as a form of memory. In this paper, our algorithm uses 7 colors, namely \texttt{off}, \texttt{defeated}, \texttt{leader}, \texttt{subordinate}, \texttt{no space}, \texttt{expand} and \texttt{final}. Initially, all robots have their lights set to \texttt{off}.

    \noindent\textbf{Look-Compute-Move cycles.} Each robot, when active, operates in \textsc{Look-Compute-Move} cycles. In the \textsc{Look} phase, a robot obtains a snapshot of the positions and colors of all robots which are visible to it. In the \textsc{Compute} phase, based on the observed configuration, each robot computes a destination point and a color according to a deterministic distributed algorithm. The destination point may be its current position as well. Finally, in the \textsc{Move} phase, it sets its light to the decided color and moves towards the destination point (if the destination point is different from its current position). When a robot transitions from one LCM cycle to the next, all of its local memory is erased, except for the color of the light.

    \noindent \textbf{Scheduler.} We assume that the robots are controlled by a semi-synchronous ($\mathcal{SSYNC}$) scheduler. In this model, time is logically divided into global rounds. In each round, one or more robots are activated. Each robot is activated infinitely often. The active robots take their snapshots simultaneously and then execute their moves concurrently. As a consequence, no robot is observed while moving.

    \noindent \textbf{Movement.} We assume that the robots have \emph{non-rigid} movements. This means that a robot may stop before it reaches its destination. However, there is a fixed  $\delta > 0$ so that each robot traverses at least the distance $\delta$ unless its destination is closer than $\delta$. For simplicity, we assume that $\delta \geq 2$. This is a reasonable assumption because, in practice, a robot should be able to move a distance of at least the extent of its diameter in one go. Also, with some minor modifications, our algorithm can be easily adapted to work without this assumption. 
    
    

\subsection{Terminology}

We shall denote the set of robots by $\mathfrak{R} = \{r_1, r_2, \dots, r_n\}$, $n \geq 3$. The position of a robot will be specified by the position of its center. So when we say that a robot is at a point $p$ on the plane, we shall mean that its center is at $p$. Also, when we say that the distance between two robots is $d$, we shall mean that the distance between their centers is $d$. With respect to the local coordinate systems of the robots, positive and negative directions of the $X$-axis will be referred to as \emph{east} and \emph{west} respectively, and the positive and negative directions of the $Y$-axis will be referred to as \emph{north} and \emph{south} respectively. Since the robots have two axis agreement, they agree on east, west, north and south. The orientation of the $X$-axis will be called horizontal, and the orientation of the $Y$-axis will be called vertical.

\section{Technical Difficulties and an Overview of Our Strategy}

In this section, we highlight some key difficulties of the problem in the current model and give an overview of our algorithm. A more formal description of the algorithm is provided in Section \ref{section 5}.

  A common approach to solve the \textsc{Mutual Visibility} problem is to bring all the robots on the convex hull of the configuration. The robots on the convex hull set their lights to some particular color. This helps the interior robots to identify the convex hull. The interior robots then move to the convex hull. The main difficulty towards employing this strategy in our setting is that it is not always possible for a robot to correctly determine whether it is on the convex hull or not (see Fig. \ref{fig: false_hull}). A robot can only be sure of being on the convex hull if it has $\geq 180^{\degree}$ unobstructed view. However, a robot on the convex hull may not necessarily have  $\geq 180^{\degree}$ unobstructed view. In this case, we may ask such a robot (which `thinks' that it may be on the convex hull but is not sure about it) to move outward so that it may get an obstructed view. However, the robots obstructing its view may also do the same and create obstruction in the resulting configuration as well. Also, since the convex hull may change due to the movements, a robot that was sure of being on the convex hull in the previous configuration and had changed its color accordingly may become an interior robot in the new configuration.

  Since the strategy of bringing the interior robots on a convex hull appears to be difficult to implement in this model, we try a different approach. Instead of trying to bring the interior robots on to the convex hull, we will sequentially build the convex hull anew. For this, we shall first elect a leader. It may appear that leader election is trivial with two axis agreement, since we can elect the leader as the southmost robot, and in the case of a tie, the eastmost robot among the southmost robots. However, we again face the problem of mistaken identity: a robot that is not southmost may not find any robot on its south (due to visibility obstructions) and thus wrongly infer that it is the southmost. We shall call such a robot a \textit{false southmost robot} (see Fig. \ref{fig: false}). So we see that the current visibility model prevents a straightforward solution to leader election by exploiting the two axis agreement assumption. We now give a brief description of how we solve the problem. When a robot finds (according to its local view) that it is southmost, it will move southwards in order to get an unobstructed view and determine whether it is actually southmost or not. If a robot finds that it is jointly southmost, it will only move if it does not see any robot to its east on the same horizontal line. Any other robot will change its color to \texttt{defeated}. We say that a robot is \textit{sure of being southmost} if it has an unobstructed view of the south of the plane. This happens when all other robots are $\geq (1-c)$ units to its north. We can show that eventually, there will be a robot that is sure of being the southmost. Now, if we ask this robot to become the leader by changing its color to \texttt{leader}, then the problem is that it may be overtaken by other robots which are also moving southwards (upon wrongly thinking of being southmost). Thus, we may have multiple robots with color \texttt{leader}. However, it can be shown that if a robot is able to reach at least $\frac{D}{\sqrt{3}}$ units ($D$ is an upper bound on the horizontal width of the initial configuration) to the south of the rest of the swarm, then there will be no subsequent movements by false southmost robots. To carry out the next stage of our mutual visibility algorithm, we also need an empty space of vertical width of at least 10 units. So we ask a robot to become the leader only when it finds that it is $\geq$ max$\{10, \frac{D}{\sqrt{3}}\}$ units south of the rest of the swarm.


%

  \begin{figure}[thb!]
     \centering
         \includegraphics[width=0.6\linewidth]{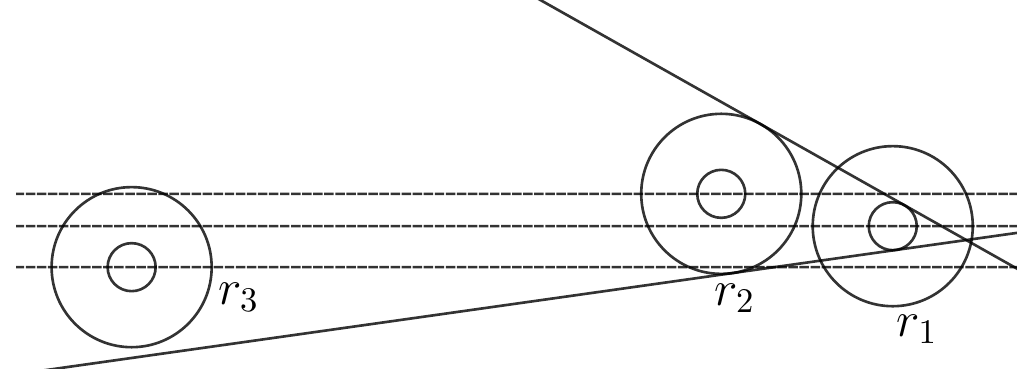}
         
         \caption{ The robot $r_1$ is a false southmost robot.}
          \label{fig: false}         
\end{figure}

 Once the leader is elected, we enter the second stage of our algorithm. Here, the other robots will build a mutually visible structure with the leader. In particular, the mutually visible structure will be a chain, as shown in Fig. \ref{fig12}, where the leader is at the southmost point, called the \emph{tip} of the chain. We want the chain to be \emph{regular} in the sense that all the sides and the  angles are equal. The length of each side is called the \emph{stretch} of the chain, and the value of each exterior angle is called the \emph{turning angle} of the chain. Notice that, as shown in Fig. \ref{fig: turning}, the turning angle of the chain needs to be large enough in order to have mutual visibility. However, if the turning angle is too large, we may not be able to accommodate all the robots in the chain. We can prove that the minimum turning angle required for mutual visibility decreases if we increase the stretch. Therefore, depending on the total number of robots in the team, the stretch and the turning angle of the chain have to be adjusted so that all the robots can be accommodated within the chain while maintaining mutual visibility. The problem is that the size of the team is not known beforehand. So, the robots will start building the chain with some predefined value of stretch and turning angle, and then if they find out that the chain cannot accommodate all the robots, it will be rebuilt with a larger stretch and a smaller turning angle.

 \begin{figure}[thb!]
     \centering
     \begin{subfigure}[b]{0.46\linewidth}
         \centering
         \includegraphics[width=\linewidth]{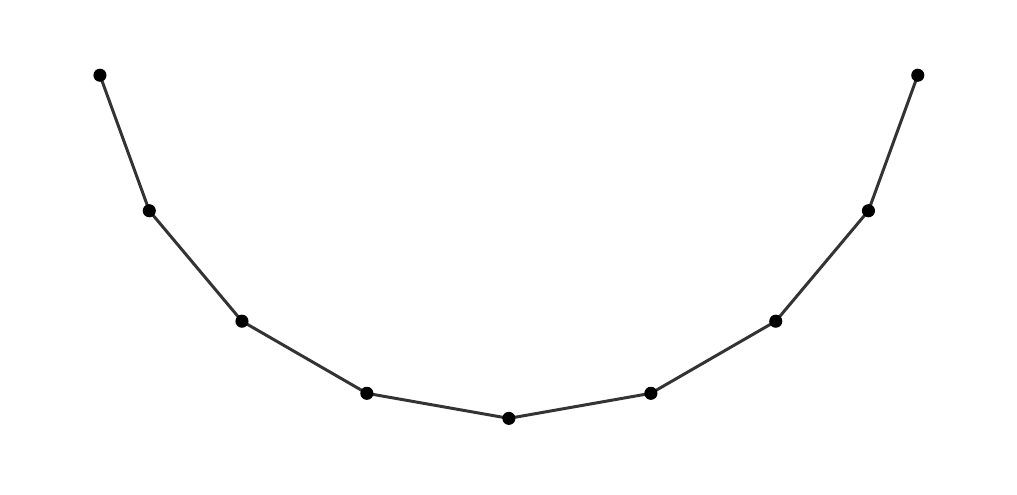}
         \caption{}
          \label{fig12}
     \end{subfigure}
     \hfill
     \begin{subfigure}[b]{0.46\linewidth}
         \centering
         \includegraphics[width=\linewidth]{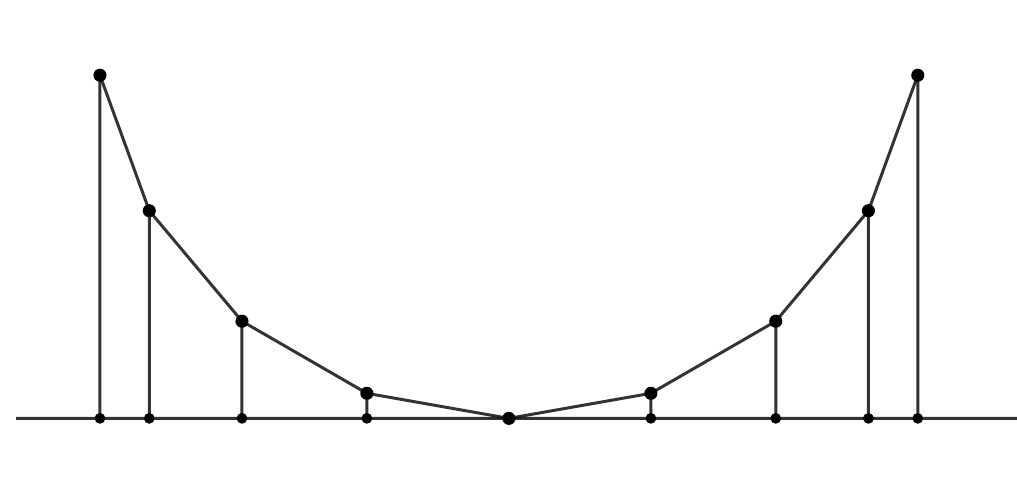}
         \caption{}
          \label{fig13}         
     \end{subfigure}
\caption[Short Caption]{a) The mutually visible chain. b) The corresponding base chain.}
\end{figure} 

 Let us now describe the process of building the chain. As mentioned earlier, the robots will start building a chain with some predefined value of stretch and turning angle. Take the projection of the chain on the horizontal line through the leader, as shown in Fig. \ref{fig13}. We call this projection the \emph{base chain}. Recall that when the leader is set, all other robots are to the north of the leader. These robots will then sequentially come and place themselves on the base chain. The first robot will place itself on the east of the leader, the second robot will place itself on the west of the leader, the next one will go to the east, and so on. Notice that the distance between two consecutive points on the base chain gradually decreases. When this distance becomes less than 2 units, a new robot cannot be accommodated. When this happens, the next robot will set its light to \texttt{no space}. Then a chain with a larger stretch and a smaller turning angle will be chosen and the robots will reposition themselves on the corresponding base chain. Finally, we will obtain a base chain where all robots can be accommodated. Then, the robots will move vertically to form the corresponding chain.

 \begin{figure}[thb!]
     \centering
     \begin{subfigure}[b]{0.46\linewidth}
         \centering
         \includegraphics[width=\linewidth]{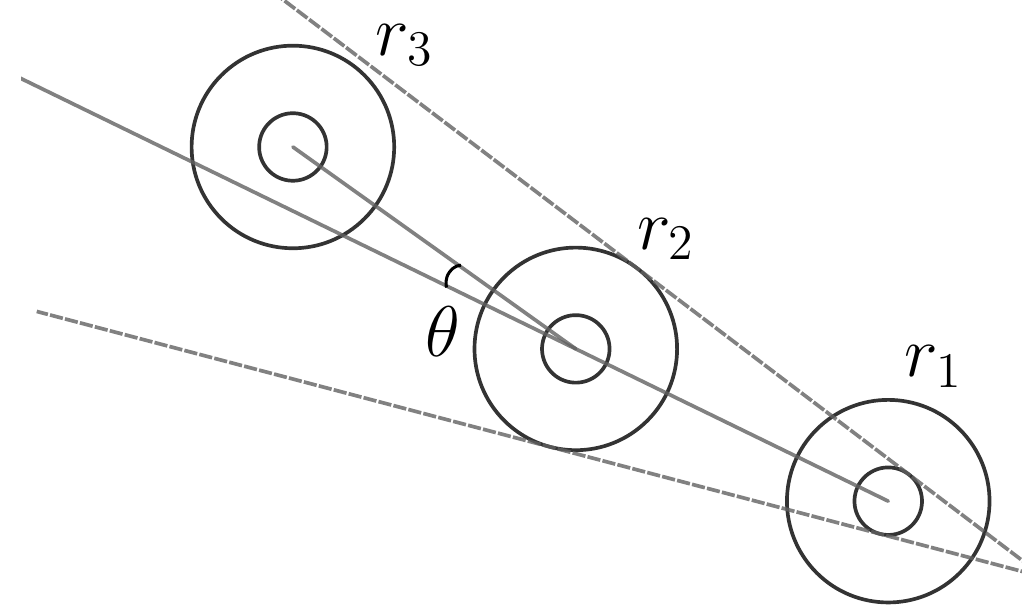}
         \caption{}
     \end{subfigure}
     \hfill
     \begin{subfigure}[b]{0.46\linewidth}
         \centering
         \includegraphics[width=\linewidth]{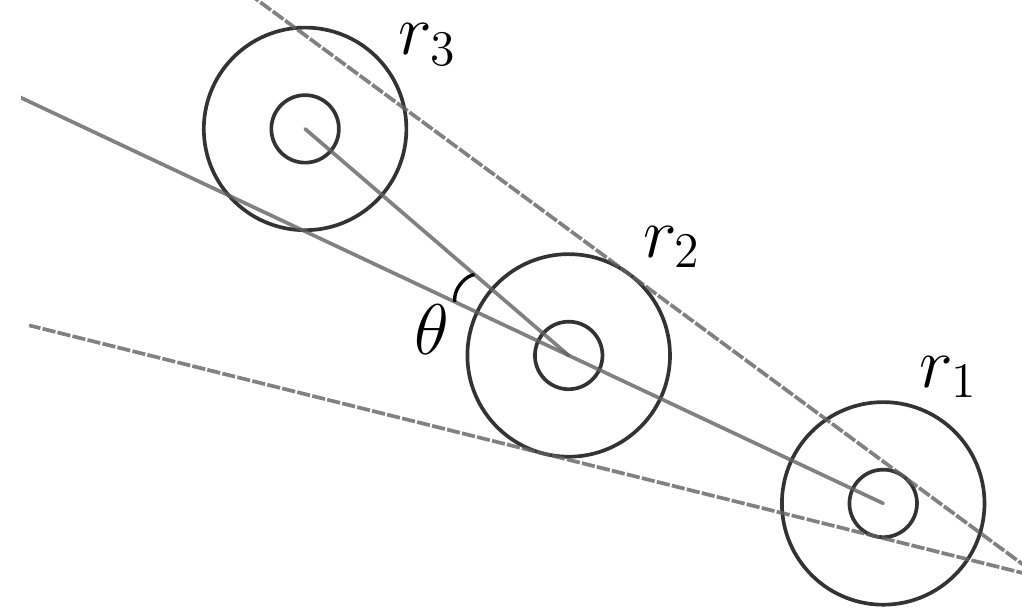}
         \caption{}
     \end{subfigure}
\caption[Short Caption]{a) The robot $r_3$ is not visible to $r_1$. b)  The robot $r_3$ is  visible to $r_1$.}
 \label{fig: turning}
\end{figure}

\section{Basic Properties of the Mutually Visible Chain}

In this section, we prove some basic properties of the mutually visible chain.

\begin{lemma}\label{lem: 3rob}
 Let $r_1, r_2$ and $r_3$ be three robots located at three consecutive points of a chain of stretch $d$ and turning angle $\theta$. The robots $r_1$ and $r_3$ will be visible to each other if and only if  $\sin \theta > \frac{1-c}{d}$.
\end{lemma}

\begin{proof}

 Let $A, B$ and $C$ denote the centers of $r_1, r_2$ and $r_3$ respectively. We have $d(A, B) = d(B, C) = d$. Assume that the turning angle $\theta$ be such that the three circles $\mathcal{C}_{r_1}, \mathcal{B}_{r_2}$ and $\mathcal{B}_{r_3}$ share a tangent as shown in Fig. \ref{fig14}. We shall show that $sin\theta = \frac{1-c}{d}$. It is easy to see that our claim follows from this.

Let us draw Fig. \ref{fig14} in a slightly different way so that the common tangent is horizontal. In particular, fix a coordinate system with the origin at $B$ and $BC$ as the negative direction of the $X$ axis (see Fig. \ref{fig15}). So the equation of the common tangent is $Y = 1$. Then clearly, the $Y$-coordinate of $A$ is $1-c$. Also, since $|AB| = d$, we can say that the $Y$-coordinate of $A$ is $d\sin\theta$. Hence, $d\sin\theta=1-c$ $\implies \sin\theta = \frac{1-c}{d}$. 

\end{proof}

 \begin{figure}[thb!]
     \centering
     \begin{subfigure}[b]{0.35\linewidth}
         \centering
         \includegraphics[width=\linewidth]{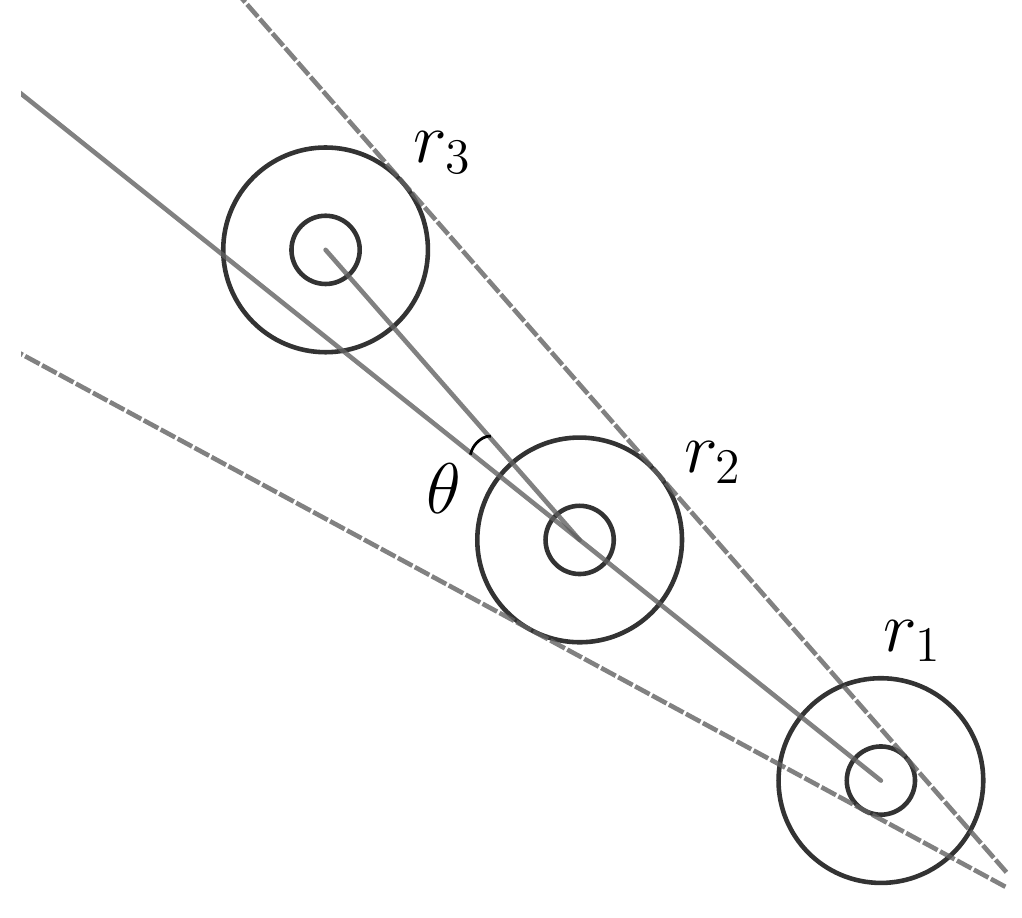}
         \caption{}
          \label{fig14}
     \end{subfigure}
     \hfill
     \begin{subfigure}[b]{0.58\linewidth}
         \centering
         \includegraphics[width=\linewidth]{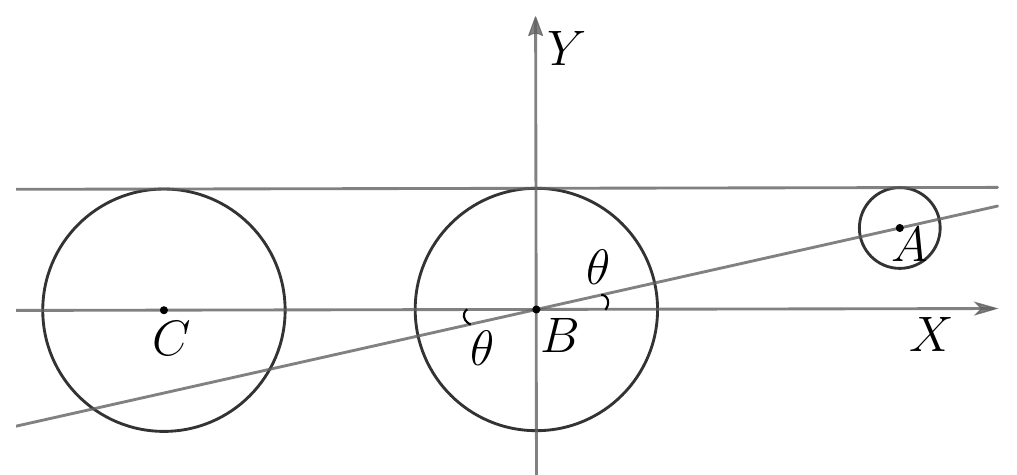}
         \caption{}
          \label{fig15}         
     \end{subfigure}
\caption[Short Caption]{Illustrations supporting the proof of Lemma \ref{lem: 3rob}.}
\end{figure}

\begin{lemma}\label{lem: chain vis}

 Consider a chain of robots having stretch $d$ and turning angle $\theta$. If $\sin \theta > \frac{1-c}{d}$, then any two robots on the chain are visible to each other. 
\end{lemma}

\begin{proof}

Let $r$ and $r'$ be any two robots on the chain. We shall show that $r$ can see $r'$. From symmetry, it will imply that $r'$ is also able to see $r$. Suppose there are $i$ robots between $r$ and $r'$ on the chain. We prove the result by using induction on $i$. For $i=0$, there is nothing to prove. For $i=1$, the result follows from Lemma \ref{lem: 3rob}. So assume that $i>1$. Assume that the result is true for $i-1$. Suppose that the robots from $r$ to $r'$ are arranged on the chain as $r, r_1, \ldots, r_i, r'$. It follows from the induction hypothesis that $r_i$ is visible to $r$. Consider the region between $l_1$ and $l_i$ where $l_1$ is a tangent on $\mathcal{C}_{r}$, $\mathcal{B}_{r_1}$ and $l_i$ is a tangent on $\mathcal{C}_{r}$, $\mathcal{B}_{r_i}$ (see Fig. \ref{fig16}). If there is any portion of $r'$ outside this region, then it must be visible to $r$. Notice that it is sufficient to show that $r'$ is not obstructed by $r_i$. In other words, it is sufficient to only consider $r$, $r_i$ and $r'$. 

If we have three robots at $A, B, C$ such that $d(A,B) = d(B,C) = d$ and $\angle ABC < \pi - \sin^{-1}(\frac{1-c}{d})$, then by Lemma  \ref{lem: 3rob}, the robot at $A$ can see the robot at $C$. Notice that the same is true if we have $d(A,B) > d(B,C) = d$. This is because the obstructed region of the second case is contained inside the obstructed region of the first case (see Fig. \ref{fig17}). Notice that this holds for the robots $r$, $r_i$ and $r'$. Hence, $r$ is able to see $r'$.

\end{proof}

 \begin{figure}[thb!]
     \centering
     \begin{subfigure}[b]{0.5\linewidth}
         \centering
         \includegraphics[width=\linewidth]{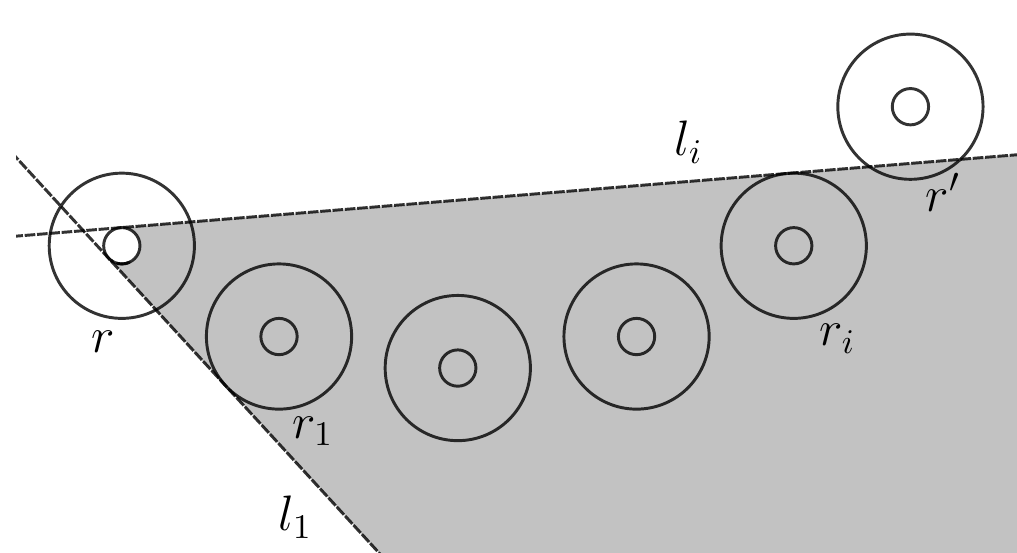}
         \caption{}
          \label{fig16}
     \end{subfigure}
     \hfill
     \begin{subfigure}[b]{0.46\linewidth}
         \centering
         \includegraphics[width=\linewidth]{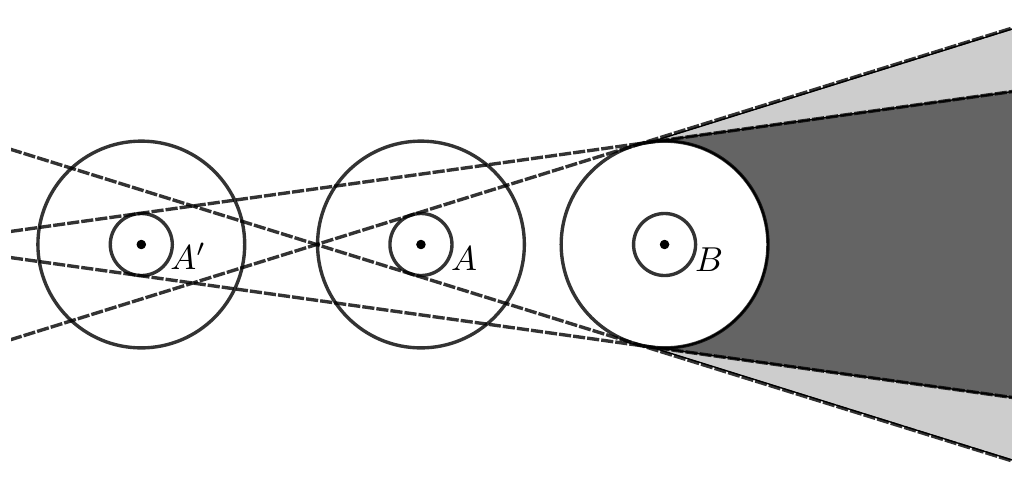}
         \caption{}
          \label{fig17}         
     \end{subfigure}
\caption[Short Caption]{Illustrations supporting the proof of Lemma \ref{lem: chain vis}.}
\end{figure}

In our algorithm, we shall set the stretch $d$ and turning angle $\theta$ of the chain so that $\sin\theta = \frac{1}{d} > \frac{1-c}{d}$. This will imply that the final configuration is a mutually visible configuration.

\begin{lemma} \label{lemma3}
 Consider a chain of stretch $d$ and turning angle $\theta$ such that $\sin\theta = \frac{1}{d}$. Let $\sigma$ denote the distance between the tip of the chain and the nearest point of the base chain. Then $d = \frac{2\sigma^2}{\sqrt{4\sigma^2 - 1}}$.
\end{lemma}

\begin{proof}
 Note that $\sigma = d \cos\frac{\theta}{2}$. We also have $\sin\theta = \frac{1}{d}$ $\implies 2\sin\frac{\theta}{2}\cos\frac{\theta}{2} = \frac{1}{d}$ $\implies \sin\frac{\theta}{2} = \frac{1}{2\sigma}$. Now, $\sin^2\frac{\theta}{2} + \cos^2\frac{\theta}{2} = 1$ $\implies \frac{\sigma^2}{d^2} + \frac{1}{4\sigma^2} = 1$ $\implies d = \frac{2\sigma^2}{\sqrt{4\sigma^2 - 1}}$. 
\end{proof}

\section{Algorithm} \label{section 5}


In this section, we shall describe our main algorithm. A set of $n \geq 3$ robots are arbitrarily placed on the plane. Initially, the lights of all robots are set to \texttt{off}. The size of the swarm, $n$, is not known to the robots. However, an upper bound $D$ on the horizontal width of the initial configuration is known in advance. This means that the horizontal distance between no two robots in the initial configuration is more than $D$.    

The execution of our algorithm can be divided into two stages: first, a leader is elected, and then the mutually visible chain is built. The execution from the beginning up till a leader is elected will be called Stage 1. The rest will be called Stage 2. A robot, however, may not be aware of the current stage of the process due to its incomplete view of the configuration. Our algorithm uses 7 colors, namely \texttt{off}, \texttt{defeated}, \texttt{leader}, \texttt{subordinate}, \texttt{no space}, \texttt{expand} and \texttt{final}. The last 5 colors can be found only in Stage 2. So we call these colors as \emph{Stage 2 colors}.

\subsection{Leader Election}

We shall first describe the process of electing the leader. We want to have a unique robot having its light set to the color \texttt{leader}. We also want to maintain a distance between this robot and the rest of the swarm. This space will be used in the next stage for building the base chain. 

Our plan is to have a single southmost robot in the configuration and elect it as the leader. For this, we need to have a robot that is sure of being the southmost. Recall that initially, all robots have color \texttt{off}. If such a robot can see a robot on its south, then it will change its color to \texttt{defeated}. It will also do the same if it finds a robot on the same horizontal line on its east. Robots with color \texttt{defeated} will not further take part in the leader election procedure. Now consider the case where a robot with color \texttt{off} does not see any robot to its south and also does not find any robot on the same horizontal line on its east. If it is not sure of being southmost, then it will move 2 units to the south. It can be shown that these movements will ensure that we have a robot that is sure of being the southmost. Now, we shall not ask this robot to become the leader immediately by changing its color to \texttt{leader}. We want to have the leader at least 10 units to the south of the rest of the swarm. As mentioned earlier, we need this space for the execution of Stage 2. So the robot may need to move further southwards so that it is $\geq 10$ units to the south of the rest. However, even after $\geq 10$ units of separation is achieved, there might still exist false southmost robots in the swarm that will move southward. We do not want such movements to occur after the leader is elected. It can be shown that if we can achieve $\geq \frac{D}{\sqrt{3}}$ units of separation, then there will not be any such unwanted movements in the future. So we shall ask the robot to move southwards until it is $\geq$ max $\{10, \frac{D}{\sqrt{3}}\}$ units to the south of the rest. When this separation is attained, the robot will change its color to \texttt{leader}.

   \subsection{Formation of the Mutually Visible Chain}\label{sec: chain formation}

   Let $r_L$ denote the elected leader. The leader $r_L$ will not make any moves in the remainder of the algorithm. Now we can fix a coordinate system with the origin at the center of $r_L$, the positive direction of the $X$-axis along the east and the positive direction of the $Y$-axis along the north. Any robot that can see the leader $r_L$ will agree with this coordinate system. For any number $d$, $L_d$ will denote the horizontal line $Y = d$.     
   \begin{figure}[thb!]
     \centering
     \begin{subfigure}[b]{0.24\linewidth}
         \centering
         \includegraphics[width=\linewidth]{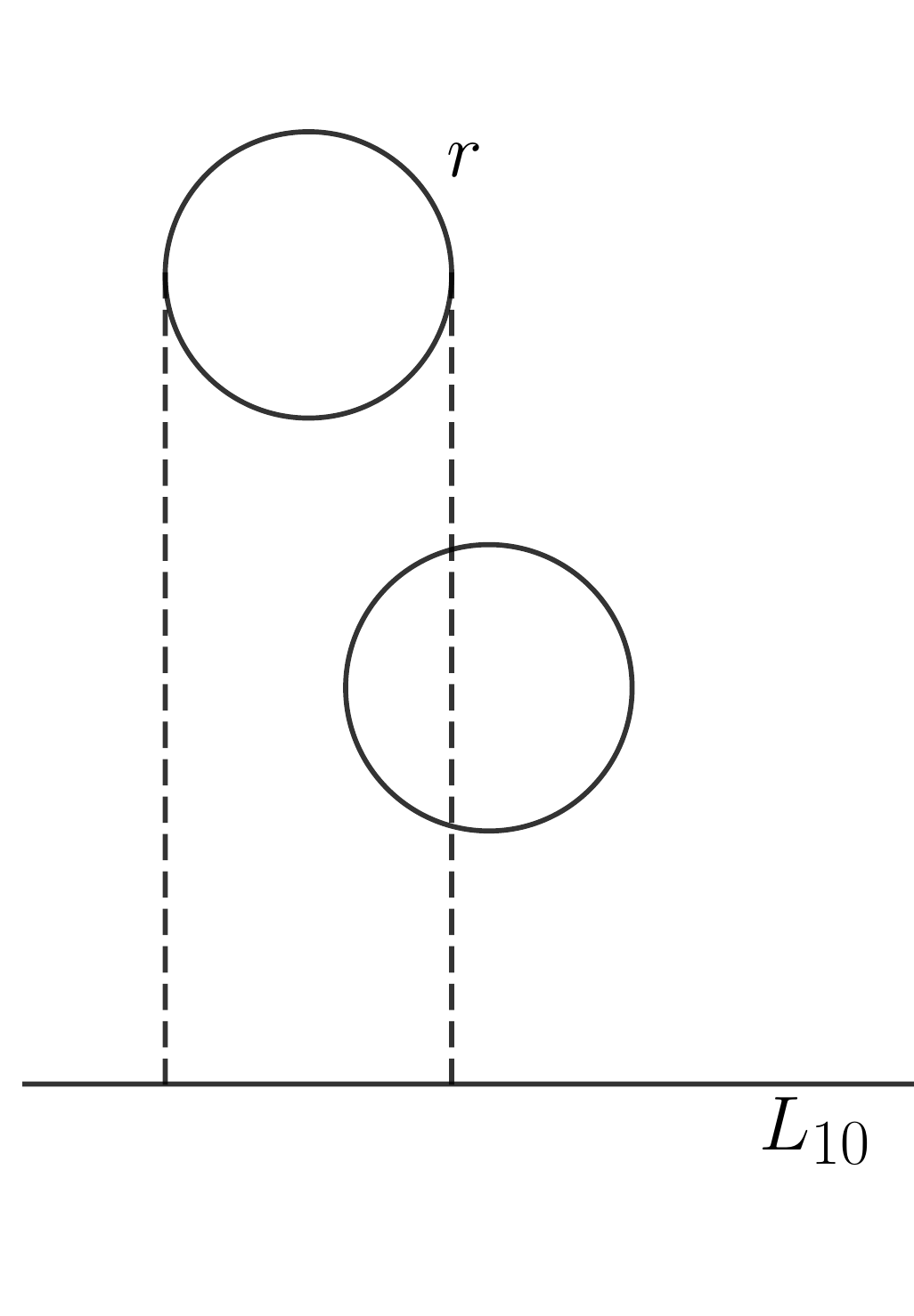}
         \caption{}
     \end{subfigure}
     \hspace{30pt}
     \begin{subfigure}[b]{0.24\linewidth}
         \centering
         \includegraphics[width=\linewidth]{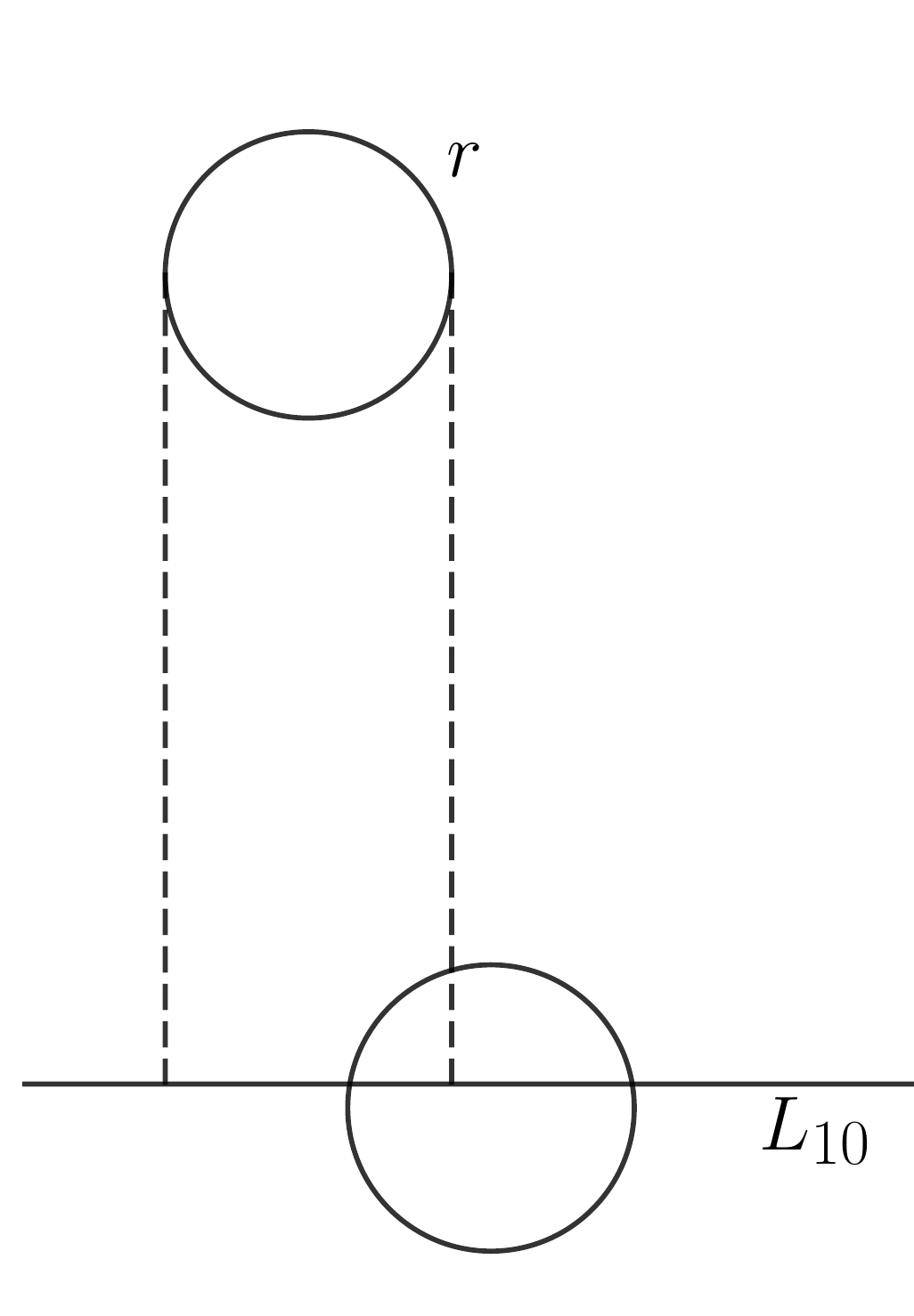}
         \caption{}
     \end{subfigure}
\caption[Short Caption]{Situations where $r$ cannot move to $L_{10}$ because the vertical route to $L_{10}$ has obstructions.}
\label{fig: block}
\end{figure}

   If an \texttt{off} or \texttt{defeated} colored robot sees any robot with Stage 2 color, then it will change its color to \texttt{subordinate}. So the robots with color \texttt{subordinate} are aware that a leader has been elected, and we have entered the next stage of the algorithm. Now the first task is to form the base chain. We want this process to be sequential, i.e., we want the robots to come one by one and place themselves on the base chain. However, sequentializing the process is a bit tricky. For example, if we ask the southmost \texttt{subordinate} colored robot to go first, then we can still have parallel movements due to the false southmost robots. So let us now describe how we sequentialize the movements. First, we ask the robots to come to $L_{10}$. In particular, if a robot $r$ with color \texttt{subordinate} is to the north of $L_{10}$ and the vertical path to $L_{10}$ has no obstructions, then it will move towards $L_{10}$. Fig. \ref{fig: block} shows situations where the vertical path of a robot to $L_{10}$ is obstructed. Now we want to move one robot from $L_{10}$ to $L_{8}$. For this, if there are multiple robots on ${L}_{10}$, then we ask the one that is closest to the $Y$-axis to move 2 units southwards to reach $L_8$. In case there are two robots closest to the $Y$-axis, then the one on the east will move. Also, the robot will move only if it finds no robot on $L_8$. However, in spite of this strategy, we may have multiple robots on $L_{8}$. An example of this is shown in Fig. \ref{fig: 2l8}. So there can be multiple robots on $L_{8}$. If we try to move south one among these robots in a similar fashion, then we can again face the same problem and have multiple robots on $L_{6}$. To resolve this, we shall ask the robots to leave via the $Y$-axis, i.e., a robot may need to move along $L_8$ towards the $Y$-axis and then leave $L_8$ only when it is touching the $Y$-axis. This strategy will ensure that at any time, we have at most one robot on $L_6$. This is because a robot on $L_{8}$ touching the $Y$-axis will be able to detect if there is a robot on $L_{6}$ as that robot also must be touching the $Y$-axis. However, notice that there is a possibility of a collision between a robot $r$ moving horizontally along $L_{8}$ and another robot $r'$ moving southwards from ${L}_{10}$ to $L_{8}$. This can happen in a situation where $r'$ cannot see $r$ and thinks that $L_8$ is empty. In order to avert such a collision, $r$ will do the following. The robot $r$ will identify all the robots on $L_{10}$ with which it can collide if that robot moves southwards. Notice that $r$ will be able to find all such robots, as there are no robots on $L_{8}$ in the portion between the $Y$-axis and itself that can obstruct its view. If $r$ finds no such robots on ${L}_{10}$, then there is no problem, and it will move towards the $Y$-axis as required (see Fig. \ref{fig27}). If $r$ finds only one such robot, say $r'$, then clearly $r'$ can also see $r$ (see Fig. \ref{fig28}). Hence, $r'$ will not move in this case, and thus, there is no possibility of a collision. So in this case, $r$ will simply ignore $r'$ and move towards the $Y$-axis. Now consider the case where $r$ has found multiple robots on $L_{10}$. In this case, only the one that is closest to the $Y$-axis, say $r'$, is to be considered as only this robot may decide to move southwards (see Fig. \ref{fig29} and \ref{fig30}). Let $d$ be the horizontal distance between $r$ and $r'$. Notice that if $d$ is small enough, i.e., $r$ and $r'$ are close enough, then $r'$ will be able to see $r$ and hence $r'$ will not move. More precisely, it can be shown that if $d \leq 5$, then $r'$ will be able to see $r$. Intuitively, the worst case of obstruction occurs when a robot is touching $r'$ and the radius of the camera is zero, i.e., a point camera (see Fig. \ref{fig31}). Formally, the obstructed region in any other case will be contained in the obstructed region for this case (see Fig. \ref{fig32}). Hence, it suffices to argue for this case. Using simple calculations, one can verify that $r$ will be visible to $r'$ in this worst case situation if $d \leq 5$ (see Fig. \ref{fig31}). So if $d \leq 5$, $r$ will ignore $r'$ and move towards the $Y$-axis. Otherwise, if $d > 5$, then in order to avert a collision, $r$ will not move more than $d-2$ units.


\begin{figure}[thb!]
     \centering
     \begin{subfigure}[b]{0.49\linewidth}
         \centering
         \includegraphics[width=\linewidth]{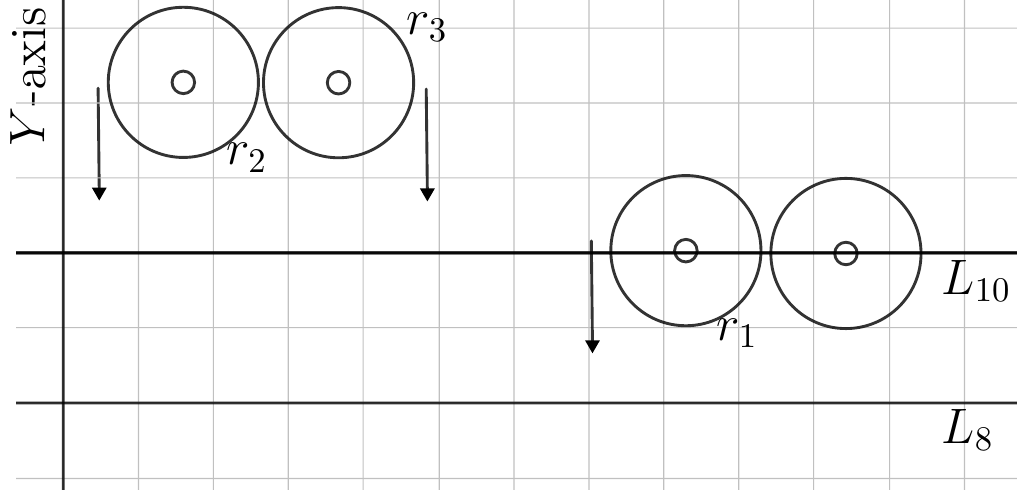}
         \caption{}
     \end{subfigure}
     \hfill
     \begin{subfigure}[b]{0.49\linewidth}
         \centering
         \includegraphics[width=\linewidth]{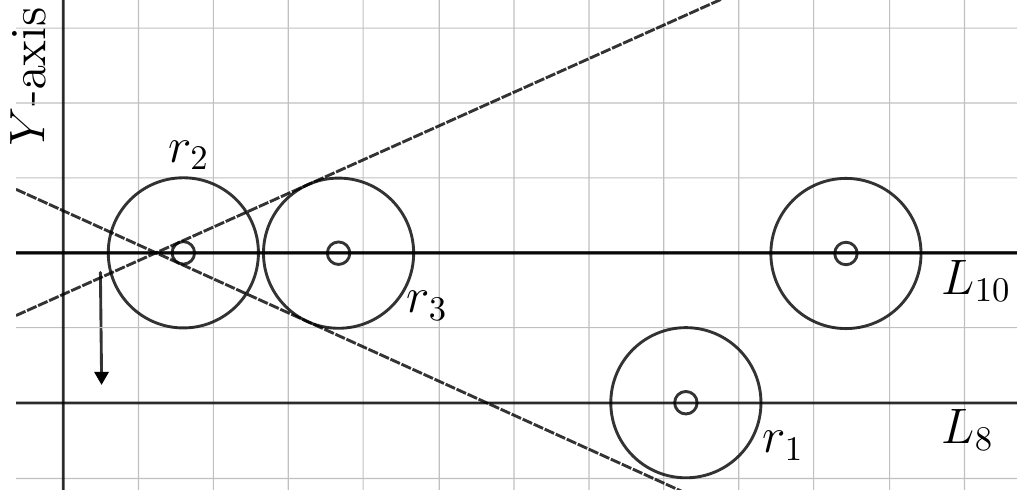}
         \caption{} \label{fig22}
     \end{subfigure}
     \\
     \begin{subfigure}[b]{0.49\linewidth}
         \centering
         \includegraphics[width=\linewidth]{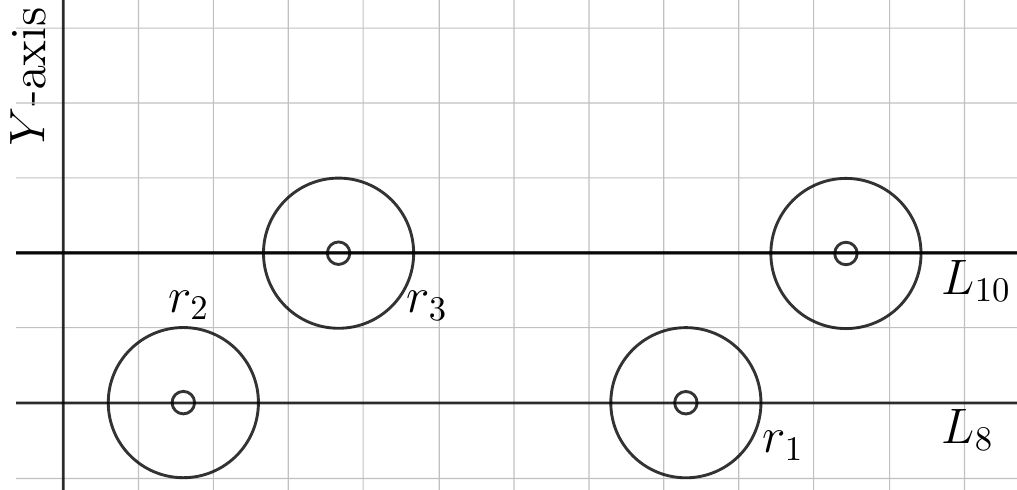}
         \caption{} 
     \end{subfigure}
\caption[Short Caption]{A situation where we have multiple robots on $L_8$. (a) The robots $r_2$ and $r_3$ decide to move to $L_{10}$. Also, $r_1$ being closest to the $Y$- axis on $L_{10}$, decides to move to $L_8$. (b) $r_2$ does not see any robot on $L_8$, and hence decides to move to $L_8$. $r_1$ is not activated. (c) We have multiple robots on $L_8$. }
\label{fig: 2l8}
\end{figure} 


\begin{figure}[thb!]
     \centering
     \begin{subfigure}[b]{0.49\linewidth}
         \centering
         \includegraphics[width=\linewidth]{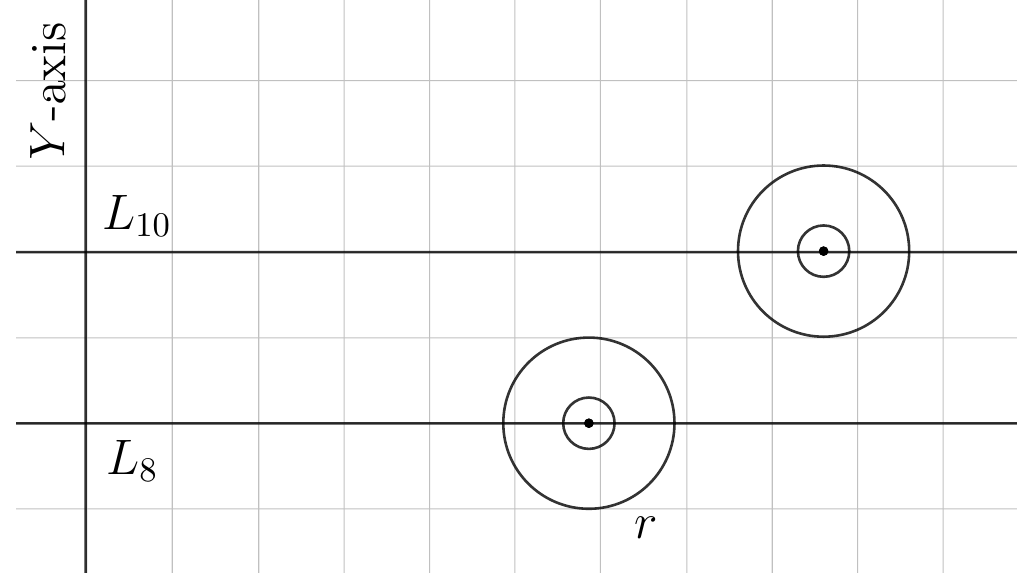}
         \caption{}\label{fig27}
     \end{subfigure}
     \hfill
     \begin{subfigure}[b]{0.49\linewidth}
         \centering
         \includegraphics[width=\linewidth]{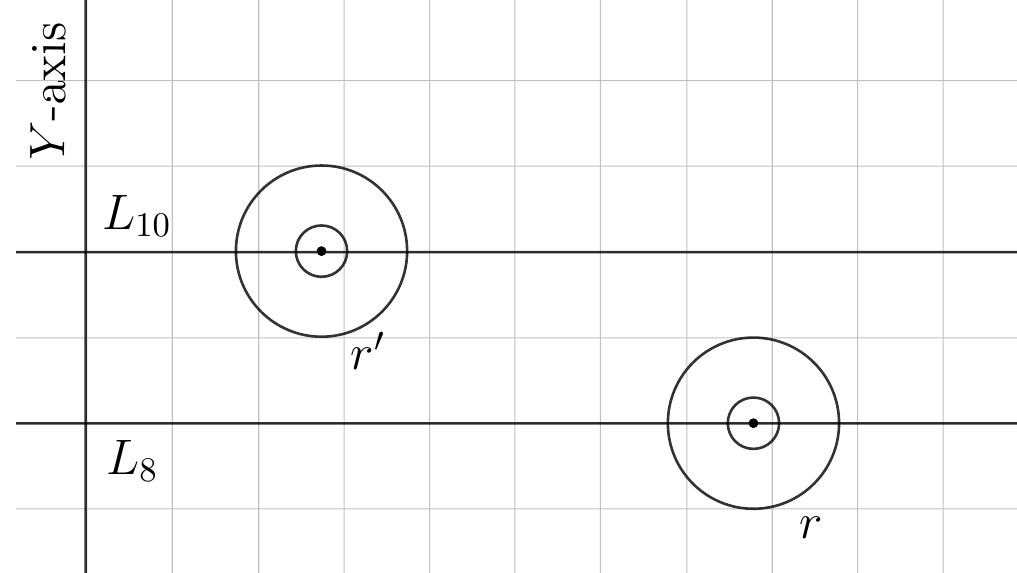}
         \caption{}\label{fig28}
     \end{subfigure}
     \\
     \begin{subfigure}[b]{0.49\linewidth}
         \centering
         \includegraphics[width=\linewidth]{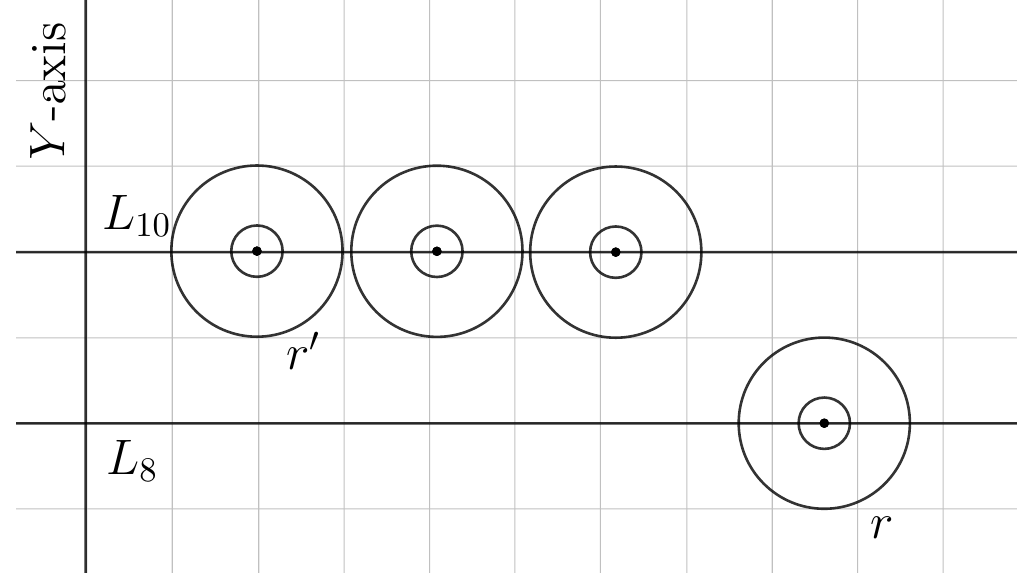}
         \caption{} \label{fig29}
     \end{subfigure}
     \hfill
     \begin{subfigure}[b]{0.49\linewidth}
         \centering
         \includegraphics[width=\linewidth]{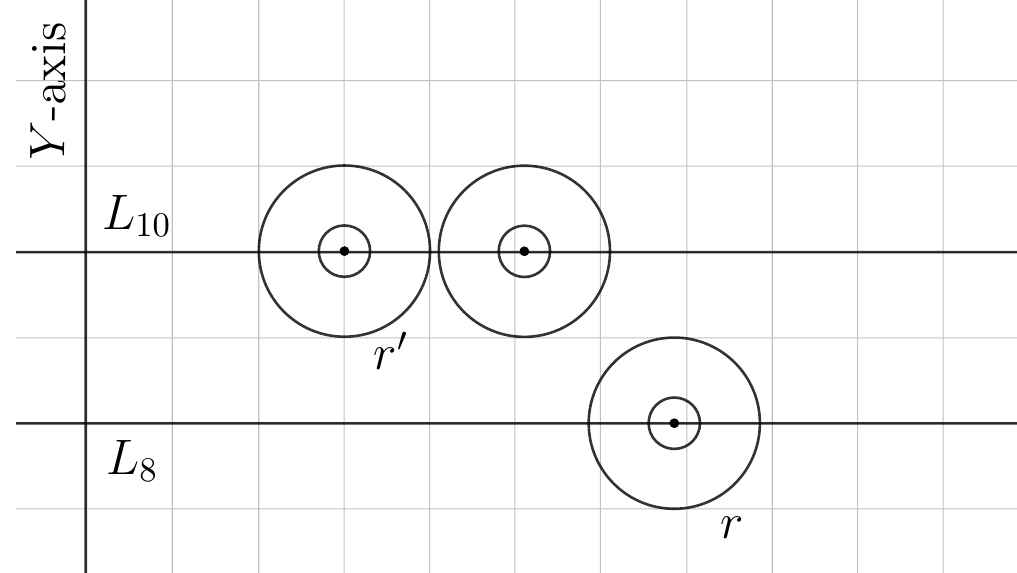}
         \caption{}\label{fig30}
     \end{subfigure}
     \\
     \begin{subfigure}[b]{0.49\linewidth}
         \centering
         \includegraphics[width=\linewidth]{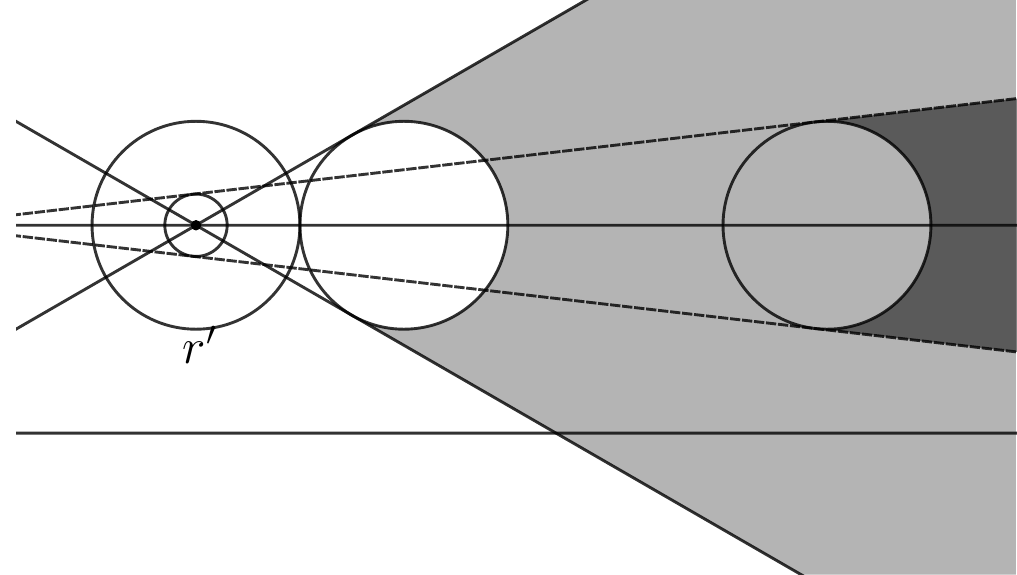}
         \caption{}\label{fig32}
     \end{subfigure}
     \hfill
     \begin{subfigure}[b]{0.49\linewidth}
         \centering
         \includegraphics[width=\linewidth]{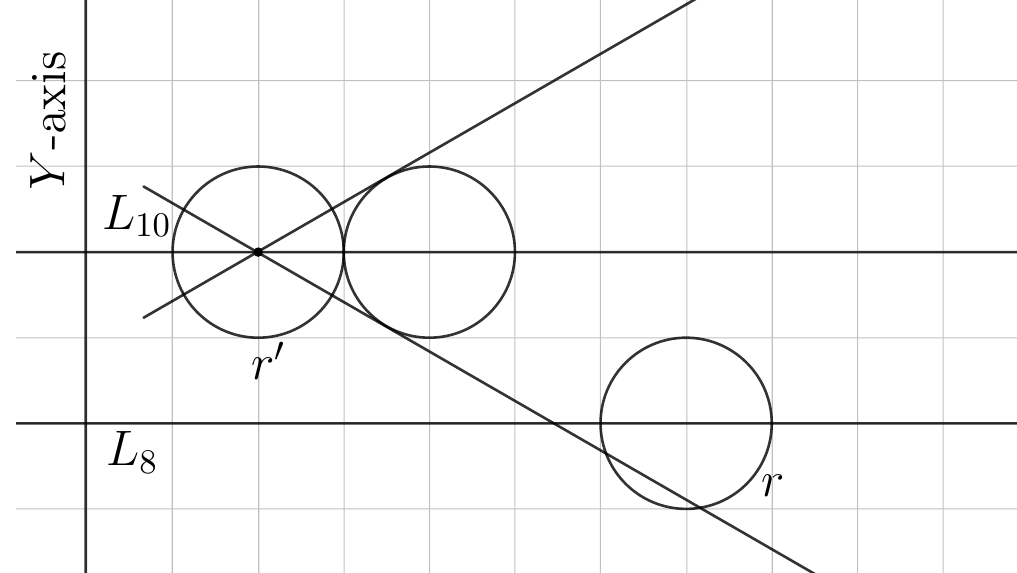}
         \caption{}\label{fig31}
     \end{subfigure}
\caption[Short Caption]{(a) There are no robots on $L_{10}$ with which $r$ can collide. (b) There is no possibility of a collision between $r'$ and $r$. This is because $r'$ can see $r$ and hence will not move south. (c) There is a possibility of a collision between $r'$ and $r$. This is because $r'$ may not be able to see $r$ as their horizontal separation is $> 5$ units. (d) There is no possibility of a collision between $r'$ and $r$. This is because $r'$ can see $r$ as their horizontal separation is $\leq 5$ units. (e) The worst case of obstruction occurs when a robot is touching $r'$ and the radius of the camera is zero. This is because the obstructed region in any other case will be contained in the obstructed region for this case. (f) In the worst case of obstruction, $r$ is visible to $r'$ if their horizontal separation is $\leq 5$.}
\label{fig: collision}
\end{figure}

   
  Our strategy ensures that the robots come to $L_6$ sequentially. When a robot $r$ comes to $L_6$, it first checks if there is any robot in the region between $L_6$ and $L_0$. If there is such a robot, then $r$ will do nothing. Otherwise, it will calculate the position on the base chain on $L_0$ where it has to place itself. Recall that the robots have to alternatively go to the east and west branches of the base chain. So if $r$ finds that both branches have $k$ robots, then it has to go to the east branch, while if it finds that the east branch has $k$ robots and the west branch has $k-1$ robots, then it has to go to the west branch. The robot can compute its position by calculating the stretch of the corresponding mutually visible chain from the distance between $r_L$ and its nearest robot on the base chain using Lemma \ref{lemma3}. If $r$ is the first robot on the branch, we set its position to be 4 units from $r_L$. If its computed position is $< 2$ units from the last robot on its branch of the base chain, then it means that there is no more space on the base chain, and an expansion is required. First, consider the case where there is room for $r$ on the base chain. In this case, it will first move to $L_2$, then move along $L_2$ to vertically align with its computed position on the base chain. Finally, it moves towards the south to place itself on the base chain on $L_0$ (see Fig. \ref{fig33} and \ref{fig34}).

  Now consider the case where $r$ finds that there is no space, and so the base chain has to be expanded. In this case, $r$ will change its color to \texttt{no space}. When the leader $r_L$ (with color \texttt{leader}) finds  $r$ with color \texttt{no space}, it will change its color to \texttt{expand}. The plan now is that all the robots of the base chain, except the leader, will move to $L_4$ and place themselves on the $X$-coordinates of the expanded chain. In this case, $L_2$ will be used for horizontal movements. The leader $r_L$ will be able to detect when all the robots have placed themselves on $L_4$ at their respective positions. While the leader detects, it will change its color back to \texttt{leader}. Finally, $r$ will also change its color to \texttt{subordinate} and the robots on $L_4$ will move vertically to $L_0$. This completes the expansion of the base chain. We describe this part of the algorithm in more detail in the following.


  \begin{figure}[thb!]
     \centering
     \begin{subfigure}[b]{0.49\linewidth}
         \centering
         \includegraphics[width=\linewidth]{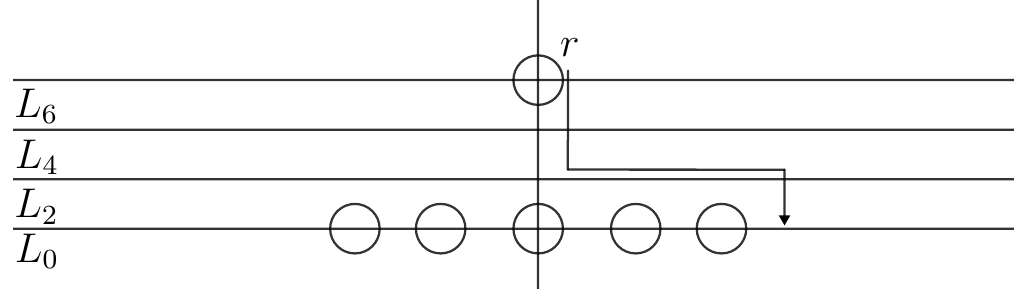}
         \caption{}
         \label{fig33}
     \end{subfigure}
     \hfill
     \begin{subfigure}[b]{0.49\linewidth}
         \centering
         \includegraphics[width=\linewidth]{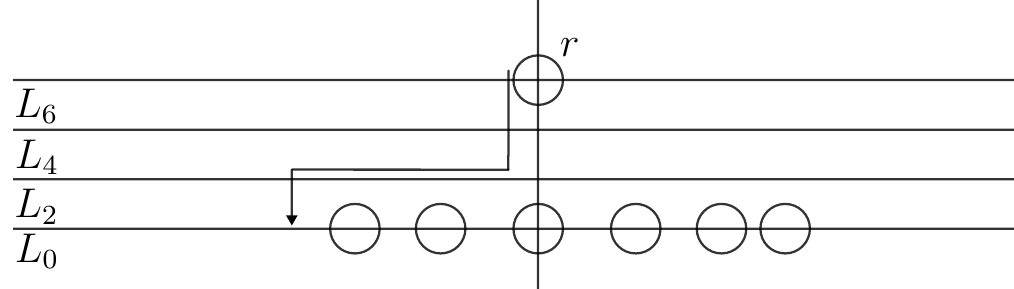}
         \caption{}
         \label{fig34}
     \end{subfigure}
     \\
     \begin{subfigure}[b]{0.49\linewidth}
         \centering
         \includegraphics[width=\linewidth]{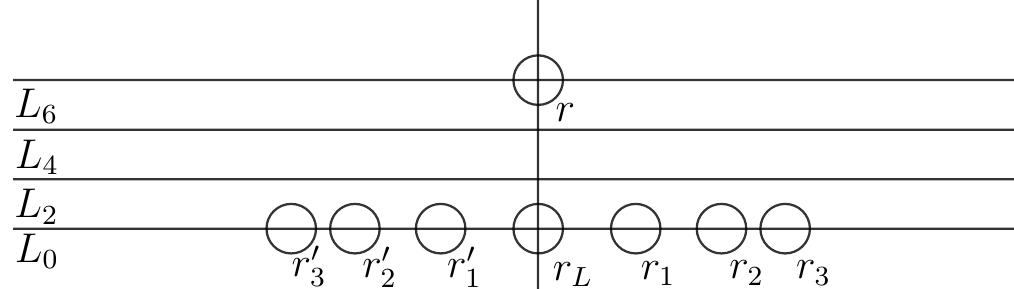}
         \caption{}
         \label{fig35}
     \end{subfigure}
     \hfill
     \begin{subfigure}[b]{0.49\linewidth}
         \centering
         \includegraphics[width=\linewidth]{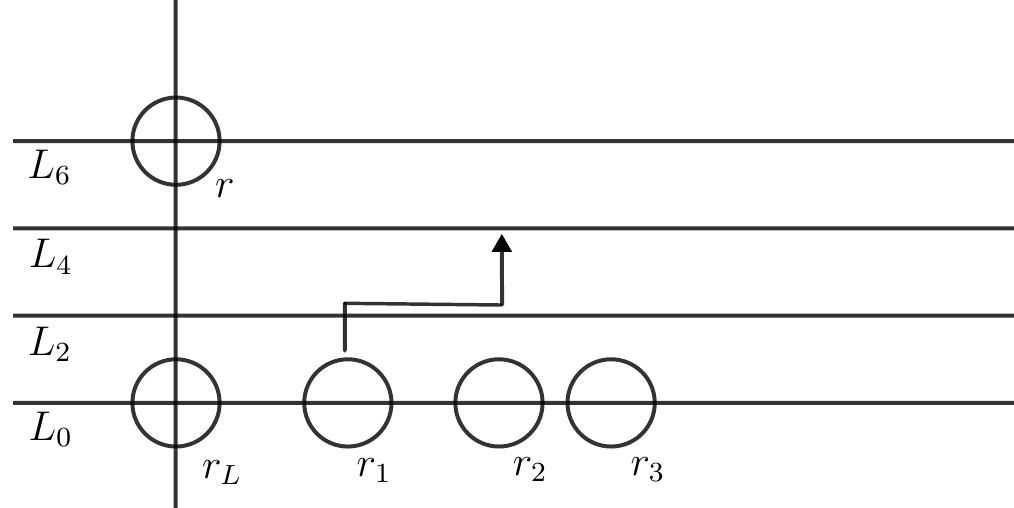}
         \caption{}
     \end{subfigure}
     \\
     \begin{subfigure}[b]{0.49\linewidth}
         \centering
         \includegraphics[width=\linewidth]{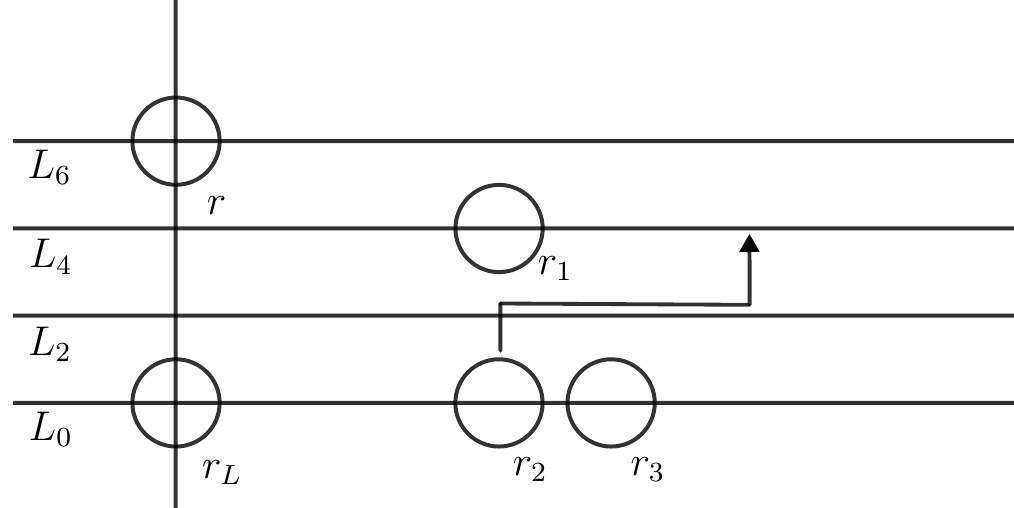}
         \caption{}
     \end{subfigure}
     \hfill
     \begin{subfigure}[b]{0.49\linewidth}
         \centering
         \includegraphics[width=\linewidth]{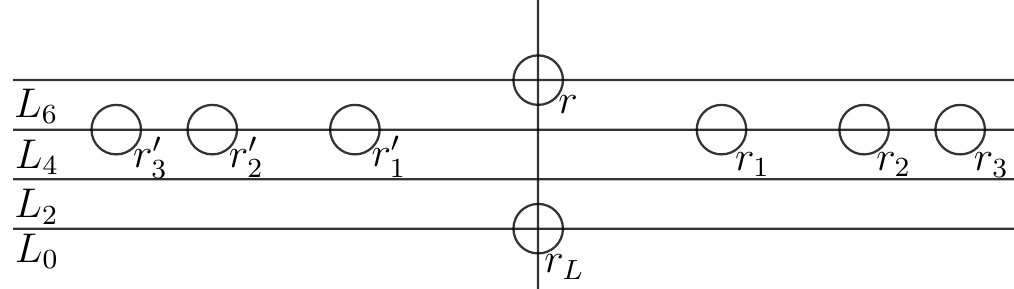}
         \caption{}
         \label{fig38}
     \end{subfigure}
\caption[Short Caption]{(a) $r$ will go to the east branch of the base chain. (b) $r$ will go to the west branch of the base chain. (c) $r$ changes its color to \texttt{no space} and then $r_L$ changes its color \texttt{expand}. (d) The movement of $r_1$ from $L_0$ to $L_4$. (e) The movement of $r_i, i>1$ from $L_0$ to $L_4$. (f) After all robots move to $L_4$, $r_L$ changes its color to \texttt{leader} and then $r$ changes its color to \texttt{subordinate}.}
\label{}
\end{figure}

    \begin{figure}[t!]
     \centering
     \begin{subfigure}[b]{0.47\linewidth}
         \centering
         \includegraphics[width=\linewidth]{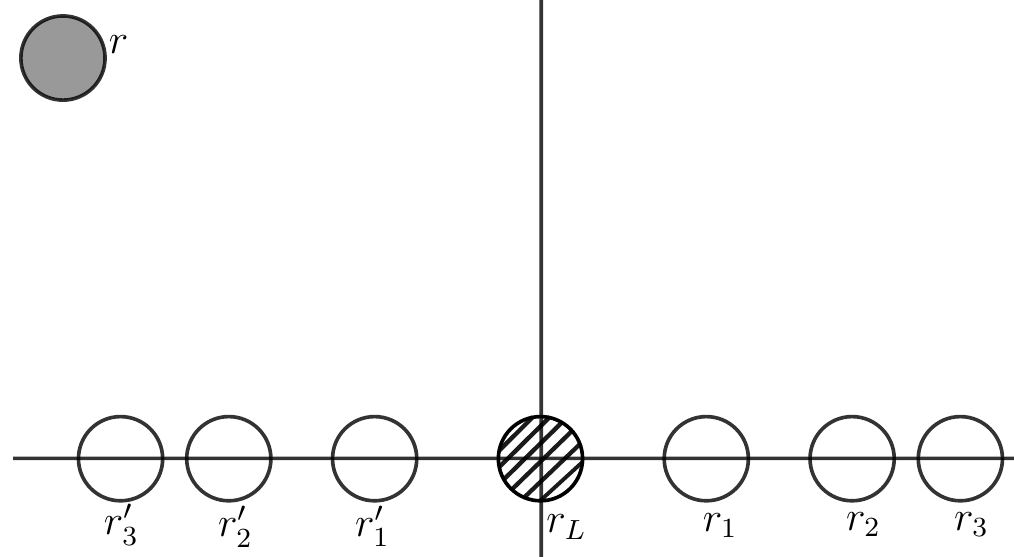}
         \caption{}
         \label{fig39}
     \end{subfigure}
     \hfill
     \begin{subfigure}[b]{0.47\linewidth}
         \centering
         \includegraphics[width=\linewidth]{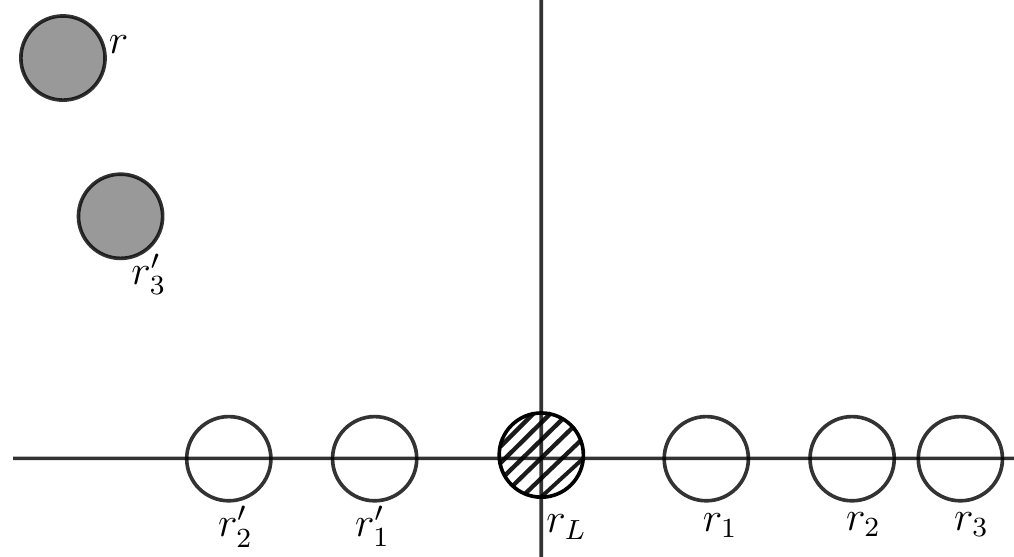}
         \caption{}
         \label{fig40}
     \end{subfigure}
     \\
     \begin{subfigure}[b]{0.47\linewidth}
         \centering
         \includegraphics[width=\linewidth]{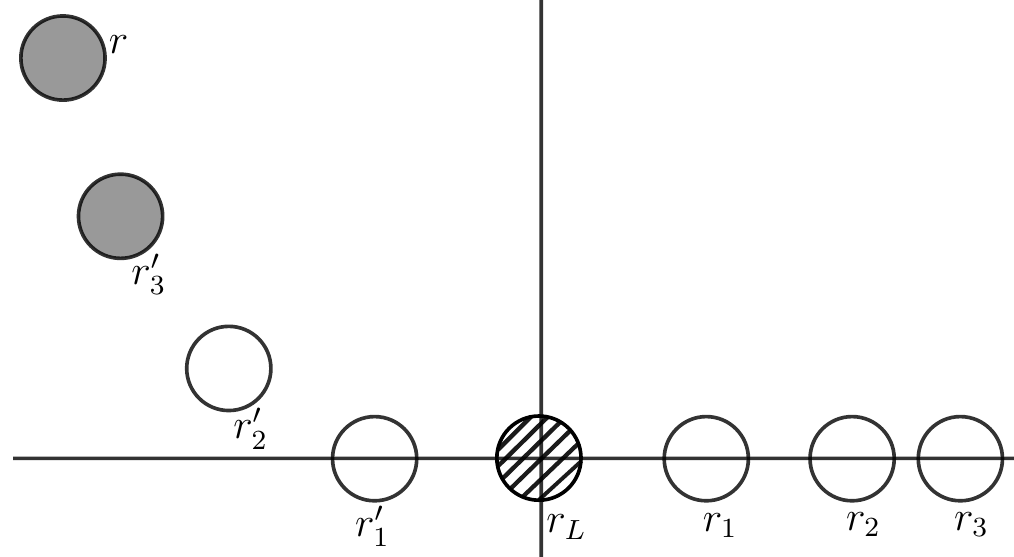}
         \caption{}
         \label{}
     \end{subfigure}
     \hfill
     \begin{subfigure}[b]{0.47\linewidth}
         \centering
         \includegraphics[width=\linewidth]{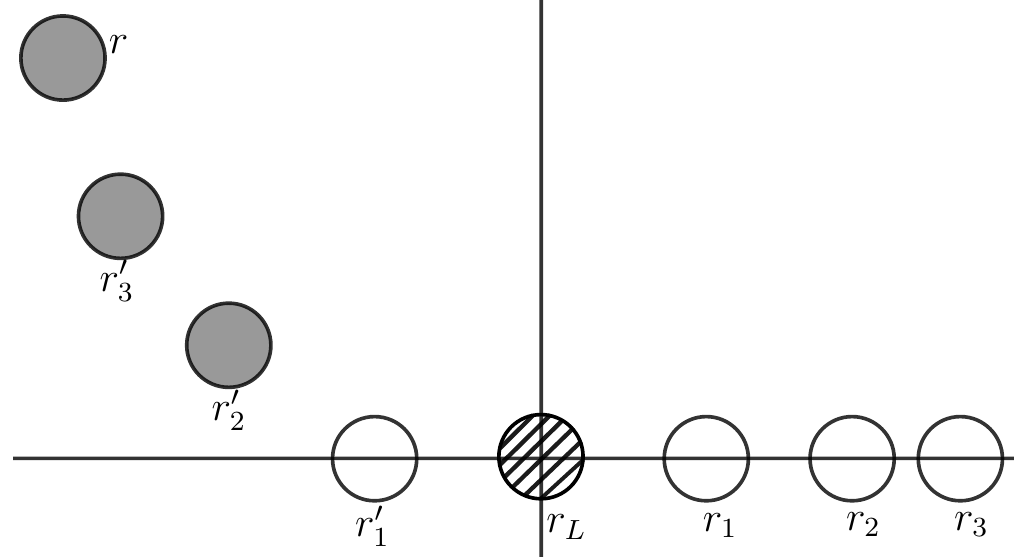}
         \caption{}
     \end{subfigure}
     \\
     \begin{subfigure}[b]{0.47\linewidth}
         \centering
         \includegraphics[width=\linewidth]{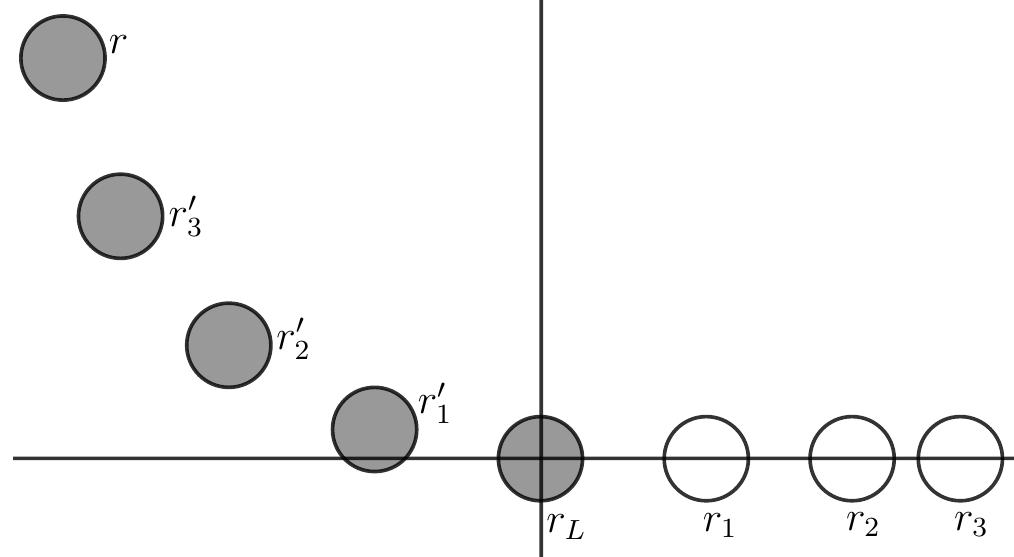}
         \caption{}
         \label{fig41}
     \end{subfigure}
     \hfill
     \begin{subfigure}[b]{0.47\linewidth}
         \centering
         \includegraphics[width=\linewidth]{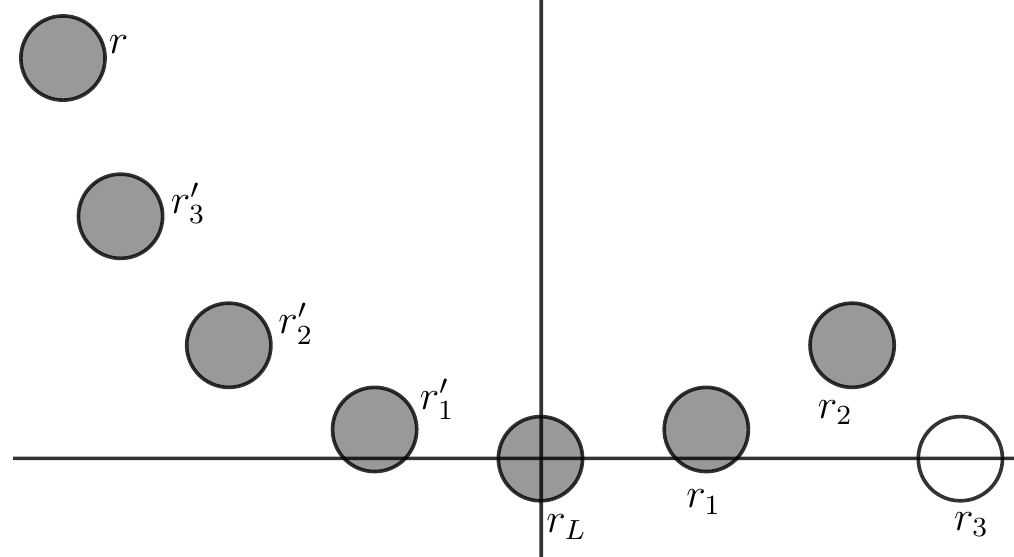}
         \caption{}
         \label{fig42}
     \end{subfigure}
     \\
     \begin{subfigure}[b]{0.47\linewidth}
         \centering
         \includegraphics[width=\linewidth]{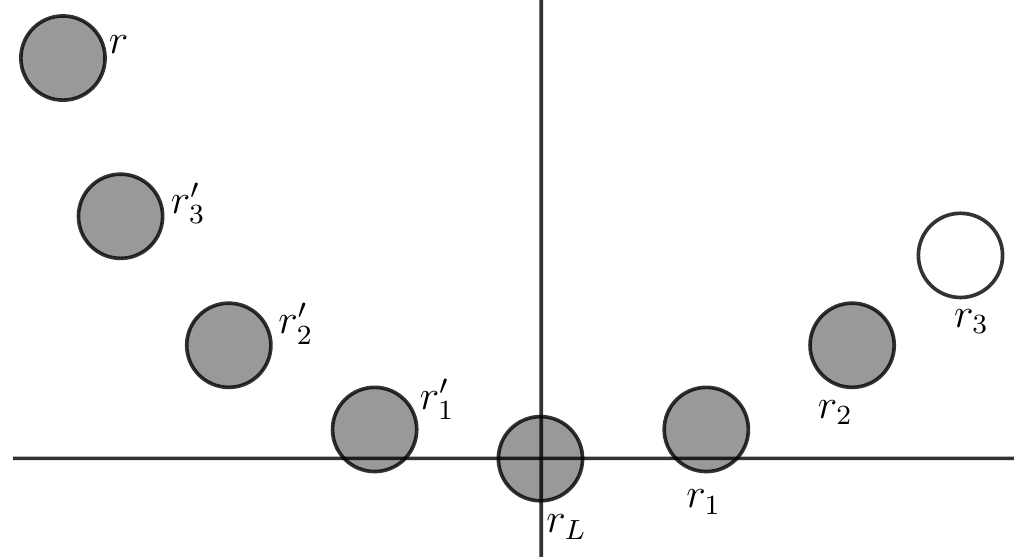}
         \caption{}
         \label{fig43}
     \end{subfigure}
     \hfill
     \begin{subfigure}[b]{0.47\linewidth}
         \centering
         \includegraphics[width=\linewidth]{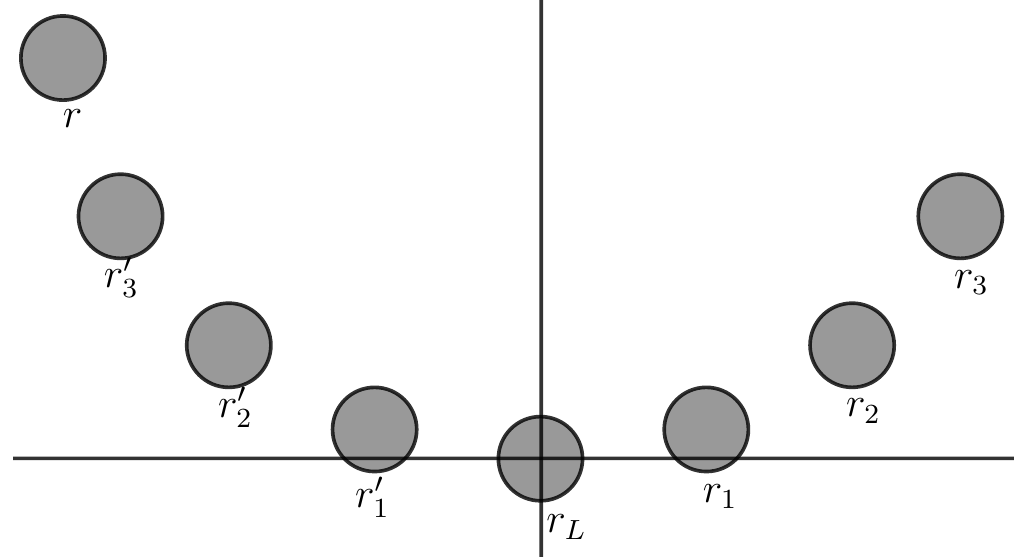}
         \caption{}
         \label{fig46}
     \end{subfigure}
\caption[Short Caption]{The shaded, striped and white robots respectively indicate robots with color \texttt{final}, \texttt{leader} and \texttt{subordinate}. (a) The last robot $r$ goes to its final position and changes its color to \texttt{final}. (b) After $r'_3$ changes its color \texttt{final}, it is $r'_2$'s turn to move. (c)  $r'_2$ first moves 2 units north so that it is able to see $r_L$. (d) Then $r'_2$ goes to its final position and changes its color to \texttt{final}. (e) After all robots of the west branch have moved and changed their colors to \texttt{final}, $r_L$ also sets its color to \texttt{final}. (f) After $r_2$ changes its color to \texttt{final}, it is $r_3$'s turn to move. (g) $r_3$ first moves 2 units to the north of $r_2$ so that it is able to see $r'_3$. (h) Then $r_3$ aligns itself with $r'_3$ and changes its color to \texttt{final}.}
\label{}
\end{figure}


  When a robot on $L_0$ with color \texttt{subordinate} sees the leader with color \texttt{expand}, it understands that the base chain is being expanded. Let us denote the robots on the base chain as $r_1$, $r_2, \ldots$ and $r_1'$, $r_2', \ldots$ as shown in Fig. \ref{fig35}.  So $r_i$ (resp. $r_1'$) will be able to see the leader (with color \texttt{expand}) when  $r_1, \ldots, r_{i-1}$ (resp. $r'_1, \ldots, r'_{i-1}$) have left $L_0$. When such a robot on $L_0$ is able to see $r_L$ with color \texttt{expand}, it will check if there is any robot on $L_2$ in its half. If there is none, then it will move to $L_2$. Now consider a robot $r_i$ or $r_i'$ on $L_2$. It can understand that it is the $i$th robot on the chain as it can see $i-1$ robots on $L_4$ on its half. It now has to move horizontally and align itself with its position in the expanded base chain. We explain how $r_i$ or $r_i'$ will calculate its position on the expanded base chain. For this, the robot just needs to find the stretch of the expanded chain. If $i \leq 2$, then it can see $r_1$ or $r_1'$ on $L_4$ and thus, they can calculate the stretch of the expanded chain using Lemma \ref{lemma3}. When $i = 1$, the robot has to set the stretch of the new chain. The stretch has to be larger than the stretch of the current chain. We also have to ensure that both $r_1$ or $r_1'$ choose the same stretch. This can be easily achieved by asking $r_1$ (resp. $r_1'$) to align with $r_2$ (resp. $r_2'$) on $L_2$ (i.e., the distance of the first robot from the leader on the new base chain will be equal to the distance of the second robot from the leader on the old base chain). The leader $r_L$ will be able to detect when all the robots have placed themselves on $L_4$ (see Fig.\ref{fig38}). Then it will change its color back to \texttt{leader}. When $r$ with color \texttt{no space} sees that $r_L$ has color \texttt{leader} and robots are queued on $L_4$, then it will understand that all robots from $L_0$ (except $r_L$) have moved to $L_4$ and then it will change its color back to \texttt{subordinate}. Now for a robot $r_i$ on $L_4$, it will decide to move back to $L_0$ when it can see $r_L$ with color \texttt{leader} and a robot on $L_6$ with color \texttt{subordinate}.

  Now let us consider the special case where a robot $r$ finds that all the other robots are on the base chain. We call $r$ the \textit{last robot}. In this case, $r$ will not try to go to the base chain. Here, $r$ will compute its final position on the mutually visible chain and go there directly. We shall ask the last robot $r$ to always go to the west branch of the chain (see Fig. \ref{fig39}). When $r$ reaches its final position, it will change its color to \texttt{final}. Let $r_1$, $r_2, \ldots$ and $r_1'$, $r_2', \ldots$ be the robots on the base chain as shown in Fig. \ref{fig39}. Now the robots $r_k'$, $r_{k-1}', \ldots, r_1'$ will have to go to their final positions sequentially, from west to east. After $r, r_k', \ldots, r_{i+1}'$ have gone to their final positions and changed their colors to \texttt{final},  $r_i'$ will find that all robots on its west have color \texttt{final} (see Fig. \ref{fig40}). Then $r_i'$ will decide to move. Notice, however, that $r_i'$ will not be able to see $r_L$ unless $i = 1$ and hence cannot compute its final position. In this case, $r_i'$ will first move 2 units north. Then it will be able to see both $r_1'$ and $r_L$, and thus can infer the stretch of the chain using Lemma \ref{lemma3}. So, then it can compute its final position and move towards it. When all the robots on the west branch have placed themselves at their final positions and changed their colors to \texttt{final}, the leader $r_L$ will also change its color to \texttt{final} (see Fig. \ref{fig41}). Now the east branch robots $r_1$, $r_2, \ldots, r_k$ will have to go to their final positions sequentially, from west to east. After $r_1, \ldots, r_{i-1}$ have gone to their final positions and changed their colors to \texttt{final}, $r_i$ will find that all robots on its west have color \texttt{final} (see Fig. \ref{fig42}). Then $r_i$ will decide to move. The robot $r_{i}$ will first move vertically 2 units to the north of $r_{i-1}$ (see Fig. \ref{fig43}). From there, $r_i$ will be able to see $r_{i-1}'$ which is horizontally aligned with $r_{i-1}$. The robot $r_i$ will also be able to identify $r_{i}'$ which is immediately to the east of $r_{i-1}'$. The final position of $r_i$ will be horizontally aligned with $r_i'$. So $r_i$ will then align itself with $r_i'$ and change its color to \texttt{final}. With this strategy, all the robots will be able to place themselves on the mutually visible chain (see Fig. \ref{fig46}).




\clearpage

\section{Correctness} \label{sec 6}

  \begin{theorem}\label{th 1}
  There is a round $T_1$ when a leader is elected.
 \end{theorem}

\begin{proof}



  In a configuration with no leader, there are two types of robots: robots with color \texttt{off} and robots with color \texttt{defeated}. Let us call the \texttt{off} colored robots live robots and \texttt{defeated} colored robots defeated robots. For any live robot $r$ in a configuration with no leader, we associate a set $\mathcal{B}(r)$ in the following way. We call a robot $r'$ \emph{bad} with respect to $r$ if either $r'$ is on the same horizontal line as $r$, or $r'$ is to the north of $r$, but their vertical separation is $< 1-c$. According to this definition, if there is a false southmost robot in the configuration, then there must be a bad robot with respect to it, which is obstructing its southward view. For any live robot $r$, the set $\mathcal{B}(r)$ consists of 1) all bad robots with respect to $r$, 2) all live robots that are to the north of $r$, and 3) all robots that are to the south of $r$. Intuitively, the set $\mathcal{B}(r)$ contains all the robots that are either currently bad or may become bad in the future. Although the robots in 2) are not currently bad with respect to $r$, they may move southward and become so in the future. Similarly, a robot from 3) may become bad with respect to $r$ in the future if $r$ moves southward. Now let $\mathcal{B}$ denote the set of all pairs $(r,r')$ where $r$ is a live robot and $r' \in \mathcal{B}(r)$.

  We first show that the size of the set $\mathcal{B}$ never increases. Observe that the number of live robots in the configuration never increases. Also, for any live robot $r$, the robots outside $\mathcal{B}(r)$ are precisely the defeated robots that are $\geq (1-c)$ units to the north of $r$. Clearly, these robots will never enter $\mathcal{B}(r)$. Hence, it follows that $|\mathcal{B}|$ cannot increase.

  Now we show that the size of the set $\mathcal{B}$ gradually decreases. Consider any round $t$. Assume that there is no leader in the configuration. Consider the robot $r$ that is northmost among all live robots. If there are multiple such robots, then take the westhmost one. Let $t' \geq t$ be the earliest round when $r$ is active.

  \noindent\textbf{Case 1.} Suppose that $r$ can see a robot $r'$ on its south. Hence $(r,r') \in \mathcal{B}$. In this case, $r$ will change its color to \texttt{defeated}. Thus, $(r,r')$ is removed from $\mathcal{B}$ and hence $|\mathcal{B}|$ decreases.  
  
  \noindent\textbf{Case 2.} Suppose that the robot $r$ thinks that it is southmost and its southward view is obstructed by a robot $r'$ to its north. Clearly, the vertical distance between $r$ and $r'$ is $< (1-c)$ units. Hence $(r,r') \in \mathcal{B}$. It follows from the definition of $r$ that the color of $r'$ is \texttt{defeated}. In this case, $r$ will move 2 units south. Thus, $(r,r')$ is removed from $\mathcal{B}$ and hence $|\mathcal{B}|$ decreases.

  \noindent\textbf{Case 3.} Suppose that the robot $r$ thinks that it is southmost, but its southward view is obstructed by a robot $r'$ on the same horizontal line. Hence $(r,r') \in \mathcal{B}$. It follows from the definition of $r$ that  $r'$ is to the east of $r$. So $r$ will change its color to \texttt{defeated}. Thus, $(r,r')$ is removed from $\mathcal{B}$ and hence $|\mathcal{B}|$ decreases.

  \noindent\textbf{Case 4.} Suppose that the robot $r$ is sure of being southmost. But $r$ is also the northmost among all the live robots. This means that $r$ is the only live robot in the configuration. Also, all other robots have color \texttt{defeated} and they are $\geq (1-c)$ units to the north of $r$. So we have $|\mathcal{B}| = 0$.

  We have shown that, while there is no leader in the configuration, $|\mathcal{B}|$ gradually decreases. In particular, it implies that there exists a round when either 1) there is a leader or 2) $|\mathcal{B}| = 0$. In the first case, we are done. In the latter case, there is only one live robot $r$, which is surely southmost, and all the other robots are defeated. Since the defeated robots do not move, $r$ will be able to make $\geq$ max $\{10, \frac{D}{\sqrt{3}}\}$ units vertical separation from the rest and then become the leader by changing its color to \texttt{leader}. Hence, there is a round $T_1$ when we have a leader. 
  
  \end{proof}

 \begin{lemma}\label{lem: no false}
There are no \texttt{off} colored false southmost robots after the leader is elected at $T_1$.
\end{lemma}
  
\begin{proof}
  Let $\mathcal{S}$ be the smallest vertical strip containing the initial configuration. The horizontal width of $\mathcal{S}$ is $\leq D$. At any time during the execution of the algorithm, an \texttt{off} colored robot must be inside $\mathcal{S}$. This is because an \texttt{off} colored robot only moves vertically. Consider the set $\mathcal{O}$ of all \texttt{off} colored robots in the configuration at round $t$. Take any $r \in \mathcal{O}$. Let $t'$ be the earliest round $\geq T_1$ when $r$ becomes active. We claim that $r$ will see some robot to its south and hence will change its color. For the sake of contradiction, assume that $r$ is a false southmost robot at round $t'$. Now consider the region in $\mathcal{S}$ to the south of $r$ that is invisible to it. Since $r$ does not see any robot on its south, the vertical extent of this region will be maximum when $r$ is obstructed by a robot $r'$ on the same horizontal line at a distance 2 and the radius of the camera is 0 (i.e., a point camera). Such a situation is shown in Fig. \ref{fig: no_false}. The vertical extent of the obstructed region is $|DE|$. Since the triangles $\Delta ABC$ and $\Delta ADE$ are similar, $|DE|/|AE| = |BC|/|AC|$ $\implies |DE| = \frac{|AE|}{\sqrt{3}}$  $\leq \frac{D}{\sqrt{3}}$. Therefore, any robot $\geq \frac{D}{\sqrt{3}}$ to the south of $r$ must be visible to it. So $r_L$ is visible to $r$. This contradicts our assumption that $r$ is false southmost. 
\end{proof}

 \begin{figure}[thb!]
     \centering
    \includegraphics[width=.5\linewidth]{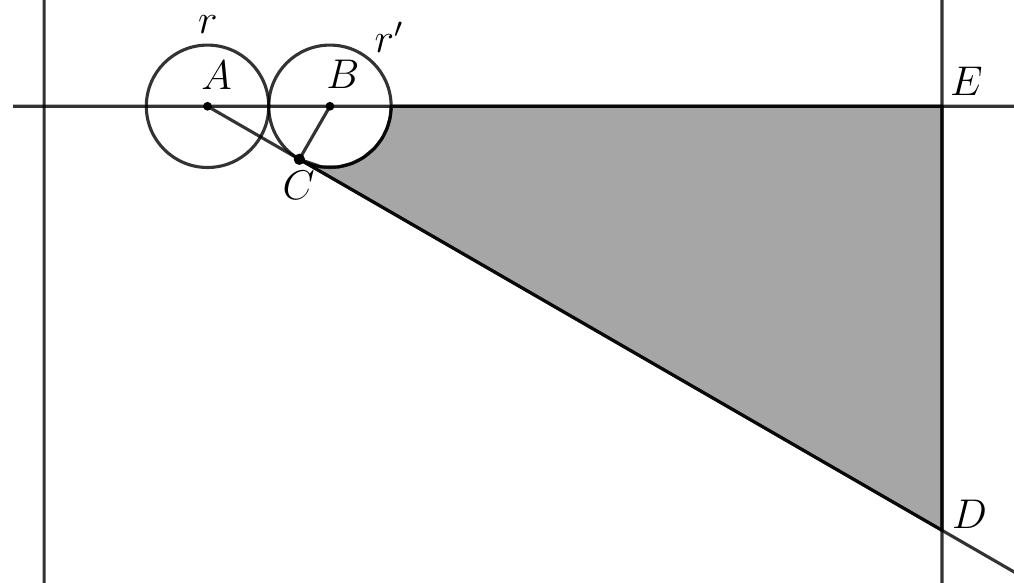}
     \caption[Short Caption]{Illustration supporting the proof of Lemma \ref{lem: no false}. }
\label{fig: no_false}
\end{figure}

  It follows from Lemma \ref{lem: no false} that if there remains any \texttt{off} colored robot after leader election, then the next round it is activated, it will change its color to \texttt{defeated} or \texttt{subordinate}. Any \texttt{defeated} colored robot that can see the leader or any \texttt{subordinate} colored robot will change its color to \texttt{subordinate}. These robots will move southwards to build the base chain. The robots will come to $L_{6}$ sequentially. After a robot comes to $L_{6}$, it will then go to  $L_0$ and place itself on the base chain if there is space. Otherwise, the base chain will be first expanded before it can go to $L_0$. From the discussions in Section \ref{sec: chain formation}, it follows that the robots will carry out these steps correctly. Therefore, a base chain consisting of $n-1$ robots will be built in this way. Then the last robot will go directly to its place in the corresponding mutually visible chain and change its color to \texttt{final}. Then the $n-1$ robots on the base chain will sequentially move vertically to place themselves in their respective positions on the corresponding chain. Again, it follows from the discussions in Section \ref{sec: chain formation} that the robots will be able to execute these steps correctly. Therefore, we have the following result.
  
  \begin{theorem}
  There is a round $T_2 > T_1$ when a base chain consisting of $n-1$ robots is built. Then there is a round $T_3 > T_2$ when the corresponding chain of $n$ robots is built.
 \end{theorem}

  By Lemma \ref{lem: chain vis} it follows that any two robots are visible to each other in the configuration at round $T_3$.  
  
  \begin{theorem}
  The final configuration obtained at round $T_3$ is a mutually visible configuration.
 \end{theorem}

   \begin{theorem}
  The algorithm requires $O(n(1+\mathfrak{e}(n)))$ epochs to complete, where $\mathfrak{e}(n)$ is the number of base chain expansions needed for $n$ robots.
 \end{theorem}

 \begin{proof}
 It follows from the proof of Theorem \ref{th 1} that the leader is elected within $O(n)$ epochs. The non-leader robots will now sequentially place themselves on the base chain. If no expansion is needed, then it takes $O(1)$ epochs for a robot to place itself on the base chain. If an expansion is needed, then it has to wait on $L_6$ for the expansion to complete. An expansion is completed within $O(n)$ epochs. So construction of the base chain takes $O(n(1+\mathfrak{e}(n)))$ epochs. Formation of the mutually visible chain from the base chain is done within $O(n)$ epochs. So in total, the algorithm requires $O(n(1+\mathfrak{e}(n)))$ epochs to complete.
 \end{proof}

\section{Experimental Evaluation}

\subsection{Simulation Framework}

We designed a Python-based simulation framework to carry out the experiments. The key component of the simulation framework is an algorithm that determines the view of a robot in our adopted visibility model. In particular, for a set $\mathfrak{R}$ of opaque fat robots with slim omnidirectional camera, we need to determine if a robot $r \in \mathfrak{R}$ can see another robot $r' \in \mathfrak{R}$. We provide a brief description of the algorithm in the following.

Without loss of generality, assume that $r$ and $r'$ are on the same vertical line. Draw the direct tangents to $\mathcal{C}_{r}$ and $\mathcal{B}_{r'}$. Let $A, A', B, B'$ be the points where the tangents touch the disks as shown in Fig. \ref{see1}.  Notice that only the arc between $A'$ and $B'$ can be visible to $r$. Let $\mathcal{L} = \{ r'' \in \mathfrak{R} \setminus \{r,r'\} \mid$ $ \mathcal{B}_{r''}$ intersects the line segment $ AA'\}$ and $\mathcal{R} = \{ r'' \in \mathfrak{R} \setminus \{r,r'\} \mid$ $ \mathcal{B}_{r''}$ intersects the line segment $ BB'  \}$. We call $\mathcal{L}$ the set of \textit{left obstructions} of $r$ and $\mathcal{R}$ the set of \textit{right obstructions} of $r$. Notice that only the robots of $\mathcal{L} \cup \mathcal{R}$ will cause obstructions. So we can ignore all the other robots. For each $r'' \in \mathcal{L} \cup \mathcal{R}$, draw direct tangents to $\mathcal{C}_{r}$ and $\mathcal{B}_{r''}$. The region between the two tangents as shown in Fig. \ref{see2} is called the \textit{type I bad region defined by} $r''$. If any portion of the arc $A'B'$ lies in this region, then trim off that portion. If any robot $r'' \in \mathcal{R} \setminus \{r,r'\}$ intersects both $AA'$ and $BB'$, i.e., we have $\mathcal{L} \cap \mathcal{R} \neq \emptyset$, then $r'$ is not visible to $r$. This is because in that case, $\mathcal{B}_r'$ lies entirely inside the bad region defined by $r''$. So now assume that $\mathcal{L} \cap \mathcal{R} = \emptyset$. For each pair $(r'', r''') \in \mathcal{L} \times \mathcal{R}$, draw the transverse tangents to $\mathcal{B}_{r''}$ and $\mathcal{B}_{r'''}$ (see Fig. \ref{see3}). The regions between the two tangents as shown in Fig. \ref{see3} are called the \textit{type II bad regions defined by} $r''$ and $r'''$. If any portion of the arc $A'B'$ lies in these regions, then trim off that portion. The algorithm decides that  $r'$ is visible to $r$ if and only if the arc $A'B'$ is not completely trimmed out.




 \begin{figure}[thb!]
     \centering
     \begin{subfigure}[b]{0.31\linewidth}
         \centering
         \includegraphics[width=\linewidth]{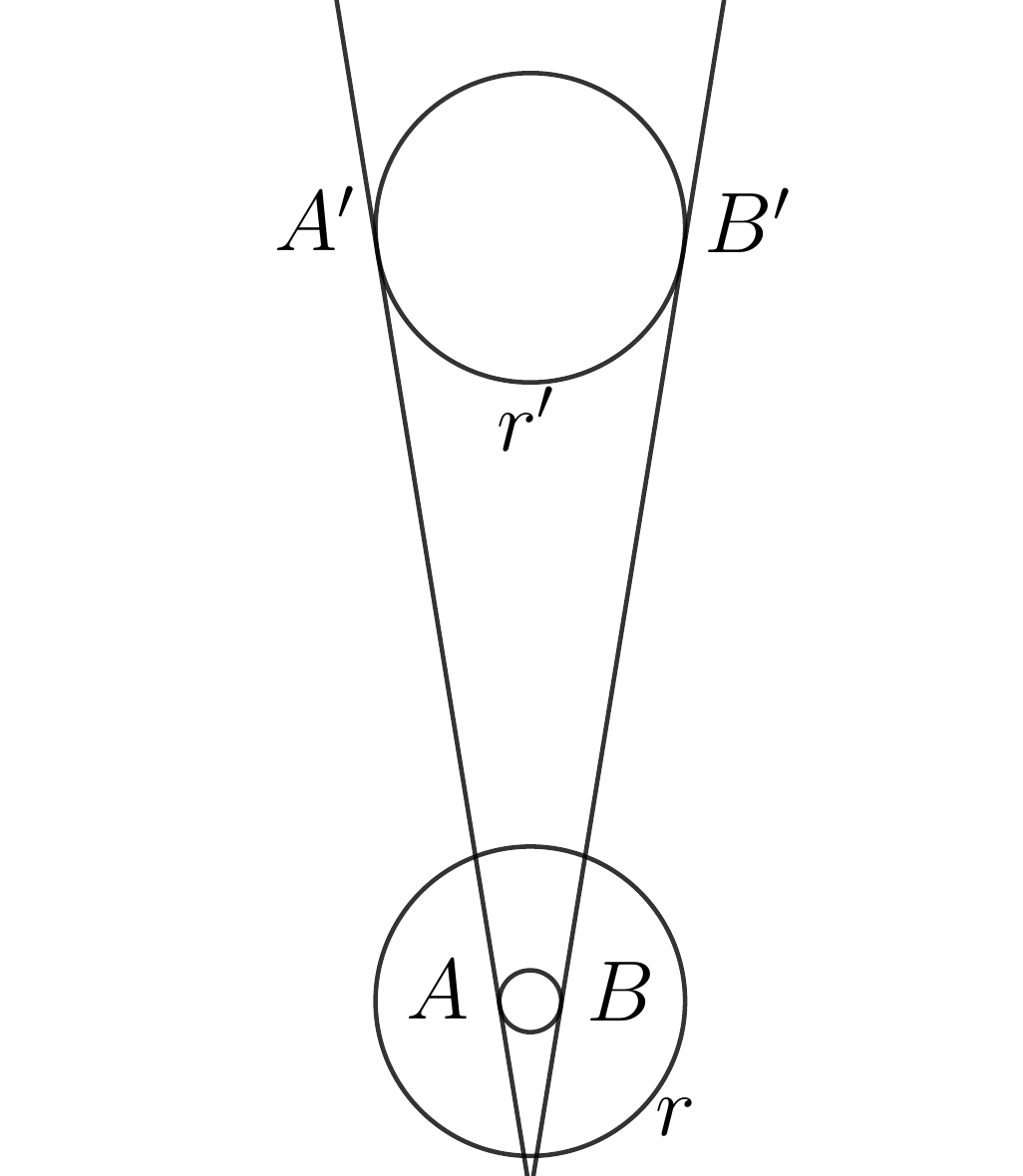}
         \caption{}
          \label{see1}
     \end{subfigure}
     \hfill
     \begin{subfigure}[b]{0.31\linewidth}
         \centering
         \includegraphics[width=\linewidth]{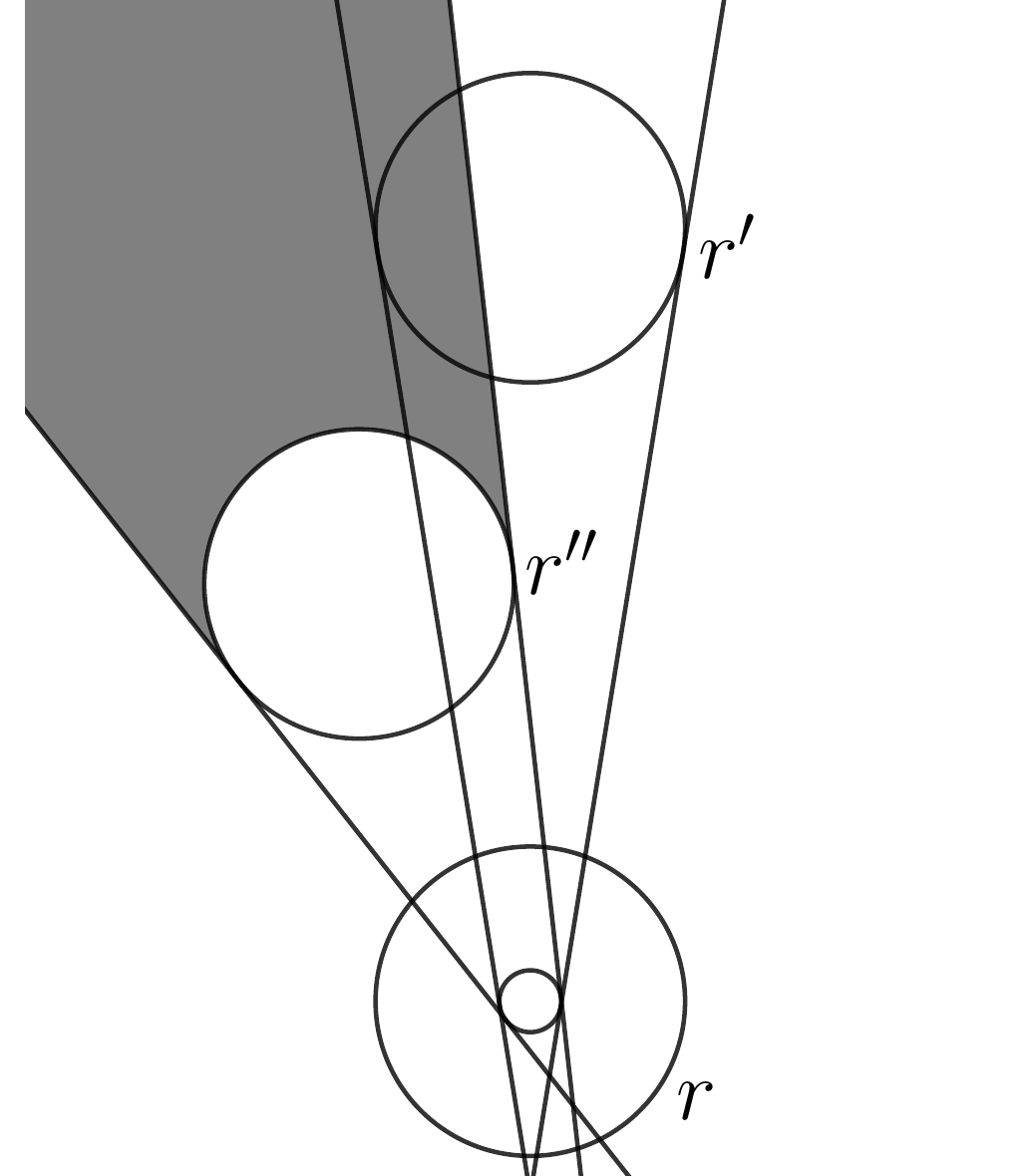}
         \caption{}
          \label{see2}         
     \end{subfigure}
     \hfill
     \begin{subfigure}[b]{0.31\linewidth}
         \centering
         \includegraphics[width=\linewidth]{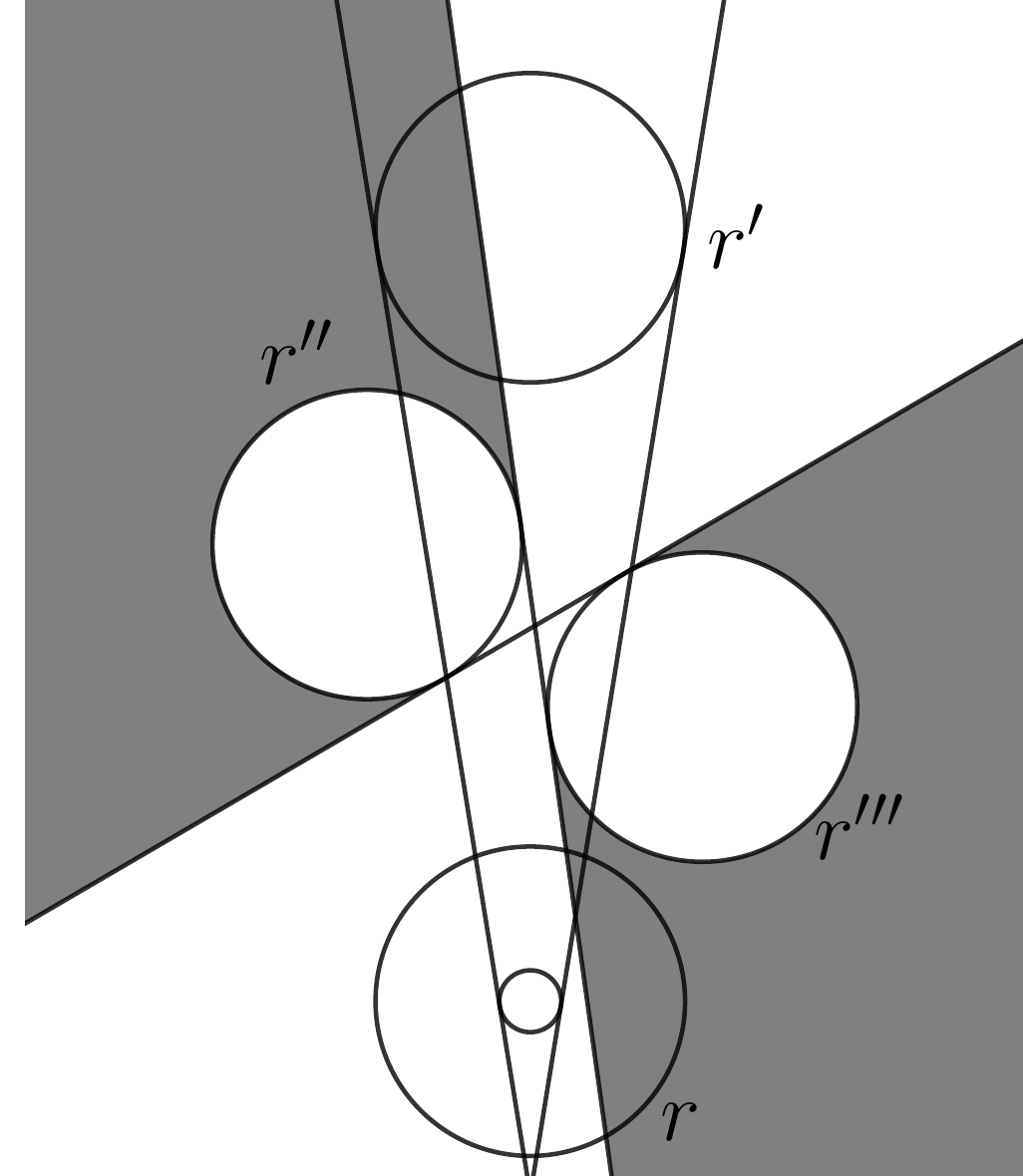}
         \caption{}
          \label{see3}         
     \end{subfigure}
\caption[Short Caption]{(a) Only the arc between $A'$ and $B'$ can be visible to $r$. (b) Type I bad region defined by $r''$. (c) Type II bad region defined by $r''$ and $r'''$.}
\end{figure}

\begin{theorem}
The  algorithm described above is correct.
\end{theorem}

\begin{proof}

If  $\mathcal{L} \cap \mathcal{R} \neq \emptyset$, the algorithm decides that $r'$ is not visible to $r$. This is obviously correct. So now assume that $\mathcal{L} \cap \mathcal{R} = \emptyset$. Take any point $p$ on the arc $A'B'$. To prove the theorem, we have to show that $p$ is visible to $r$ if and only if it survives the trimmings. Now draw tangents from $p$ to the robots of $\mathcal{L}$ and let $\ell_{\mathcal{L}}$ be the rightmost one (see Fig. \ref{see4}). Similarly, draw tangents from $p$ to the robots of $\mathcal{R}$ and let $\ell_{\mathcal{R}}$ be the leftmost one (see Fig. \ref{see4}). Observe that the region that is to the right of $\ell_{\mathcal{L}}$ and the left of $\ell_{\mathcal{R}}$ (the shaded region in Fig. \ref{see4}) is the union of all unobstructed rays emanating from $p$. Let us denote this region by $\mathcal{U}_p$. So $p$ will be visible to $r$ if and only if $\mathcal{C}_r \cap \mathcal{U}_p \neq \emptyset$ (see Fig. \ref{fig: 3cases}). Recall that we have to show that $p$ is visible to $r$ if and only if it survives the trimmings. So we have to show that $p$ survives the trimmings if and only if $\mathcal{C}_r \cap \mathcal{U}_p \neq \emptyset$. Before we prove this, let us make some observations. Rigorous proofs of the observation are skipped.

\begin{itemize}
    \item \textbf{Observation 1}. Let $r'' \in \mathcal{L}$ and $\ell$ be the right tangent from $p$ to $\mathcal{B}_{r''}$. Then observe that $p$ is in the type I bad region defined by $r''$  if and only if $\mathcal{C}_r$ is completely to the left of $\ell$ (see Fig. \ref{see5} and \ref{see6}). Analogous observation holds when $r'' \in \mathcal{R}$.  
    

    \item \textbf{Observation 2}. Let $r'' \in \mathcal{L}$ and Let $r''' \in \mathcal{R}$. Also let $\ell''$ be the right tangent from $p$ to $\mathcal{B}_{r''}$ and $\ell'''$ be the left tangent from $p$ to $\mathcal{B}_{r'''}$. Then observe that $p$ is in the type II bad region defined by $r''$ and $r'''$  if and only if  $\ell''$ is to the right of $\ell'''$ (see Fig. \ref{see8} and \ref{see7} ). 
    
\end{itemize}

\begin{figure}[thb!]
     \centering
     \begin{subfigure}[b]{0.31\linewidth}
         \centering
         \includegraphics[width=\linewidth]{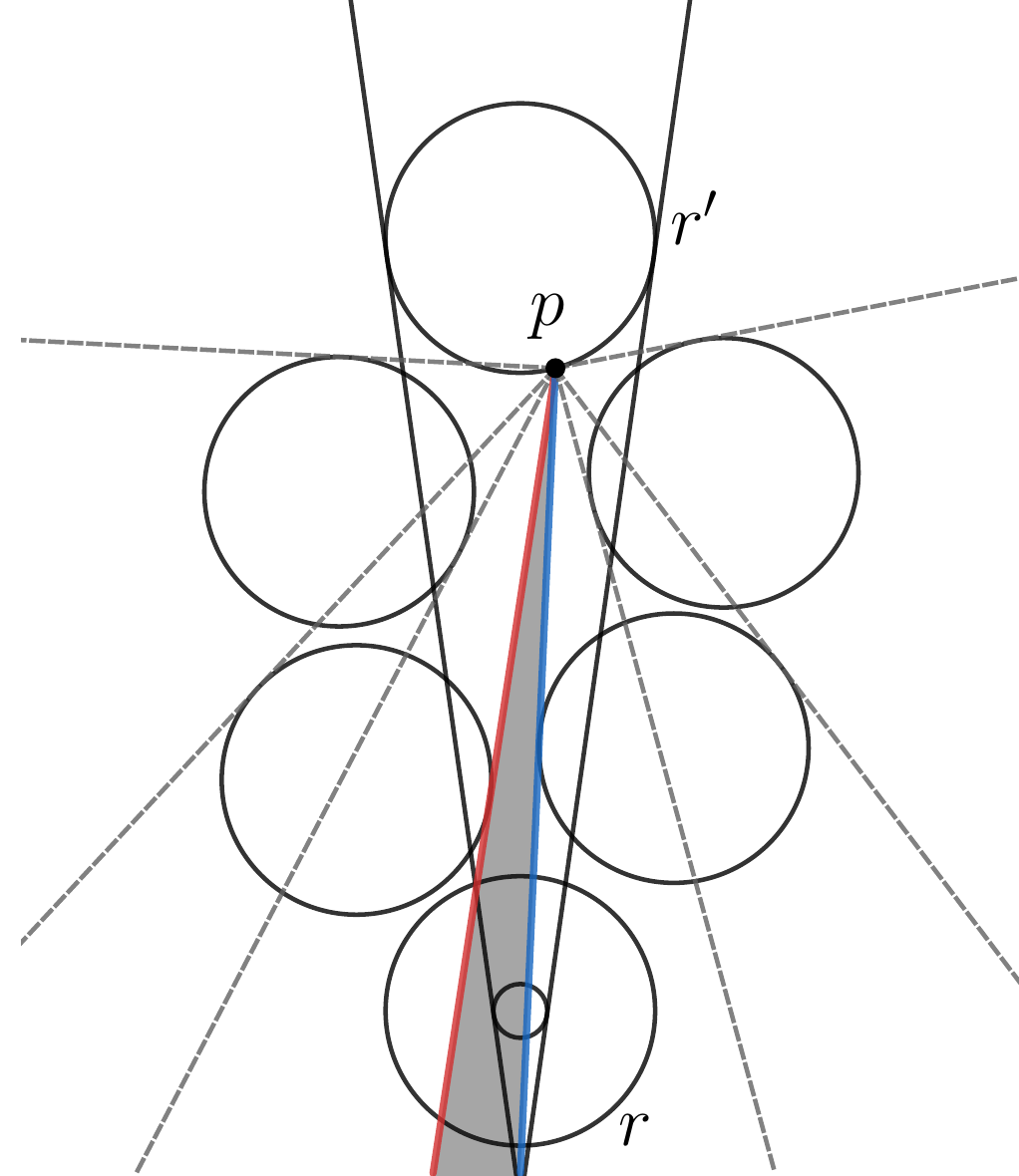}
         \caption{}
          \label{see4}
     \end{subfigure}
     \hfill
     \begin{subfigure}[b]{0.31\linewidth}
         \centering
         \includegraphics[width=\linewidth]{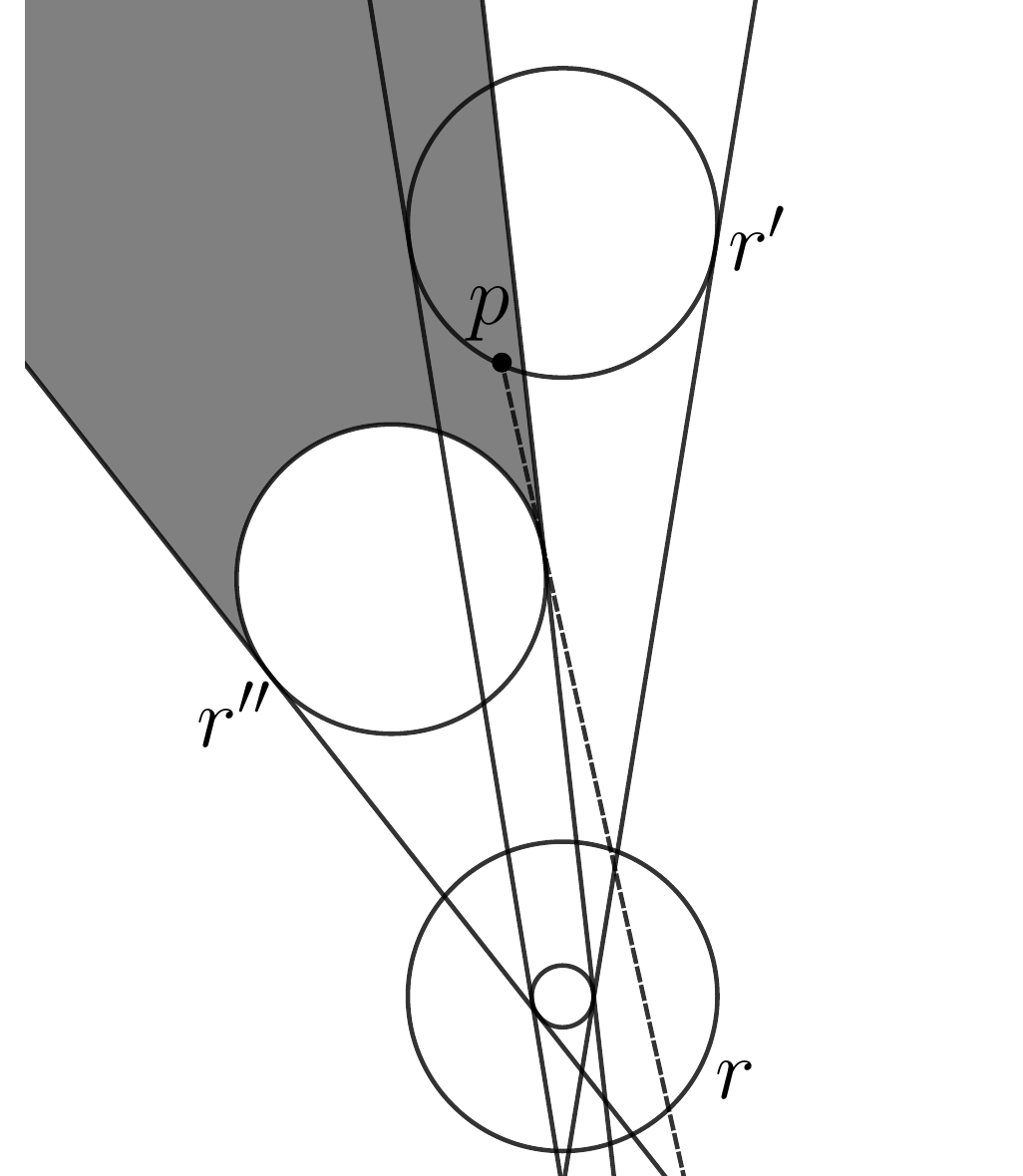}
         \caption{}
          \label{see5}         
     \end{subfigure}
     \hfill
     \begin{subfigure}[b]{0.31\linewidth}
         \centering
         \includegraphics[width=\linewidth]{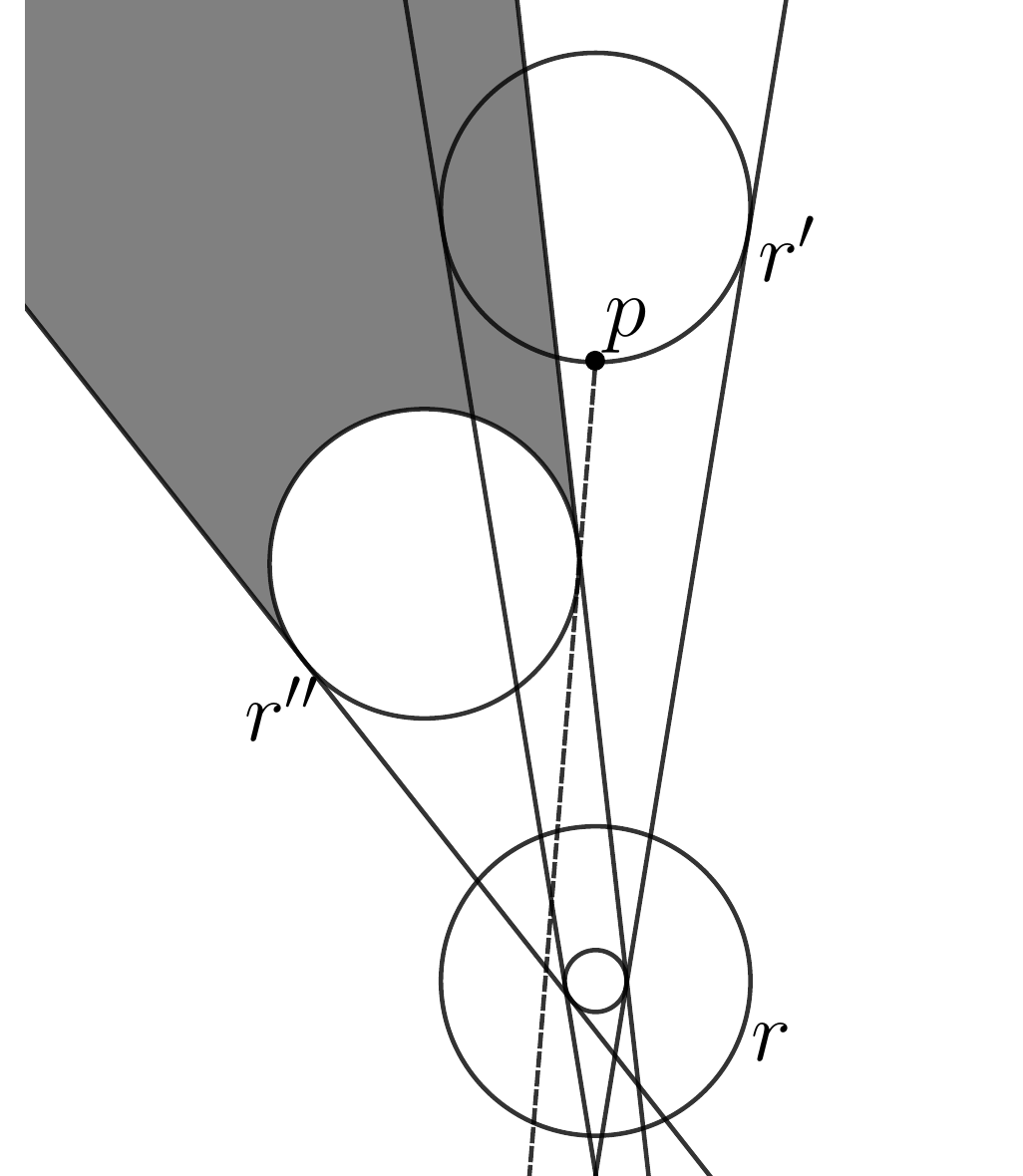}
         \caption{}
          \label{see6}         
     \end{subfigure}
     \\
     \begin{subfigure}[b]{0.31\linewidth}
         \centering
         \includegraphics[width=\linewidth]{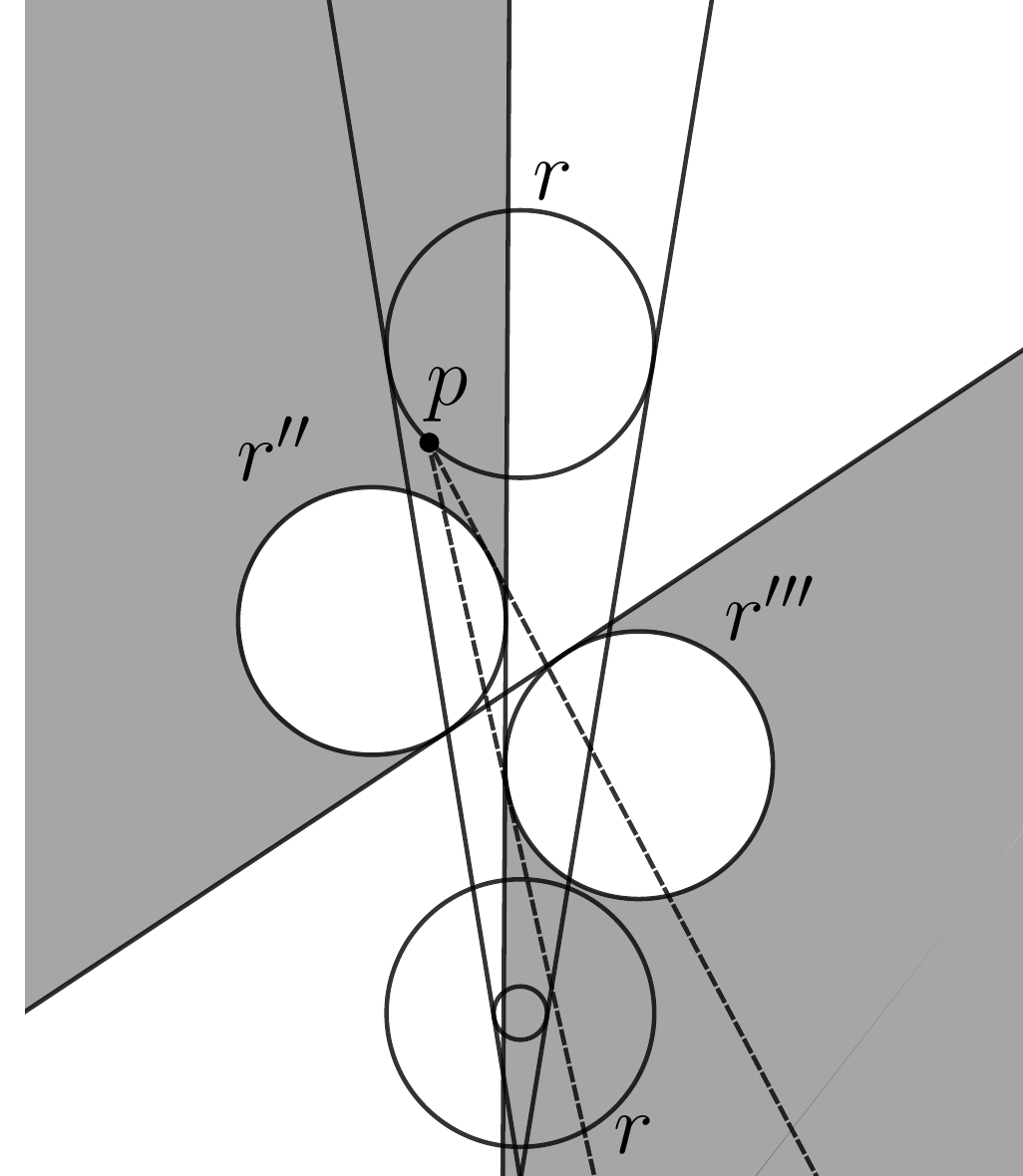}
         \caption{}
          \label{see8}         
     \end{subfigure}
     \hspace{10pt}
     \begin{subfigure}[b]{0.31\linewidth}
         \centering
         \includegraphics[width=\linewidth]{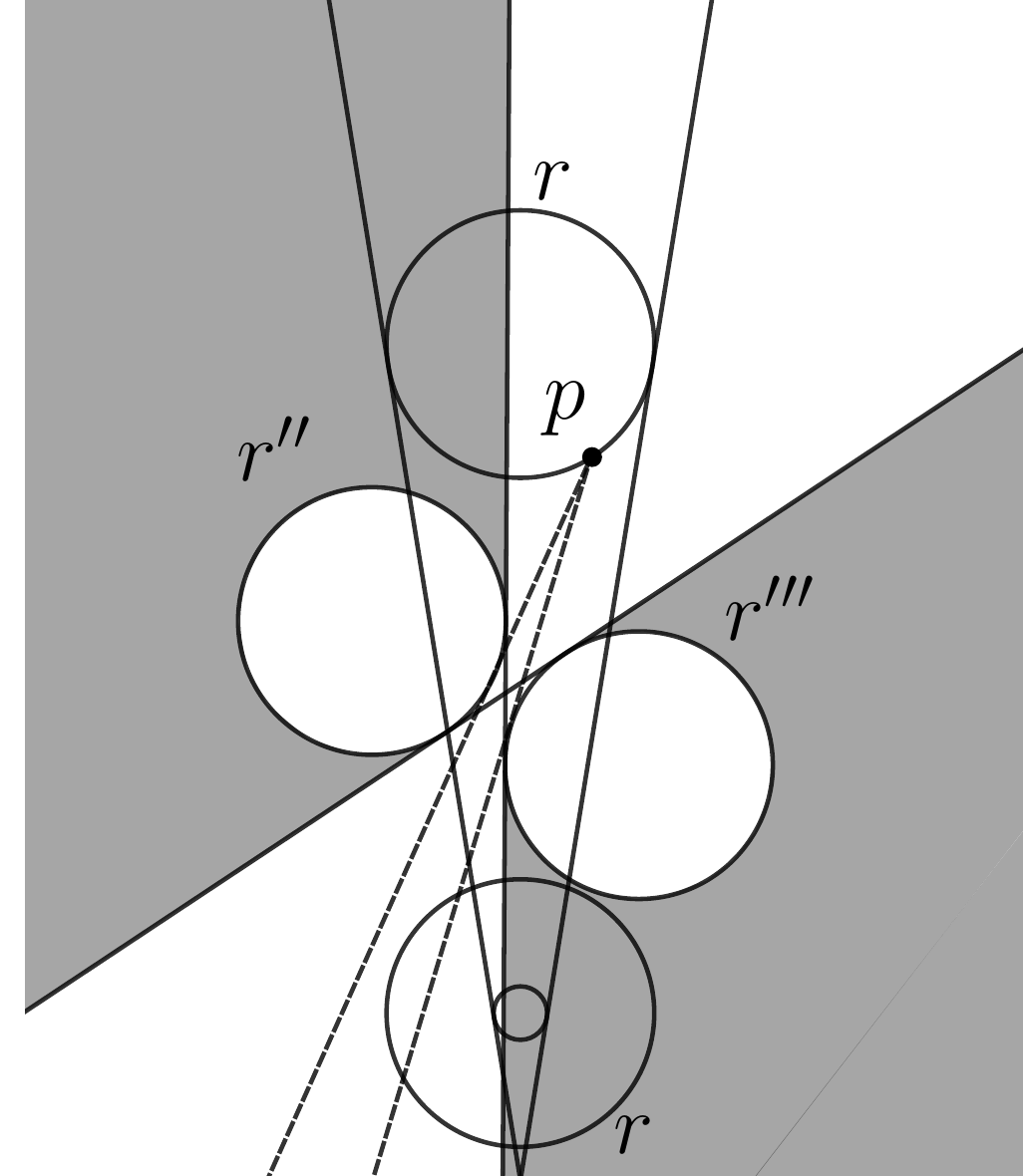}
         \caption{}
          \label{see7}         
     \end{subfigure}
\caption[Short Caption]{(a) The red ray is $\ell_{\mathcal{L}}$ and the blue ray is $\ell_{\mathcal{R}}$. The shaded region is $\mathcal{U}_p$. (b) When $p$ is in the bad region,  $\mathcal{C}_r$ is completely to the left of the right tangent from $p$ to $\mathcal{B}_{r''}$. (c) When $p$ is outside the bad region, some portion of $\mathcal{C}_r$ is to the right of the right tangent from $p$ to $\mathcal{B}_{r''}$. (d) When $p$ is in the bad region, the right tangent from $p$ to $\mathcal{B}_{r''}$ is to the right of the left tangent from $p$ to $\mathcal{B}_{r'''}$. (e)  When $p$ is outside the bad region, the right tangent from $p$ to $\mathcal{B}_{r''}$ is to the left of the left tangent from $p$ to $\mathcal{B}_{r'''}$.}
\end{figure}

Recall that we have to prove that $p$ survives the trimmings if and only if $\mathcal{C}_r \cap \mathcal{U}_p \neq \emptyset$. Let us first prove that if $p$ does not survive the trimmings, then $\mathcal{C}_r \cap \mathcal{U}_p = \emptyset$. So take a point $p$ such that it is inside a bad region of type I or type II. Suppose that $p$ is inside a bad region of type I defined by, without loss of generality, a left obstacle $r'' \in \mathcal{L}$. If  $\ell$ be the right tangent from $p$ to $\mathcal{B}_{r''}$, then by Observation 1, $\mathcal{C}_r$ is completely to the left of $\ell$. This implies that $\mathcal{C}_r$ is completely to the left of $\ell_{\mathcal{L}}$ and hence  $\mathcal{C}_r \cap \mathcal{U}_p = \emptyset$. Now suppose that $p$ is inside a bad region of type II defined by $r'' \in \mathcal{L}$ and $r''' \in \mathcal{R}$. Let $\ell''$ be the right tangent from $p$ to $\mathcal{B}_{r''}$ and $\ell'''$ be the left tangent from $p$ to $\mathcal{B}_{r'''}$. Then, by Observation 2, $\ell''$ is to the right of $\ell'''$, which implies that $\ell_{\mathcal{L}}$ is to the right of $\ell_{\mathcal{R}}$. This implies that $\mathcal{U}_p = \emptyset$, and hence $\mathcal{C}_r \cap \mathcal{U}_p = \emptyset$.

Finally, we show that if $p$ survives the trimmings, then $\mathcal{C}_r \cap \mathcal{U}_p \neq \emptyset$. So take a $p$ such that it is outside all bad regions. For any  $r'' \in \mathcal{L}$ and $r''' \in \mathcal{R}$,  $\ell''$ is to the left of $\ell'''$, where $\ell''$ is the right tangent from $p$ to $\mathcal{B}_{r''}$ and $\ell'''$ is the left tangent from $p$ to $\mathcal{B}_{r'''}$ (by Observation 2). In particular, $\ell_{\mathcal{L}}$ is to the left of $\ell_{\mathcal{R}}$. This means that $\mathcal{U}_p \neq \emptyset$. For any $r' \in \mathcal{L}$, some portion of $\mathcal{C}_r$ is to the right of $\ell$, where $\ell$ is the right tangent from $p$ to $\mathcal{B}_{r''}$ (by Observation 1). In particular, some portion of $\mathcal{C}_r$ is to the right of $\ell_{\mathcal{L}}$. Similarly, some portion of $\mathcal{C}_r$ is to the left of $\ell_{\mathcal{R}}$. This implies that $\mathcal{C}_r \cap \mathcal{U}_p \neq \emptyset$. 

\end{proof}

\begin{figure}[thb!]
     \centering
     \begin{subfigure}[b]{0.31\linewidth}
         \centering
         \includegraphics[width=\linewidth]{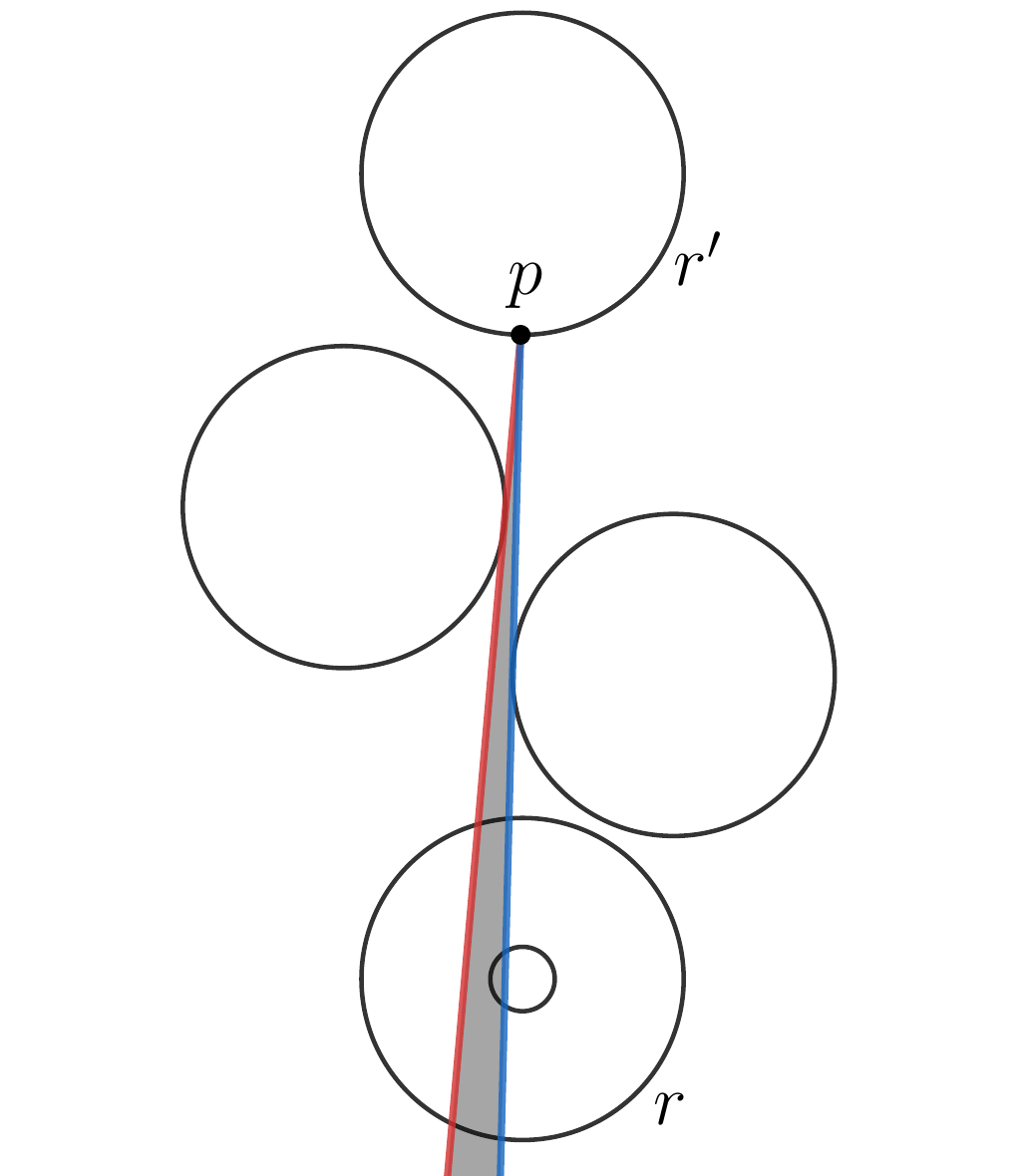}
         \caption{}
          \label{see10}
     \end{subfigure}
     \hfill
     \begin{subfigure}[b]{0.31\linewidth}
         \centering
         \includegraphics[width=\linewidth]{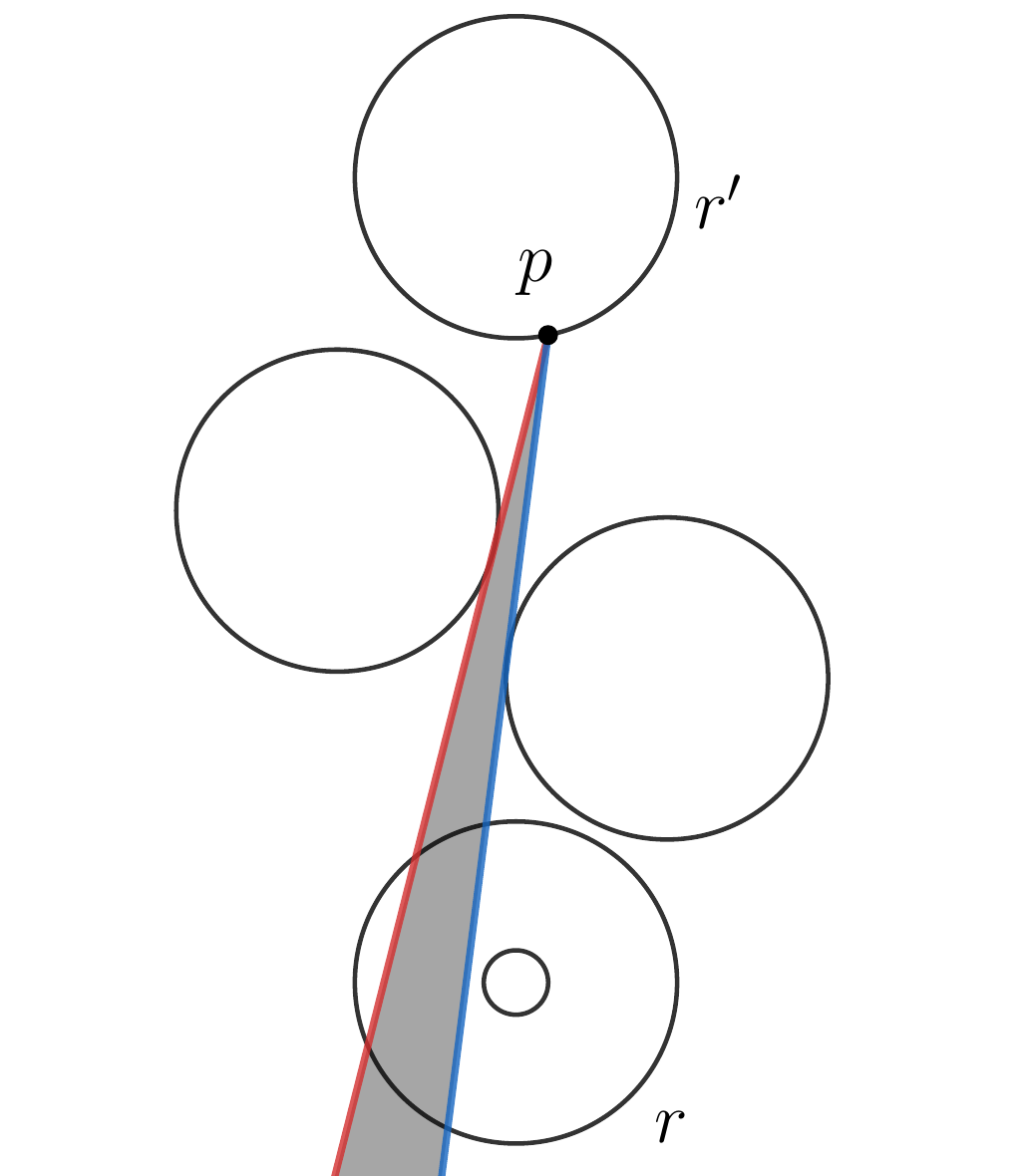}
         \caption{}
          \label{see9}         
     \end{subfigure}
     \hfill
     \begin{subfigure}[b]{0.31\linewidth}
         \centering
         \includegraphics[width=\linewidth]{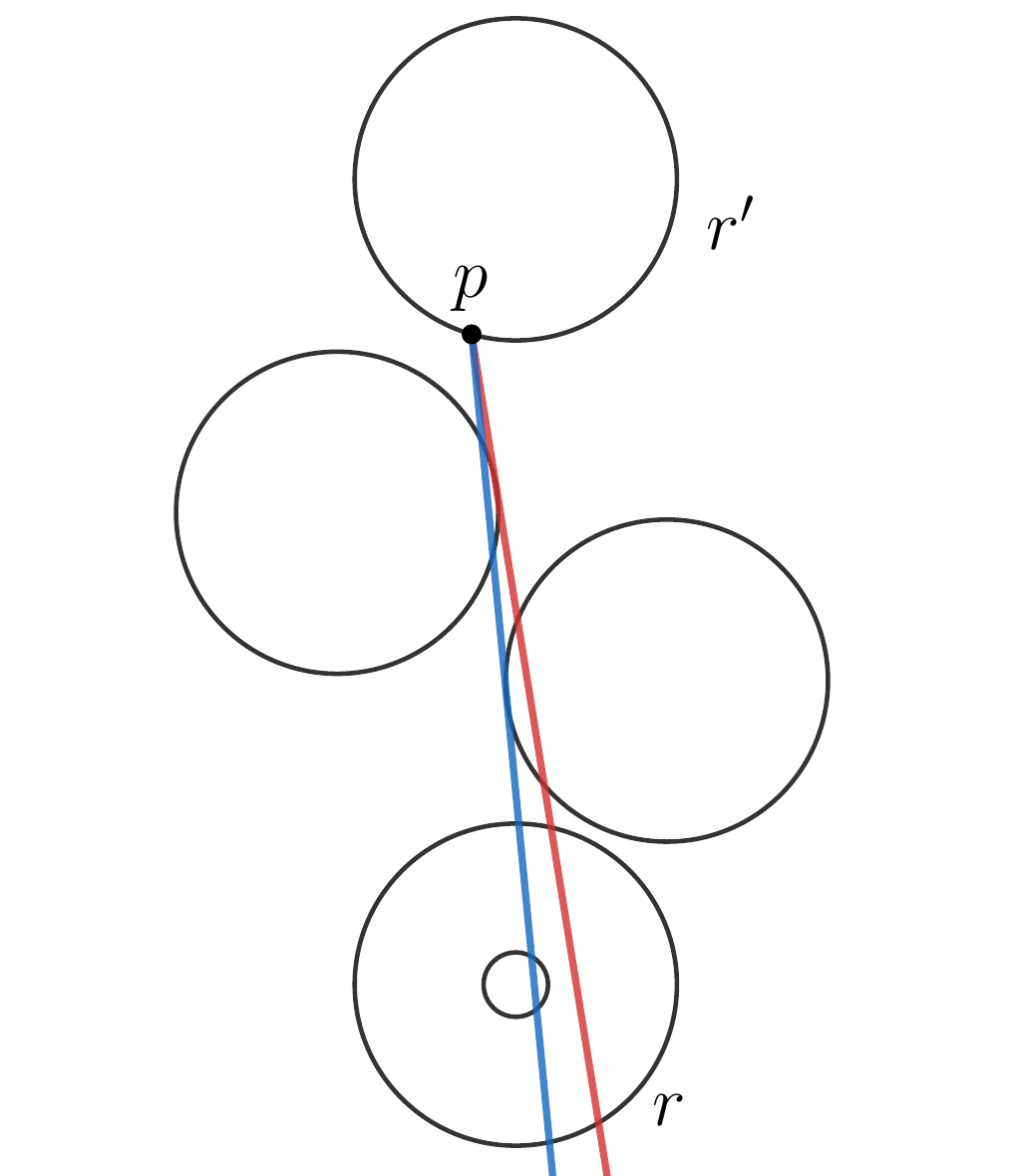}
         \caption{}
          \label{see11}         
     \end{subfigure}

\caption[Short Caption]{The red ray is $\ell_{\mathcal{L}}$ and the blue ray is $\ell_{\mathcal{R}}$. (a) $p$ is visible to $r$ because $\mathcal{C}_r \cap \mathcal{U}_p \neq \emptyset$ (b)-(c) $p$ is not visible to $r$ because  $\mathcal{C}_r \cap \mathcal{U}_p = \emptyset$. In case of (c), $\mathcal{U}_p$ is empty, while in case of (b), $\mathcal{U}_p$ is non-empty, but it does not overlap $\mathcal{C}_r$.} 
 \label{fig: 3cases}
\end{figure}

\subsection{Experimental Results}


We first focus on the leader election stage of the algorithm. We analyzed the round complexity of the leader election stage in Section \ref{sec 6} where we have proved an upper bound of $O(n)$ epochs. The main source of difficulty in the leader election problem was the existence of false southmost robots. Movements of the false southmost robots can prevent the actual southmost robot from creating the required vertical separation from the rest of the team. We showed that $n-1$ robots would change their colors to \texttt{defeated} within $O(n)$ rounds. After this, there are no movements from these robots and the southmost robot can establish the required vertical separation in $O(1)$ epochs. This gave the $O(n)$ upper bound for the leader election stage. However, our simulations reveal that the problematic movements of the false southmost robots occur scarcely. As a result, the $O(n)$ epoch bound may be quite loose. Details of the experimental results are expounded in the following.

We randomly generate initial configurations of robots inside a rectangular region. In particular, the position of each robot is generated uniformly at random inside the rectangular region. If the generated point is within 2 units of a previously generated point, then we discard it and try again. To simulate an $\mathcal{SSYNC}$ scheduler, the activation of each robot at any round is set as a Bernoulli process $Bern(p)$. When $p = 1$, it coincides with an $\mathcal{FSYNC}$ scheduler. Under this framework, we first simulate our leader election algorithm in $\mathcal{FSYNC}$ and note the following:
\begin{enumerate}
    \item $\mathfrak{m}=$ the number of moves by false southmost robots,
    \item $\mathfrak{r}=$ the number of epochs needed for $n-1$ robots to change their colors to \texttt{defeated}.
\end{enumerate}

\begin{figure}[h!]
     \centering
     \begin{subfigure}[b]{0.38\linewidth}
         \centering
         \includegraphics[width=\linewidth]{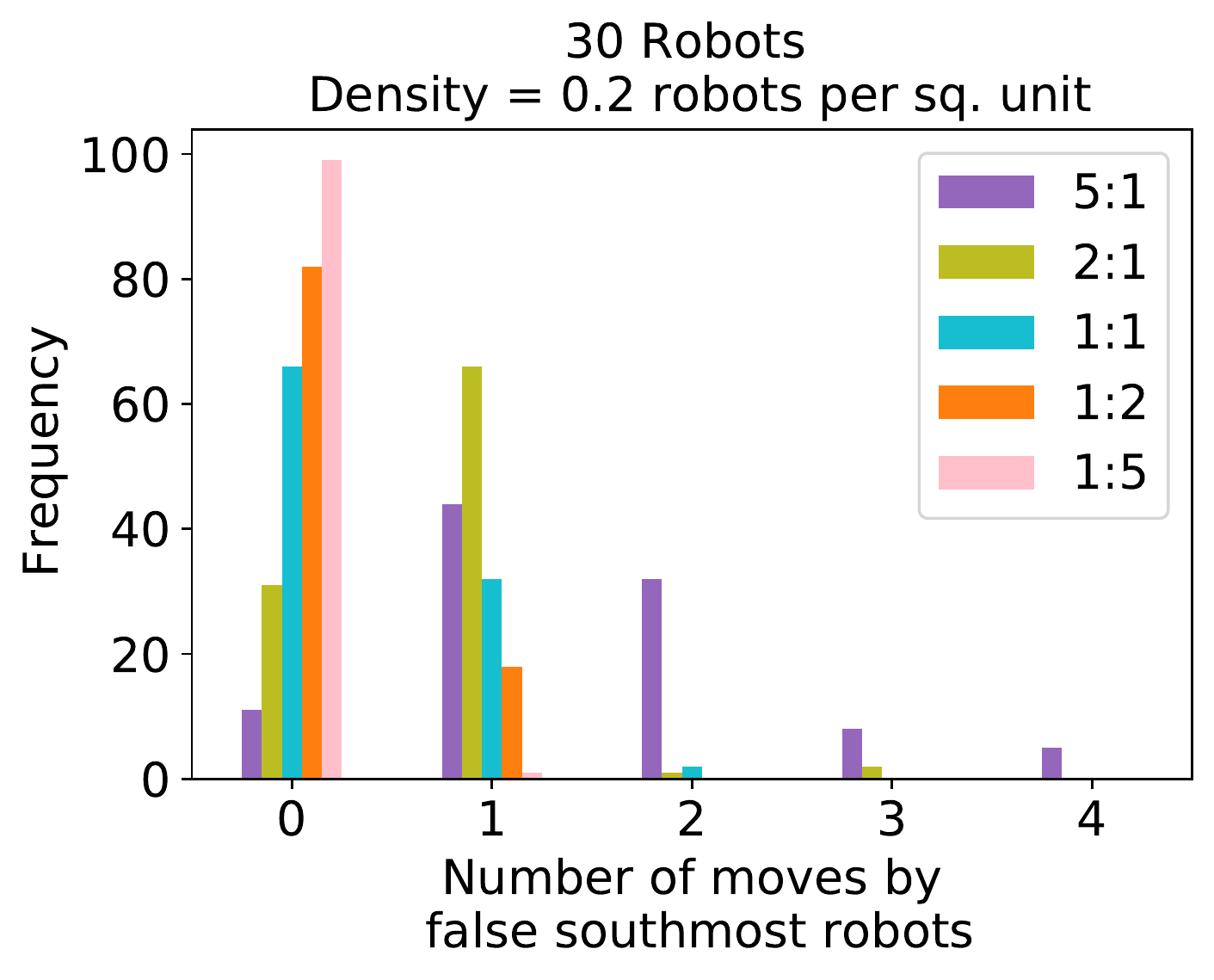}
         \caption{}
         \label{}
     \end{subfigure}
     \hfill
     \begin{subfigure}[b]{0.38\linewidth}
         \centering
         \includegraphics[width=\linewidth]{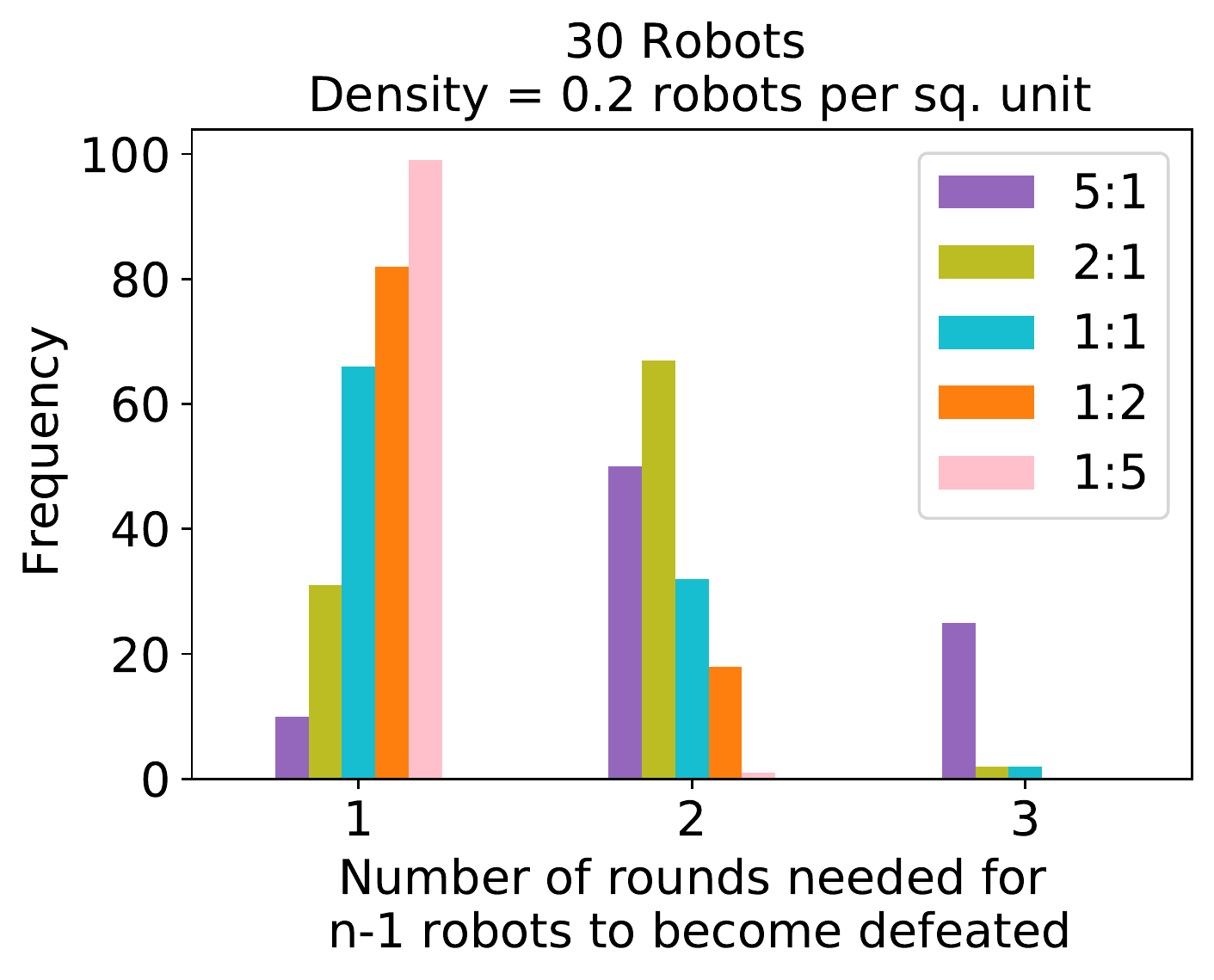}
         \caption{}
         \label{}
     \end{subfigure}
     \\
     \begin{subfigure}[b]{0.38\linewidth}
         \centering
         \includegraphics[width=\linewidth]{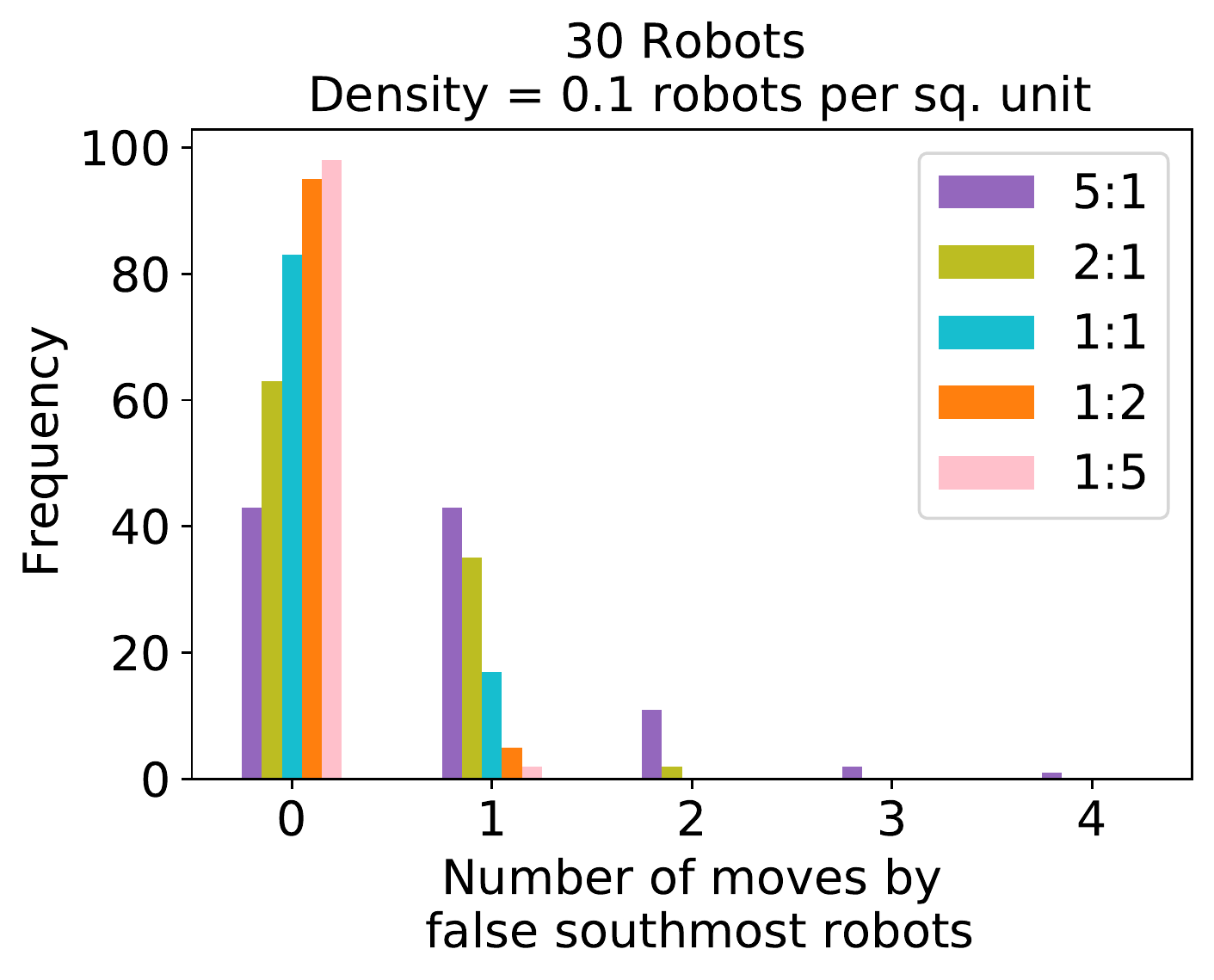}
         \caption{}
         \label{}
     \end{subfigure}
     \hfill
     \begin{subfigure}[b]{0.38\linewidth}
         \centering
         \includegraphics[width=\linewidth]{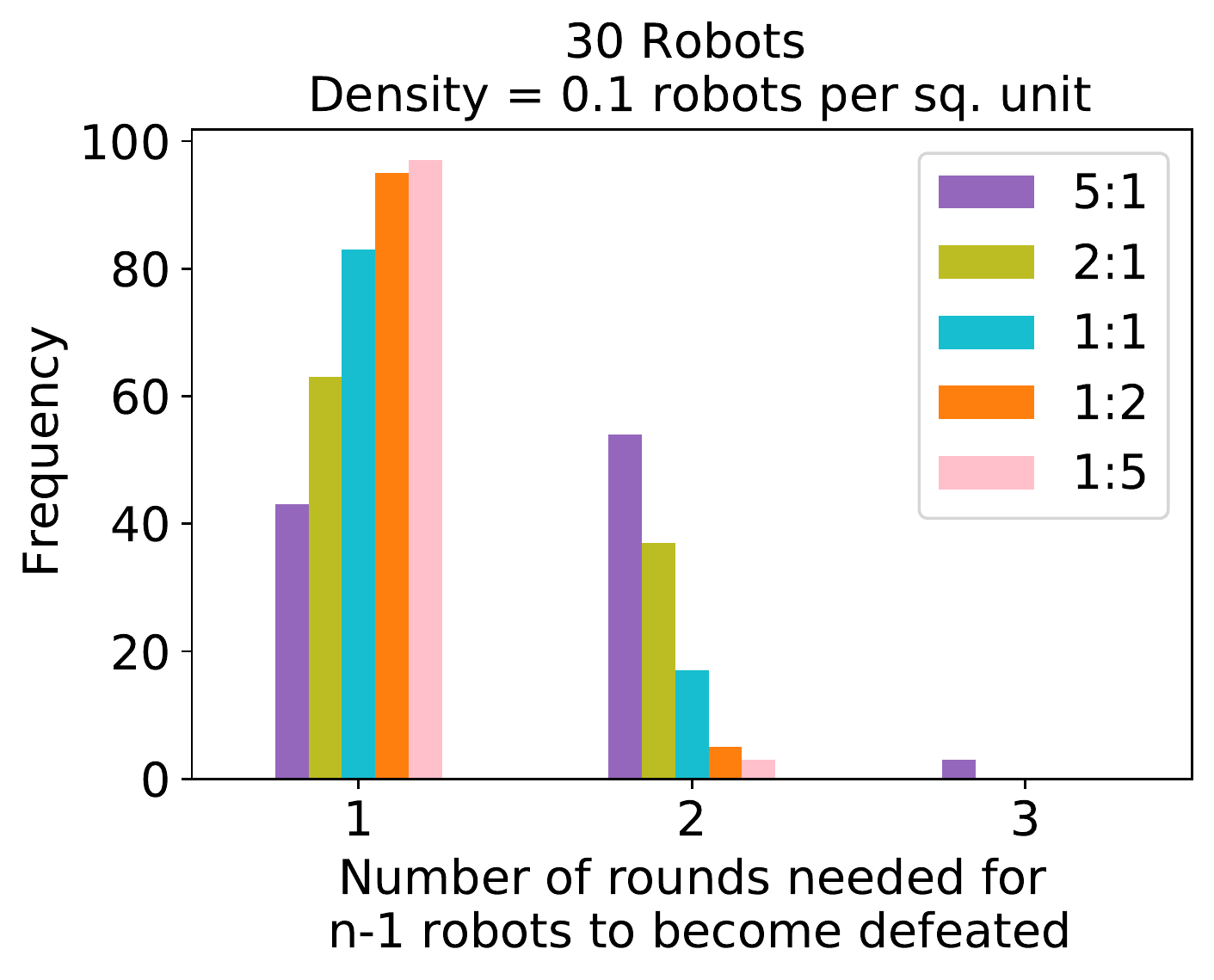}
         \caption{}
         \label{}
     \end{subfigure}
     \\
     \begin{subfigure}[b]{0.38\linewidth}
         \centering
         \includegraphics[width=\linewidth]{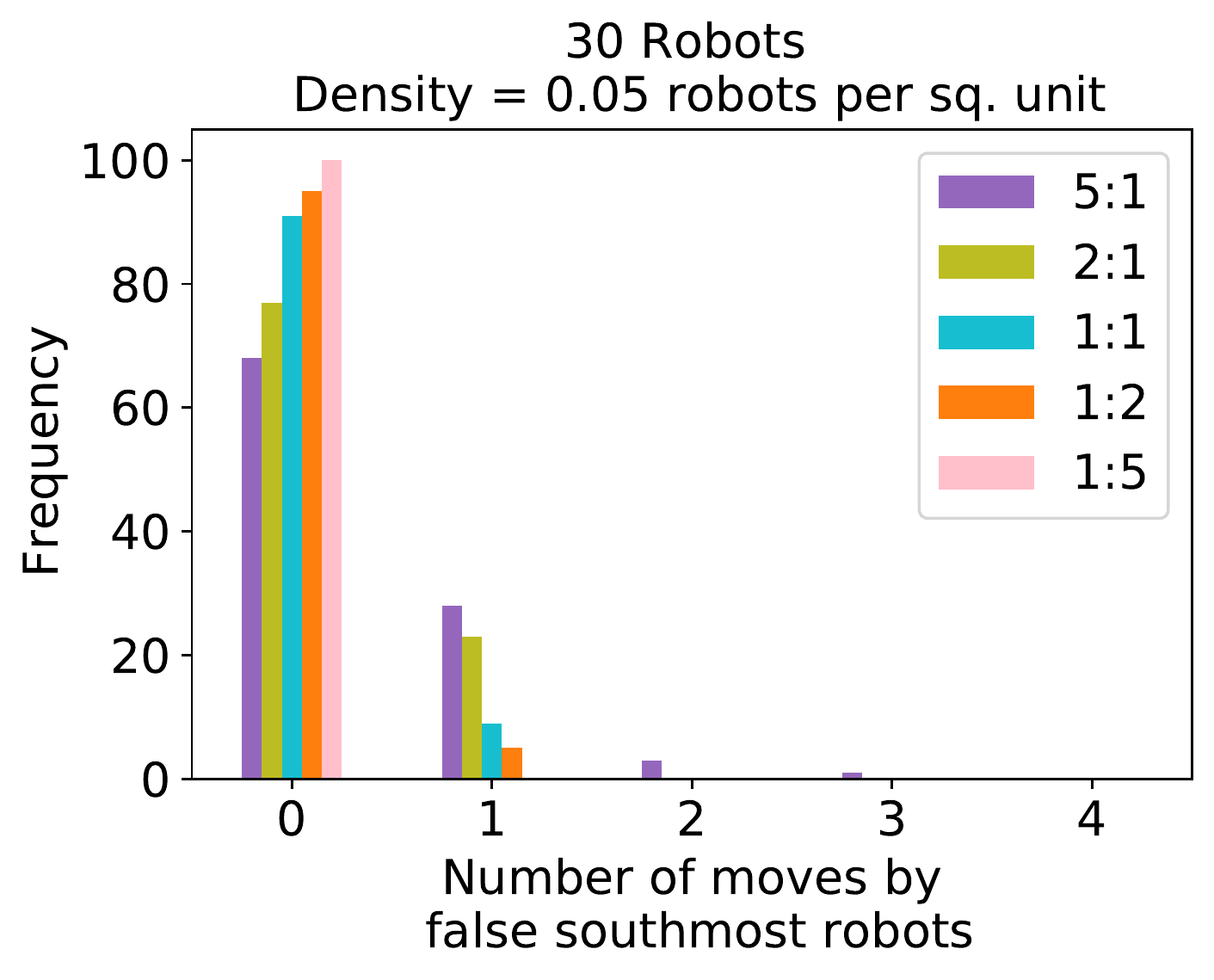}
         \caption{}
         \label{}
     \end{subfigure}
     \hfill
     \begin{subfigure}[b]{0.38\linewidth}
         \centering
         \includegraphics[width=\linewidth]{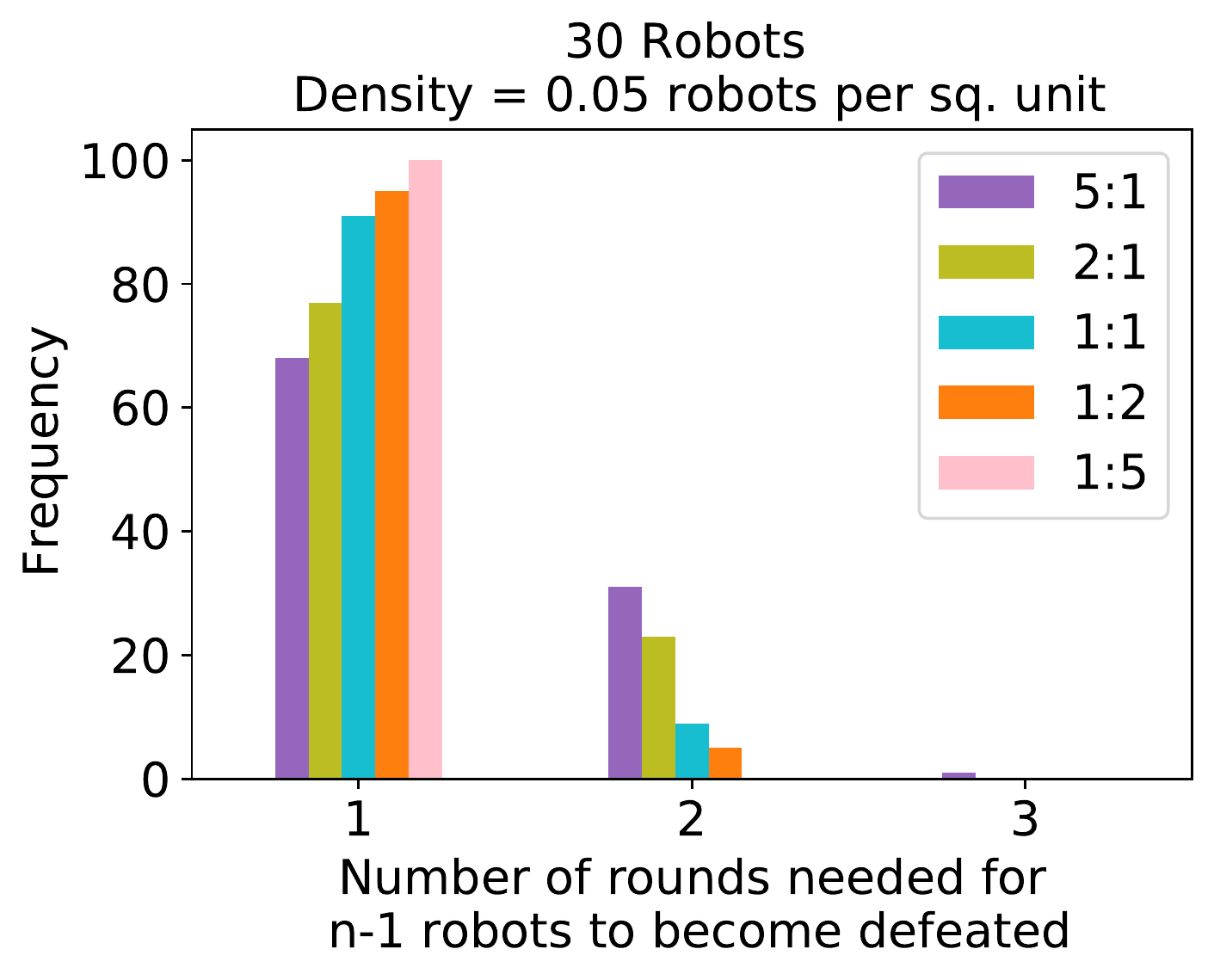}
         \caption{}
         \label{}
     \end{subfigure}
\caption[Short Caption]{Observations for the leader election algorithm run for 30 robots for different robot density and different initial deployment region. The robot density is 0.2 robots per square unit in the first row, 0.1 robots per square unit in the second row and 0.05 robots per square unit in the third row. In each case, the algorithm is run for deployment regions having horizontal side length: vertical side length = 5:1, 2:1, 1:1, 1:2 and 1:5.}
\label{fig: bar1}
\end{figure} 
We first run the algorithm for different rectangular deployment regions for $n = 30$ robots. The algorithm is run 100 times in each case. The results are presented in Fig. \ref{fig: bar1}. The results reveal the effect of the shapes of the rectangles (length of horizontal and vertical sides) and robot density (number of robots per unit$^2$ area). We can see that the problematic moves by false southmost robots are quite rare, in particular $\mathfrak{m} \leq 4$ in each run. As a result, the number of rounds needed for $n-1$ robots to change their colors to \texttt{defeated} is significantly less than what our theoretical analysis suggests. In particular, we see that $\mathfrak{r} \leq 4$ in each run. We also observe that $\mathfrak{m}$ and $\mathfrak{r}$ are higher in the case of deployment regions with larger horizontal sides. This is because false southmost robots are usually found in the southmost layer of the configuration (see Fig. \ref{south1}). As a result, we are more likely to have numerous false southmost robots in horizontally stretched deployment regions. However, it is possible to have false southmost robots at any part of a configuration as shown in Fig. \ref{south2} where even a northmost robot is false southmost. Although one can contrive many such configurations, they are unlikely to show up in random deployments. We also see that $\mathfrak{m}$ and $\mathfrak{r}$ tend to drop as robot density decreases. This is quite natural, as there are more visibility obstruction in dense configurations. We can infer from these observations that the horizontal extent and the density of the initial configuration are the main factors determining $\mathfrak{m}$ and $\mathfrak{r}$. To further test the above statement, we repeat the experiments for different sizes of the swarm, but choosing the deployment regions in such a way that the density and horizontal width are the same in each case. In particular, the density and horizontal width in each case are respectively 0.2 robots/unit$^2$ and 25 units. The results are presented in Fig. \ref{fig: bar2}. As expected, we did not see any discernible effects due to the varying size of the swarm. The experiments we have discussed till now were all carried out in the $\mathcal{FSYNC}$ setting. Notice that in $\mathcal{FSYNC}$, the robot that becomes the leader is the actual southmost robot in the initial configuration. However, in $\mathcal{SSYNC}$, a false southmost robot may also become the leader. The main observation from the experiments conducted in $\mathcal{FSYNC}$ was that the problematic moves by robots other than the eventual leader are only a few in number. In $\mathcal{SSYNC}$, they are even rarer. We ran the algorithm for different activation probabilities, observing that these moves become rarer as the activation probability decreases (see Fig. \ref{fig: bar3}). Recall that the algorithm asks the leader to establish  $\geq$ max $\{10, \frac{D}{\sqrt{3}}\}$ units of vertical separation from the rest of the swarm. The $\geq$ $\frac{D}{\sqrt{3}}$ unit vertical separation ensures that no moves by \texttt{off} colored false southmost robots occur after the leader is elected. This property helps to simplify the correctness proof, but requires the assumption that the robots have the knowledge of $D$. However, the experiments show that such moves are quite rare. Also, even if such a move occurs after the leader is elected, it is unlikely to hamper the algorithm. Therefore, it is fine in practice to ask for only $\geq 10$ units of vertical separation and drop the assumption about the knowledge of $D$.



 \begin{figure}[thb!]
     \centering
     \begin{subfigure}[b]{0.49\linewidth}
         \centering
         \includegraphics[width=\linewidth]{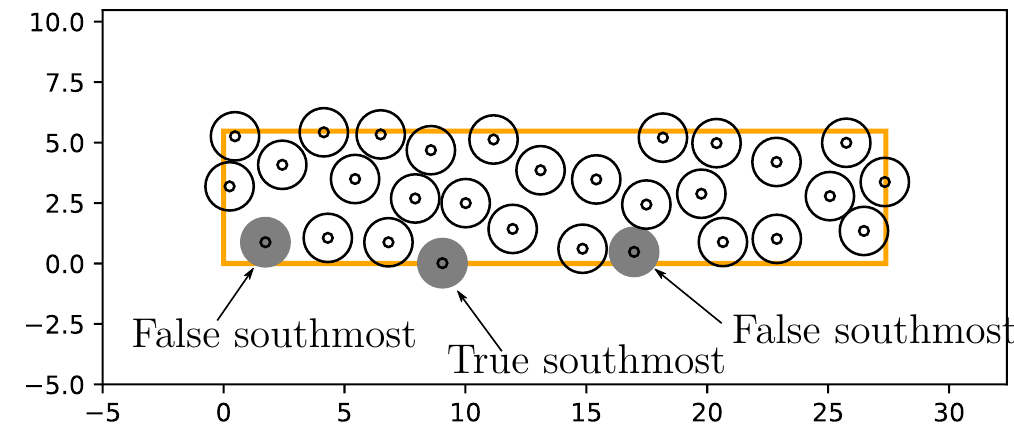}
         \caption{}
         \label{south1}
     \end{subfigure}
     \hfill
     \begin{subfigure}[b]{0.49\linewidth}
         \centering
         \includegraphics[width=\linewidth]{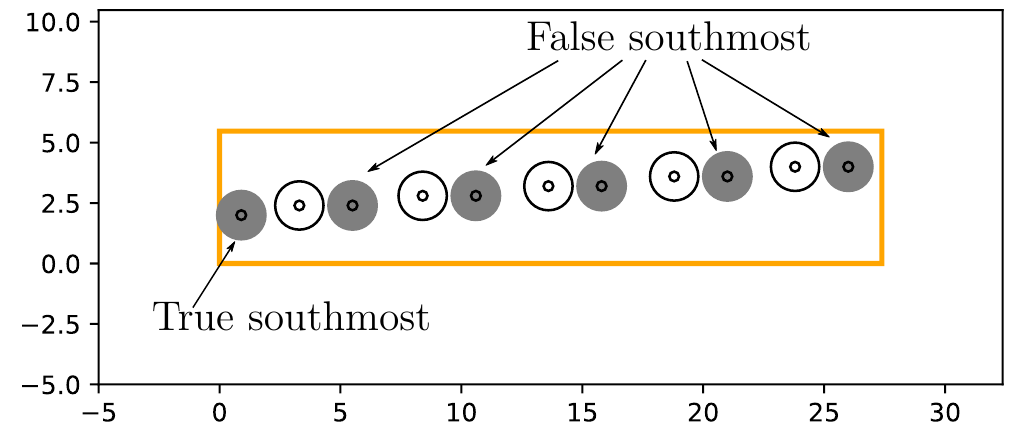}
         \caption{}
         \label{south2}
     \end{subfigure}
     \caption[Short Caption]{(a) The false southmost robots are from the southmost layer in case of a randomly generated configuration. (b) A configuration in which even the northmost robot is false southmost.}
\label{}
\end{figure}

\begin{figure}[thb!]
     \centering
     \begin{subfigure}[b]{0.38\linewidth}
         \centering
         \includegraphics[width=\linewidth]{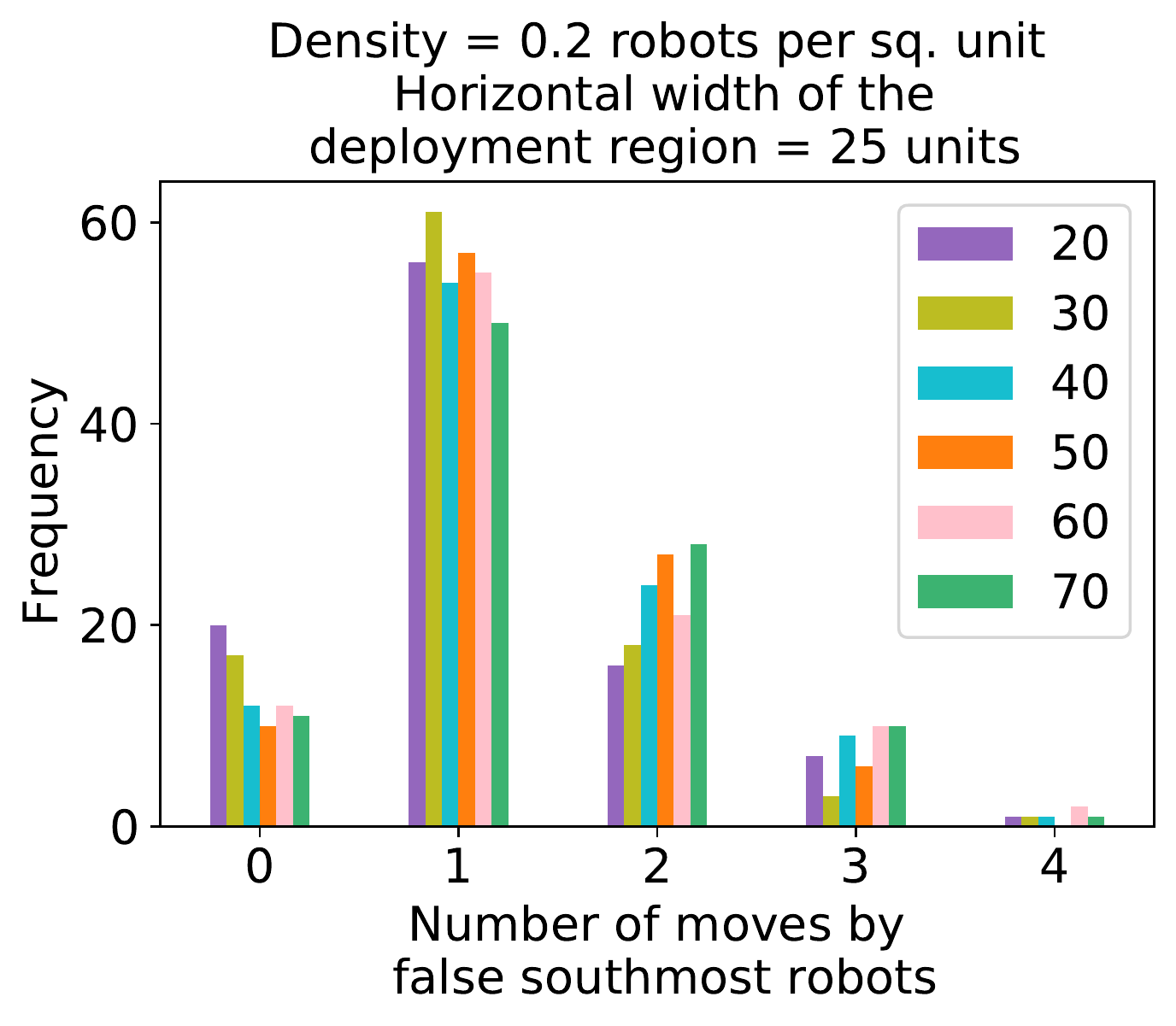}
         \caption{}
         \label{}
     \end{subfigure}
     \hfill
     \begin{subfigure}[b]{0.38\linewidth}
         \centering
         \includegraphics[width=\linewidth]{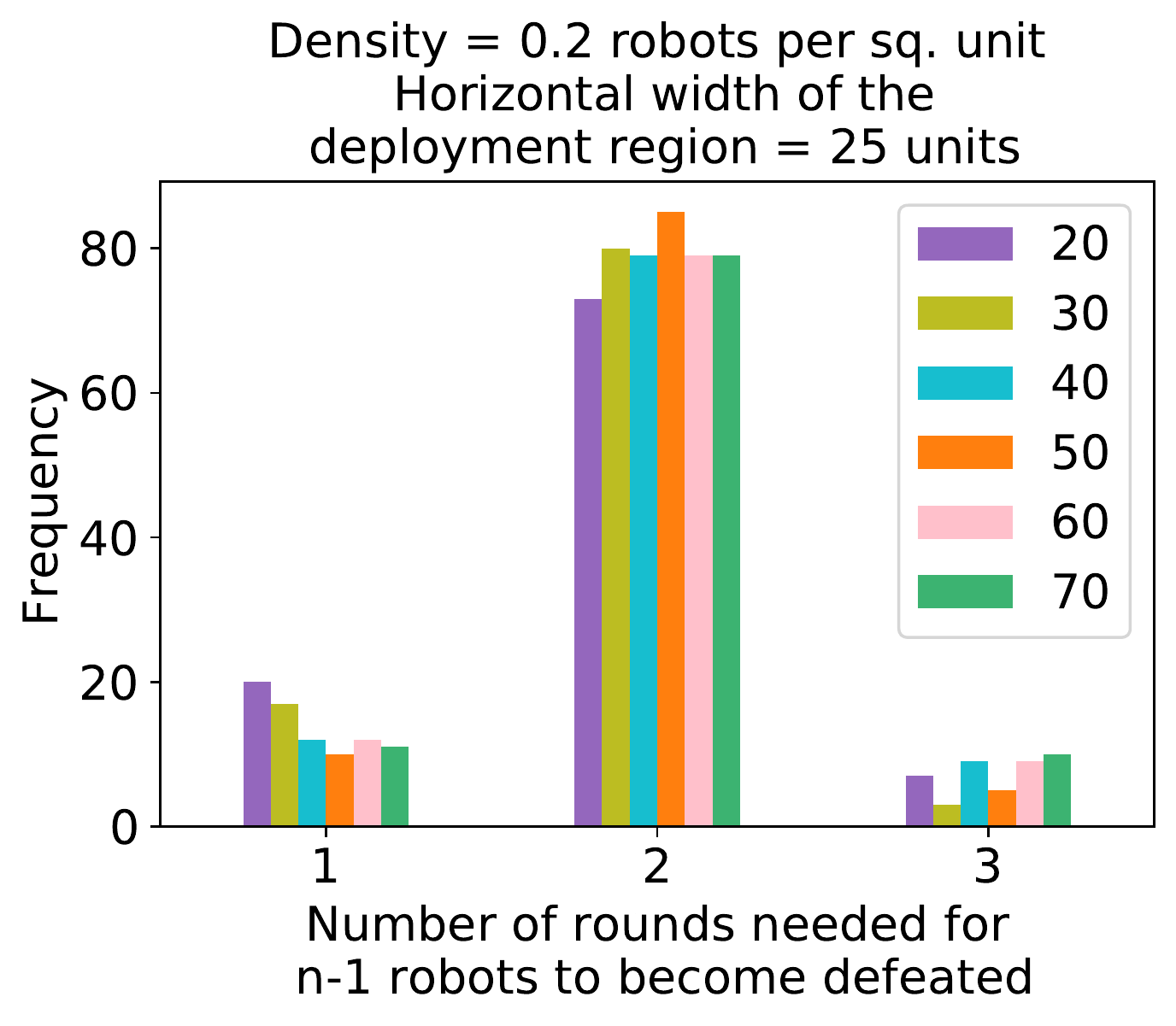}
         \caption{}
         \label{}
     \end{subfigure}
          \caption[Short Caption]{Observations for the leader election algorithm for 20, 30, 40, 50, 60 and 70 robots. In each case, the length of the horizontal side of the deployment region is 25 units and the length of the vertical side is adjusted so the robot density is 0.2 robots per square unit.}
\label{fig: bar2}
\end{figure}



\begin{figure}[thb!]
     \centering
     
         \includegraphics[width=0.38\linewidth]{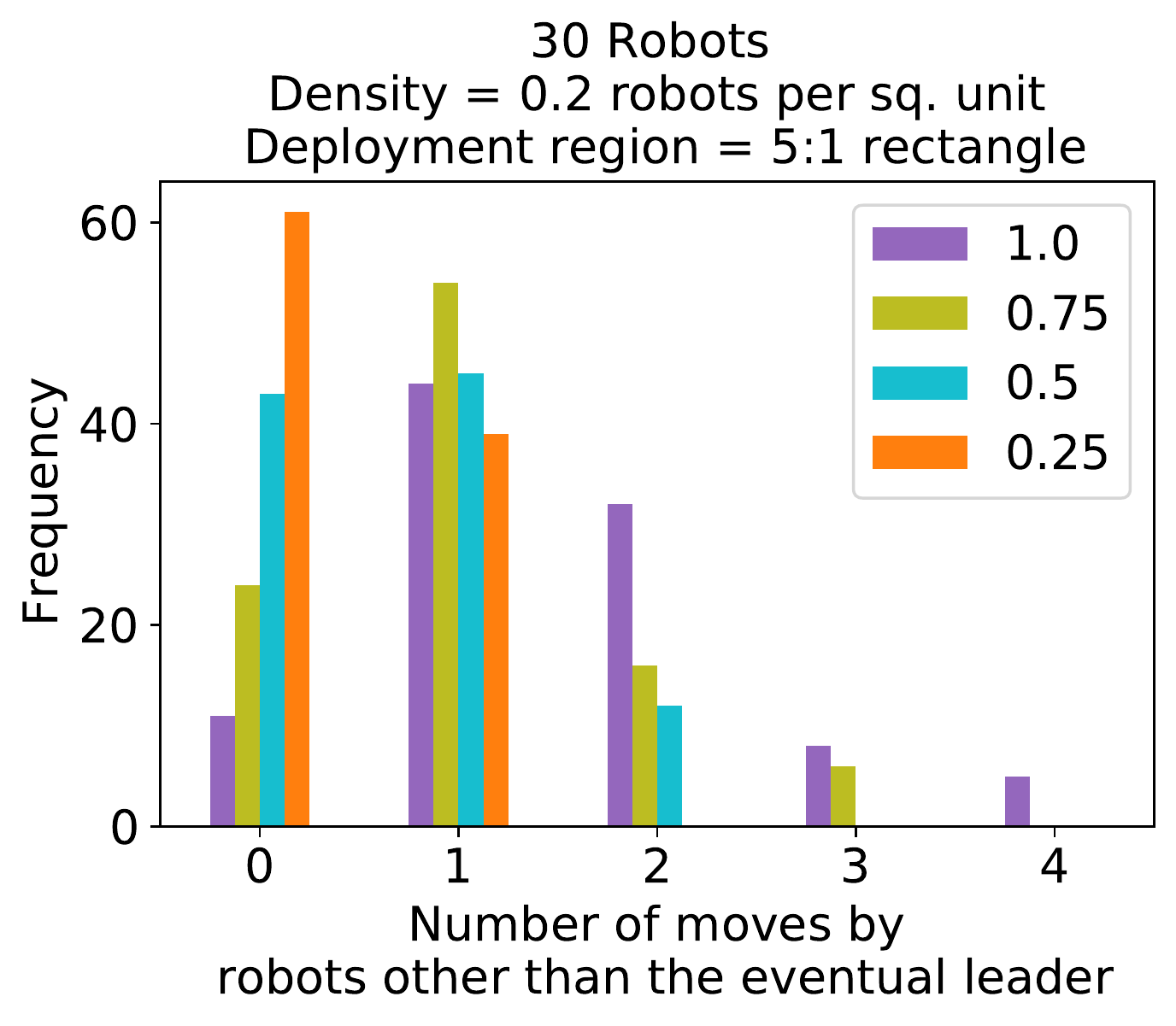}

          \caption[Short Caption]{Observations for the leader election algorithm for 30 robots in $\mathcal{SSYNC}$ for activation probabilities $p =$ 1.0, 0.75, 0.5 and 0.25. In each case, the deployment region is a 5:1 rectangle and the robot density is 0.2 robots per square unit.}
\label{fig: bar3}
\end{figure}

\clearpage

We now focus on the entire mutual visibility algorithm. We evaluate the performance of our algorithm in terms of time complexity and total distance traversed by the robots. Recall that we assumed that the size of the team is not known beforehand. If the number of robots was known, then the optimum stretch required for accommodating all robots on the base chain can be calculated in advance. Without this knowledge, we adopt a trial and error approach, i.e., start with some value of stretch and increase it by some amount when we find it to be inadequate. In Fig. \ref{fig: stretch_plot}, we compare the stretch set by our algorithm with the optimum value sufficient to form the base chain. The jumps in the graph correspond to the expansions, i.e., when we find that our previously set stretch is not large enough and increase it (by setting the position of the first robot to be the position of the second robot of the previous base chain). The stretch determines the horizontal and vertical widths of the final configuration. They are shown in Fig. \ref{fig: final_plot}. One can see that these graphs also have bumps corresponding to the ones in Fig. \ref{fig: stretch_plot}. The expansion also affects the time complexity and total distance traversed. Observed values of time and total distance traversed are presented in Fig. \ref{fig: time_plot} and Fig. \ref{fig: dist_plot} respectively. These graphs also have similar bumps as expected. For these experiments, we chose a square deployment region with sides of length 25 units, $\mathcal{FSYNC}$ scheduler and rigid movement model.

\begin{figure}[thb!]
     \centering
     \begin{subfigure}[b]{0.4\linewidth}
         \centering
         \includegraphics[width=\linewidth]{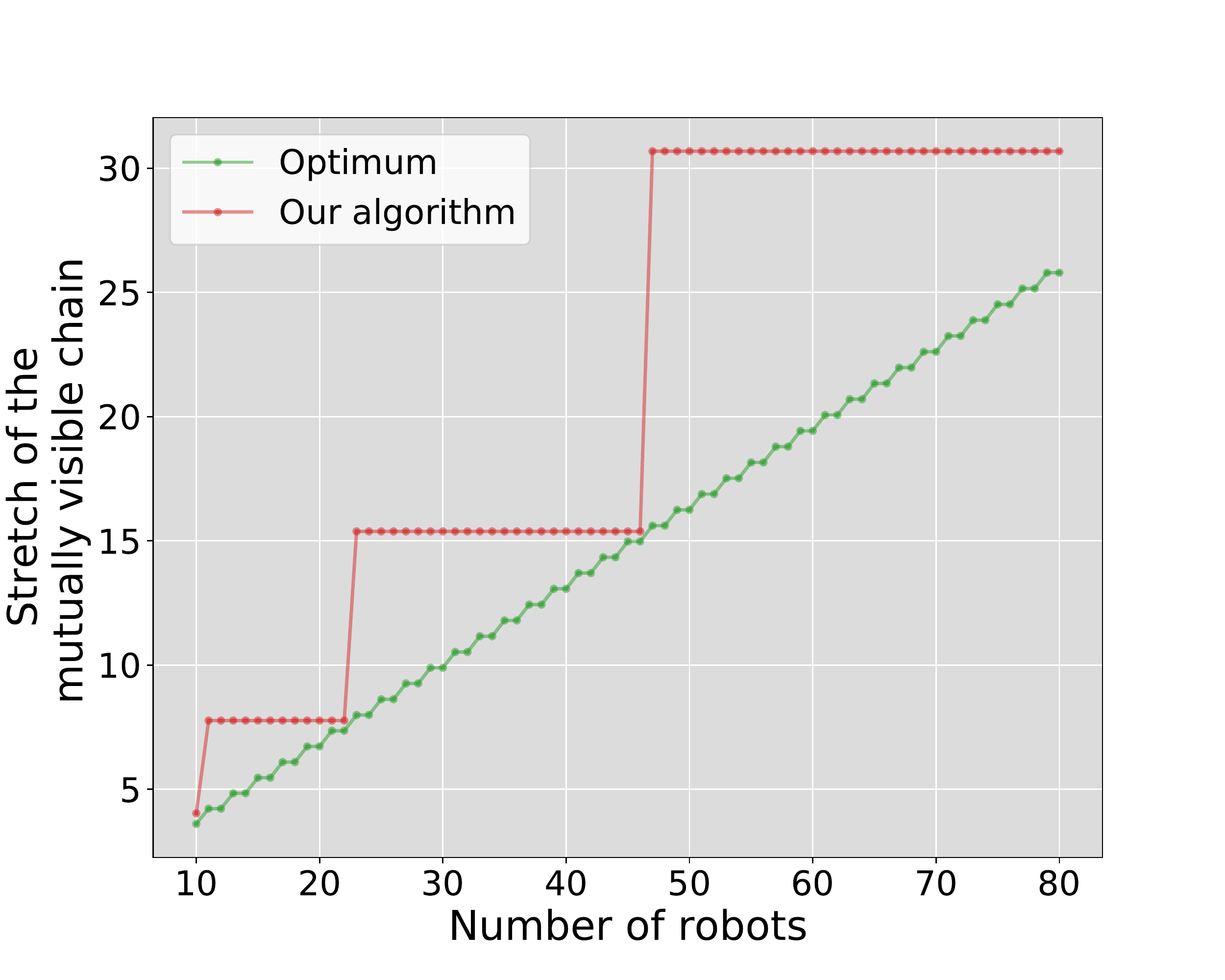}
         \caption{}
         \label{fig: stretch_plot}
     \end{subfigure}
     \hfill
     \begin{subfigure}[b]{0.4\linewidth}
         \centering
         \includegraphics[width=\linewidth]{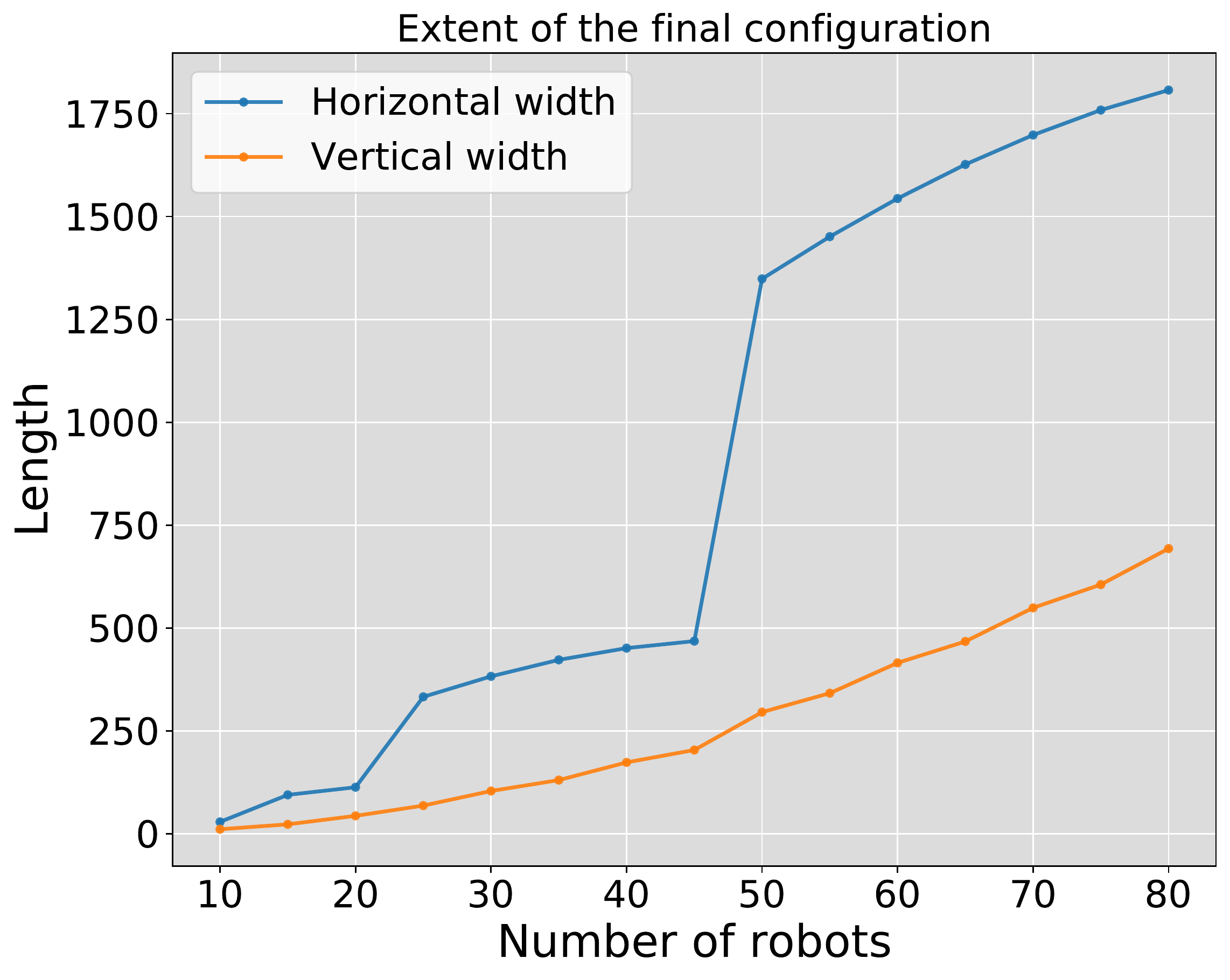}
         \caption{}
         \label{fig: final_plot}
     \end{subfigure}
          \caption[Short Caption]{(a) Comparison between the optimal stretch required (so that the base chain can accommodate $n-1$ robots) and the stretch set by our algorithm (when the size of the swarm is not known beforehand). (b) The horizontal and vertical width of the final configuration for varying swarm size.}
\label{}
\end{figure}

\begin{figure}[thb!]
     \centering
     \begin{subfigure}[b]{0.4\linewidth}
         \centering
         \includegraphics[width=\linewidth]{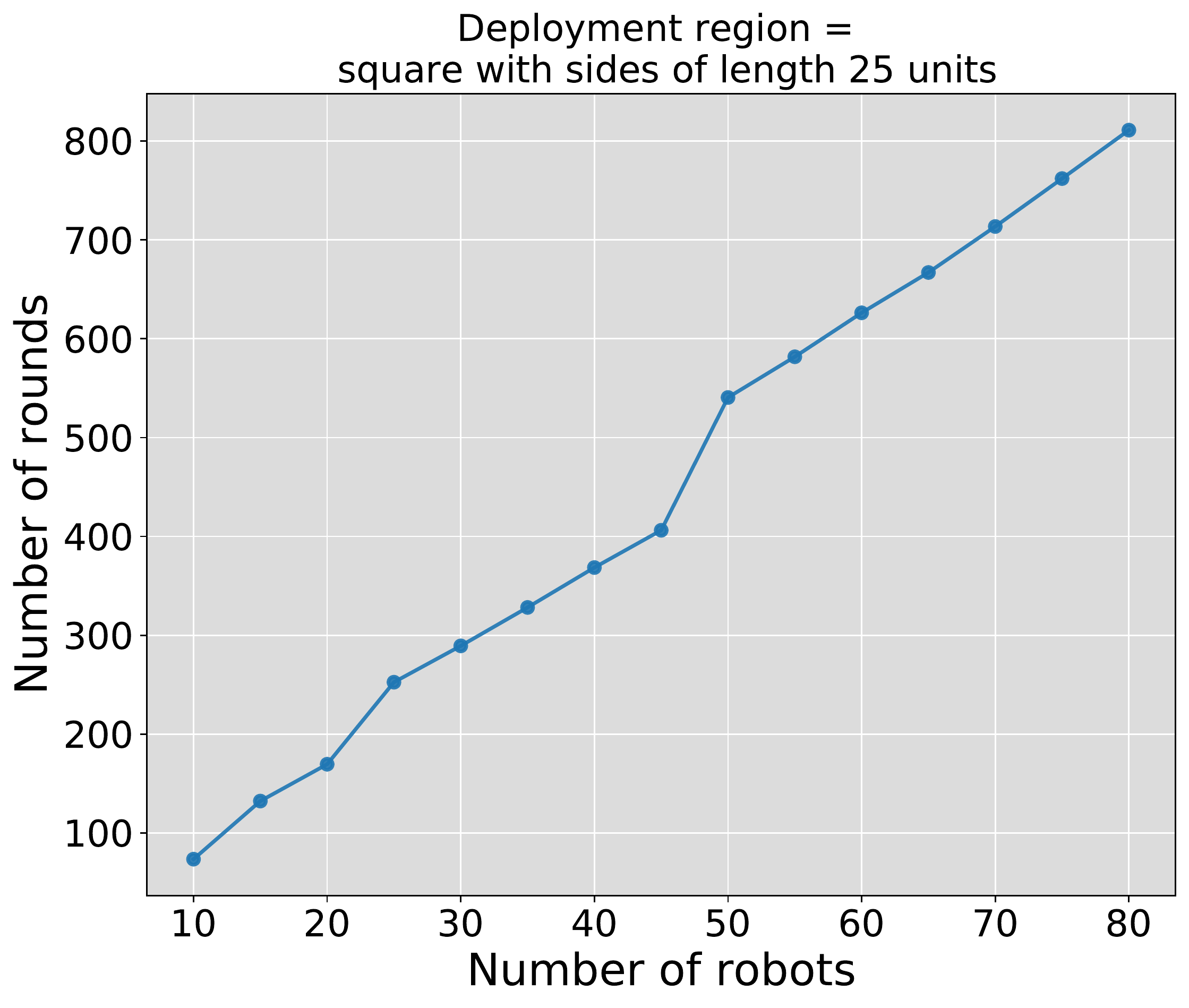}
         \caption{}
         \label{fig: time_plot}
     \end{subfigure}
     \hfill
     \begin{subfigure}[b]{0.4\linewidth}
         \centering
         \includegraphics[width=\linewidth]{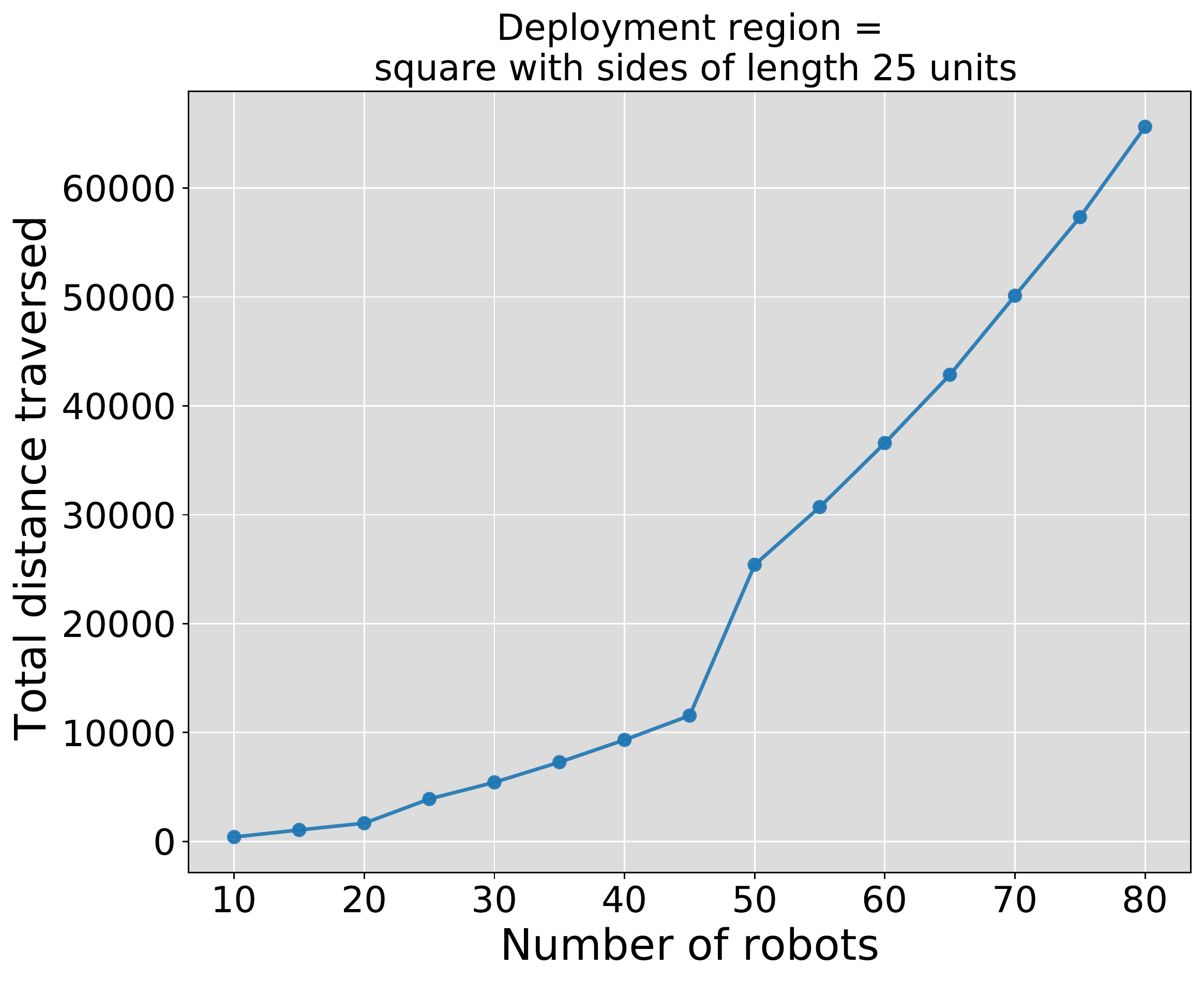}
         \caption{}
         \label{fig: dist_plot}
     \end{subfigure}
          \caption[Short Caption]{(a) Number of rounds required and (b) total distance traversed for varying swarm size.}
\label{}
\end{figure}

Note that the total distance traversed should also depend on the horizontal width $H$  and vertical width $V$ of the initial configuration. To observe their effects, we ran the algorithm for horizontally stretched (10:1) and vertically stretched (1:10) deployment regions with 30 robots. The results are presented in Fig. \ref{fig: dist2_plot}. Our previous experiments suggest that in the case of horizontally stretched regions, there will be more moves by false southmost robots. This will slightly increase the total distance traversed in the case of horizontally stretched regions. However, the results in Fig. \ref{fig: dist2_plot} show the opposite. The reason is the following. $H$ affects the total distance traversed because the robots have to move horizontally to reach the $Y$-axis (passing through the leader). On the other hand, there is an influence of $V$ because the robots have to move southwards to reach the leader. Notice that in the case of a horizontally stretched $10a \times a$ deployment region, the best case (with respect to the horizontal moves) is when the elected leader is from the middle portion of the region. In this case, the robots have to move roughly at most $5a$ units to reach the $Y$-axis. The worst case (with respect to the horizontal moves) is when the leader is from the eastmost or westmost portion of the region, as in this case, a large section (roughly half) of the robots have to move more than $5a$ units. For the vertically stretched $a \times 10a$ deployment region, in all cases, roughly half of the robots move more than $5a$ units to come out of the deployment region. For this reason, the total distance traversed on average for the horizontally stretched deployment regions is lower, but the standard deviation is higher; on the other hand, the total distance traversed on average for the vertically stretched deployment regions is higher, but the standard deviation is lower (see Fig. \ref{fig: dist2_plot}). 
 \begin{figure}[thb!]
     \centering
    \includegraphics[width=.5\linewidth]{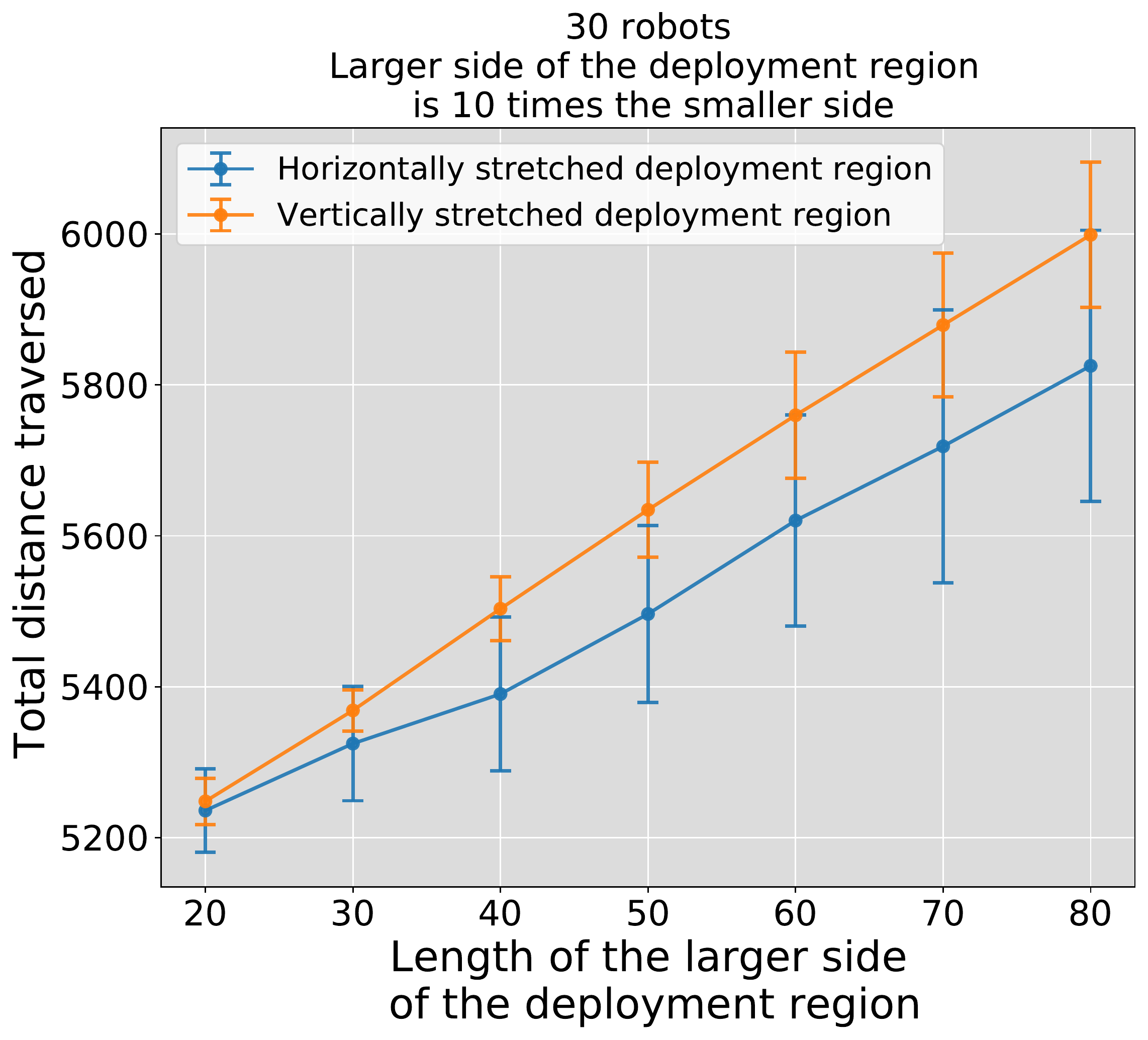}
     \caption[Short Caption]{Comparison between total distance traversed for horizontally and vertically stretched deployment regions for varying swarm size.}
\label{fig: dist2_plot}
\end{figure}

\section{Concluding Remarks}

 In this paper, we presented an algorithm for \textsc{Mutual Visibility} by opaque fat robots with slim omnidirectional camera. We proved the correctness of our algorithm in the $\mathcal{SSYNC}$ setting. The $\mathcal{SSYNC}$ assumption is required to prove the correctness of the leader election stage of the algorithm. The rest of the algorithm can be easily adapted to the $\mathcal{ASYNC}$ setting with some minor modifications. The main difficulty of the leader election stage arises due to the moves by false southmost robots. However, they are quite rare, as we have seen in our simulations. Hence, our algorithm should work well also in the $\mathcal{ASYNC}$ setting as a heuristic. However, it might be of theoretical interest to have a provably correct algorithm for leader election in $\mathcal{ASYNC}$. We believe that it might be possible to adapt our approach to the $\mathcal{ASYNC}$ setting to obtain a provably correct algorithm.

 We assumed that the robots have an agreement on the direction and orientation of both $X$ and $Y$ axes of their local coordinate systems. An interesting direction for future research is to solve the problem without any agreement on the coordinate system. Our approach to solve the problem was to first elect a leader and then arrange the robots on a mutually visible chain. Without any agreement on the coordinate system, there are some initial symmetric configurations in which leader election is deterministically unsolvable. So, in this case, one can try to solve leader election using randomization and then try to build the mutually visible chain. However, it might be of theoretical interest to have a fully deterministic solution. Since leader election is not always deterministically solvable, a new approach is needed so that the problem can be directly solved without requiring to elect a leader.

 \bibliographystyle{plainurl} 
 \bibliography{ref}

\end{document}